\documentclass[11pt,a4paper]{article}
\usepackage[utf8]{inputenc}

\usepackage{mathrsfs,amsthm,xcolor,verbatim,bbm,amsmath,amsfonts,amssymb,nicefrac,enumitem,hyperref,bm,mathtools,xparse,etoolbox,textcomp}
\usepackage[margin=1in]{geometry}
\usepackage[capitalise]{cleveref} 
\usepackage{xurl}
\usepackage{breakurl}

\crefname{enumi}{item}{items}

\crefname{equation}{}{}
\crefname{subsection}{Subsection}{Subsections}

\hypersetup{
    colorlinks,
    linkcolor={red!50!black},
    citecolor={green},
    urlcolor={blue!80!black}
}

\theoremstyle{plain}
\newtheorem{theorem}{Theorem} [section]
\newtheorem{lemma}[theorem]{Lemma}
\newtheorem{prop}[theorem]{Proposition}
\newtheorem{cor}[theorem]{Corollary}

\newtheorem{setting}[theorem]{Setting}

\theoremstyle{remark}

\theoremstyle{definition}

\DeclareMathAlphabet{\mathpzc}{OT1}{pzc}{m}{it}

\DeclareFontEncoding{LS1}{}{}
\DeclareFontSubstitution{LS1}{stix}{m}{n}
\DeclareMathAlphabet{\mathscr}{LS1}{stixscr}{m}{n}

\newcommand{\R}{\mathbb{R}}
\newcommand{\N}{\mathbb{N}}

\newcommand{\Z}{\mathbb{Z}}

\newcommand{\w}[1]{\mathfrak{w}^{#1}}
\renewcommand{\b}[1]{\mathfrak{b}^{#1}}
\renewcommand{\v}[1]{\mathfrak{v}^{#1}}

\renewcommand{\c}[1]{\mathfrak{c}^{#1}}

\newcommand{\smallsum}{\textstyle\sum}

\newcommand{\with}{\curvearrowleft}

\newcommand{\cB}{\mathcal{B}}

\newcommand{\cE}{\mathcal{E}}

\newcommand{\cG}{\mathcal{G}}

\newcommand{\cI}{\mathcal{I}}

\newcommand{\cL}{\mathcal{L}}

\newcommand{\bfa}{\mathbf{a}}

\newcommand{\bfm}{\mathbf{m}}

\newcommand{\scrL}{\mathscr{L}}

\newcommand{\scrN}{\mathscr{N}}

\newcommand{\scra}{\mathscr{a}}
\newcommand{\scrb}{\mathscr{b}}

\newcommand{\fC}{\mathfrak{C}}

\newcommand{\fG}{\mathfrak{G}}

\newcommand{\fL}{\mathfrak{L}}

\newcommand{\fb}{\mathfrak{b}}
\newcommand{\fc}{\mathfrak{c}}
\newcommand{\fd}{\mathfrak{d}}

\newcommand{\fs}{\mathfrak{s}}

\newcommand{\fv}{\mathfrak{v}}
\newcommand{\fw}{\mathfrak{w}}
\newcommand{\fx}{\mathscr{x}}
\newcommand{\fy}{\mathscr{y}}
\newcommand{\fz}{\mathfrak{z}}

\renewcommand{\emptyset}{\varnothing}

\DeclarePairedDelimiter{\norm}{\lVert}{\rVert}
\DeclarePairedDelimiter{\abs}{\lvert}{\rvert}
\DeclarePairedDelimiter{\rbr}{(}{)}
\DeclarePairedDelimiter{\br}{[}{]}
\DeclarePairedDelimiter{\cu}{\{}{\}}
\DeclarePairedDelimiter{\spro}{\langle}{\rangle}

\newcommand{\Rect}{\mathfrak{R}}

\renewcommand{\d}{ \mathrm{d}}

\newcommand{\qandq}{\qquad\text{and}\qquad}

\newcommand{\indicator}[1]{\mathbbm{1}_{\smash{#1}}}

\newcommand{\realization}[1] {\mathscr{N} ^{ #1  }}

\newcommand{\width}{H}

\ExplSyntaxOn

\bool_new:N \g_noteobserve

\NewDocumentCommand{\nobs}{}{
  \bool_if:nTF { \g_noteobserve } {
    \bool_gset_false:N \g_noteobserve 
    note~
  } {
    \bool_gset_true:N \g_noteobserve 
    observe~
  }
}

\NewDocumentCommand{\Nobs}{}{
  \bool_if:nTF { \g_noteobserve } {
    \bool_gset_false:N \g_noteobserve 
    Note~
  } {
    \bool_gset_true:N \g_noteobserve 
    Observe~
  }
}

\seq_new:N \g_cflist_loaded
\seq_new:N \g_cflist_pending

\NewDocumentCommand{\cfadd}{ m }
{
  \seq_if_in:NnF \g_cflist_loaded { #1 } {
    \seq_if_in:NnF \g_cflist_pending { #1 } {
      \seq_gput_right:Nn \g_cflist_pending { #1 }
    }
  }
}

\NewDocumentCommand{\cfconsiderloaded}{ m }{
  \seq_gput_right:Nn \g_cflist_loaded {#1}
}

\NewDocumentCommand{\cfremove}{ m }
{
  \seq_gremove_all:Nn \g_cflist_pending { #1 }
}

\NewDocumentCommand{\cfload}{ o }
{
  \seq_if_empty:NTF \g_cflist_pending {\unskip} {
    (cf.\ \cref{\seq_use:Nn \g_cflist_pending {,}})\IfValueTF{#1}{#1~}{\unskip}
    \seq_gconcat:NNN \g_cflist_loaded \g_cflist_loaded \g_cflist_pending
    \seq_gclear:N \g_cflist_pending
  }
}

\NewDocumentCommand{\cfclear} {} {
  \seq_gclear:N \g_cflist_loaded
  \seq_gclear:N \g_cflist_pending
}

\NewDocumentCommand{\cfout}{ o }
{
  \seq_if_empty:NTF \g_cflist_pending {\unskip} {
    (cf.\ \cref{\seq_use:Nn \g_cflist_pending {,}})\IfValueTF{#1}{#1~}{\unskip}
    \seq_gclear:N \g_cflist_pending
  }
}

\NewDocumentCommand{\ifnocf} { m } {
  \seq_if_empty:NT \g_cflist_pending { #1 }
}

\ExplSyntaxOff

\NewDocumentEnvironment{cproof}{m}
{\begin{proof}[Proof of \cref{#1}]}%
{\noindent The proof of~\cref{#1} is thus complete.
\end{proof}}

\NewDocumentEnvironment{cproof2}{m}
{\begin{proof}[Proof of \cref{#1}]}%
{\noindent This completes the proof of~\cref{#1}.
\end{proof}}

\title{Convergence analysis for gradient flows in the training \\
of artificial neural networks with ReLU activation}

\author{Arnulf Jentzen$^{1, 2}$ and
Adrian Riekert$^3$
\bigskip
\\
\small{$^1$ Applied Mathematics: Institute for Analysis and Numerics,}
\vspace{-0.1cm}\\
\small{University of Münster, Germany, e-mail: \texttt{ajentzen}\textcircled{\texttt{a}}\texttt{uni-muenster.de}}
\smallskip
\\
\small{$^2$ School of Data Science and Shenzhen Research Institute of Big Data,}
\vspace{-0.1cm}\\
\small{The Chinese University of Hong Kong, Shenzhen, China, e-mail: \texttt{ajentzen}\textcircled{\texttt{a}}\texttt{cuhk.edu.cn}}
\smallskip
\\
\small{$^3$ Applied Mathematics: Institute for Analysis and Numerics,}
\vspace{-0.1cm}\\
\small{University of Münster, Germany, e-mail: \texttt{ariekert}\textcircled{\texttt{a}}\texttt{uni-muenster.de}}}

\date{\today}

\begin{document}

\maketitle

\begin{abstract}
    Gradient descent (GD) type optimization schemes are the standard methods to train artificial neural networks (ANNs) with rectified linear unit (ReLU) activation. Such schemes can be considered as discretizations of gradient flows (GFs) associated to the training of ANNs with ReLU activation and most of the key difficulties in the mathematical convergence analysis of GD type optimization schemes in the training of ANNs with ReLU activation seem to be already present in the dynamics of the corresponding GF differential equations. It is the key subject of this work to analyze such GF differential equations in the training of ANNs with ReLU activation and three layers (one input layer, one hidden layer, and one output layer).
    In particular, in this article we prove in the case where the target function is possibly multi-dimensional and continuous and in the case where the probability distribution of the input data is absolutely continuous with respect to the Lebesgue measure that the risk of every bounded GF trajectory converges to the risk of a critical point. 
    In addition, in this article we show in the case of a $1$-dimensional affine linear target function and in the case where the probability distribution of the input data coincides with the standard uniform distribution that the risk of every bounded GF trajectory converges to zero if the initial risk is sufficiently small. 
    Finally, in the special situation where there is only one neuron on the hidden layer (1-dimensional hidden layer) 
    we strengthen the above named result for affine linear target functions by proving that that the risk of every (not necessarily bounded) GF trajectory converges to zero if the initial risk is sufficiently small.
\end{abstract}

\tableofcontents

\section{Introduction}

Gradient descent (GD) type optimization schemes are the standard tools in the training of feedforward fully connected artificial neural networks (ANNs) with rectified linear unit (ReLU) activation. Such GD type optimization schemes can be considered as temporal discretization methods for the associated gradient flow (GF) differential equations and most of the key difficulties which arise in the mathematical convergence analysis of GD type optimization schemes in the training of ANNs with ReLU activation already arise in the mathematical convergence analysis of the corresponding GFs. It is the key subject of this article to analyze such GFs arising in the training of ANNs with ReLU activation and, in particular, to prove that the risk of every bounded GF trajectory converges in the training of ANNs with ReLU activation to the risk of a critical point. 
We are particularly interested in the mathematical convergence analysis of GF trajectories instead of time discrete GD optimization schemes since, on the one hand,
most of the key difficulties which arise in the mathematical analysis of GD type optimization schemes in the training of ANNs with ReLU activation already arise in the mathematical analysis of the corresponding GFs
and since, on the other hand,
the consideration of such GF trajectories allows us to focus on precisely such key difficulties.

In the scientific literature there are several quite promising approaches regarding the mathematical convergence analysis for GD type optimization schemes and GFs, respectively. For instance,
we point to \cite{BachChizatOyallon2019,DuZhaiPoczosSingh2018arXiv,EMaWu2020,JacotGabrielHongler2018} for results on the convergence of GF in the training of ANNs in the overparametrized regime, where the number of neurons has to be sufficiently large when compared to the number of used input-output data pairs.
Another promising idea is to view the neurons of an ANN as interacting particles and consider the limit of the associated
empirical measures
as the number of neurons increases to infinity.
The limiting process of the corresponding GFs is known in the scientific literature as Wasserstein gradient flow;
cf., e.g.,~\cite{ZhengdaoRotskoff2020,Chizat2021,ChizatBach2018}, the overview article \cite{EMaWojtowytschWu2020}, and the references mentioned therein.
Most convergence results for the Wasserstein gradient flow require smoothness assumptions on the considered risk function, which are not satisfied for ANNs with ReLU activation.
To overcome this issue, a different parametrization for ReLU networks has been proposed in \cite[Section~4.2]{ChizatBach2018}.
In \cite{BahRauhutTerstiege2021,ChitourLiaoCouillet2019} GF processes have been considered in the context of training deep linear neural networks, 
in which the employed activation function is the identity.
The behavior of the realization functions of ANNs with one hidden layer and
ReLU activation under the GF dynamics has been investigated in more detail in \cite{MaennelGelly2018,WilliamsTragerPanozzo2019}.
Another recent idea is to consider only very special target functions
and we refer, in particular, to \cite{CheriditoJentzenRiekert2021,JentzenRiekert2021} for convergence results for GF and GD processes in the case of constant target functions.
In the more general case of affine linear target functions,
the critical points of the risk function were characterized in
\cite{CheriditoJentzenRossmannek2021} and parts of the analysis in this article
exploit this characterization.
For further abstract convergence results on GF processes we point, e.g., to~\cite{AbsilMahonyAndrews2005,BolteDaniilidis2006,FehrmanGessJentzen2020,Santambrogio2017} and the references mentioned therein.

It is the key topic of this article to provide some first basics steps regarding the mathematical convergence analysis of GFs arising in the training of ANNs with ReLU activation. Specifically, in one of main results of this article, see \cref{theo:intro:gen:item5} in \cref{theo:intro:general} in this introductory section, we prove that the risk of every bounded GF trajectory converges in the training of ANNs with one hidden layer and ReLU activation to the risk of a critical point. In \cref{theo:intro:general} below we study fully connected feedfoward ANNs with a $d$-dimensional input layer (with $d \in \N = \cu{ 1, 2, 3, ... }$ neurons on the input layer), with an $\width$-dimensional hidden layer (with $\width \in \N$ neurons on the hidden layer), and with a $1$-dimensional output layer (with one neuron on the output layer). There are thus $ \width d $ scalar real weight parameters and $ \width $ scalar real bias parameters to describe the affine linear transformation in between the $ d $-dimensional input layer and the $ \width $-dimensional hidden layer and there are thus $ \width $ scalar real weight parameters and $ 1 $ scalar real bias parameter to describe the affine linear transformation in between the $ \width $-dimensional hidden layer and the $1$-dimensional output layer. Overall the ANNs in \cref{theo:intro:general} thus consist of precisely $\fd = d \width + 2 \width + 1$ scalar real ANN parameters.

In \cref{theo:intro:general} we study fully connected feedfoward ANNs with the ReLU activation function 
$\R \ni x \mapsto \max \cu{ x, 0 } \in \R$ (which is also referred to as rectifier function) as the activation function. The ReLU activation function $\R \ni x \mapsto \max \cu{ x, 0 } \in \R$ fails to be differentiable and can thus not be used to specify gradients in GD type optimization schemes and GFs, respectively. A common procedure to overcome this issue (cf.~\cite{JentzenRiekert2021} and \cite{CheriditoJentzenRiekert2021}) is to approximate the ReLU activation function $\R \ni x \mapsto \max \cu{ x, 0 } \in \R$ through appropriate continuously differentiable functions which converge pointwise to the ReLU activation function and whose derivatives converge pointwise to the \emph{left derivative} 
of the ReLU activation function. In \cref{theo:intro:general} the function 
$\Rect_{ \infty } \colon \R \to \R$ specifies the ReLU activation function and 
the functions $\Rect_r \colon  \R \to \R$, $r \in \N$, serve as such 
continuously differentiable approximations of the ReLU activation function; 
see \cref{theo:intro:general:eq1} in \cref{theo:intro:general}.

The finite measure $\mu \colon [\scra , \scrb ]^d \to [0, \infty ]$ in \cref{theo:intro:general} specifies up to a normalization constant the 
probability distribution of the input data of the supervised learning problem 
considered in \cref{theo:intro:general}. In \cref{theo:intro:general} we assume that the measure $\mu \colon [\scra , \scrb ]^d \to [0, \infty]$ is absolutely continuous 
with respect to the Lebesgue measure. 
The functions $\cL_r \colon \R^\fd \to \R$, $r \in \N \cup \cu{ \infty }$, in \cref{theo:intro:general} describe 
the risk functions associated to the considered ANNs in the sense that 
for all $r \in \N \cup \cu{ \infty }$ we have that $\cL_r \colon \R^\fd \to \R$ is the risk 
function associated to the target function $f \colon [\scra , \scrb ] ^d \to \R$ and the fully connected feedforward ANNs with the activation function $\Rect_r \colon \R \to \R$; see \cref{theo:intro:general:eq2} in \cref{theo:intro:general} for details.

The function $\norm{\cdot} \colon \R^\fd \to \R$ in \cref{theo:intro:general} is nothing else but the standard norm 
on the ANN parameter space $\R^{ \fd } = \R^{ d \width + 2 \width + 1 }$. The function $\cG \colon \R^\fd \to \R^\fd$ in \cref{theo:intro:general} specifies the generalized gradients of the risk function $\cL_{ \infty } \colon \R^\fd \to \R$ using the continuously differentiable approximations $\Rect_r \colon \R \to \R$, $r \in \N$.

\Cref{theo:intro:gen:item1} in \cref{theo:intro:general} asserts that the generalized gradient function $\cG \colon \R^\fd \to \R^\fd$ is locally bounded and measurable. This statement is provided to ensure that for every continuous function $\Theta = ( \Theta_t )_{ t \in [0,\infty) } \colon [0,\infty) \to \R^{ \fd } $ and every $t \in [0,\infty)$ we have that the Lebesgue integral $\int_0^t \mathcal{G}( \Theta_s ) \, \d s$ makes sense (cf.~\cref{theo:intro:gen:item5,theo:intro:gen:item6} in \cref{theo:intro:general}).

\Cref{theo:intro:gen:item2} in \cref{theo:intro:general} reveals that the generalized gradient function $\cG \colon \R^\fd \to \R^\fd$ is \emph{lower semicontinuous}. In the case of ANNs with smooth activation functions it follows directly from Lebesgue's theorem of dominated convergence that the gradient function of the risk function is continuous. In the case of ANNs with ReLU activation, however, the generalized gradient function $\cG \colon \R^\fd \to \R^\fd$ fails to be continuous but in \cref{theo:intro:gen:item2} in \cref{theo:intro:general} we prove that this generalized gradient function is instead lower semicontinuous.

\Cref{theo:intro:gen:item3} in \cref{theo:intro:general} connects the generalized gradient function $\cG \colon \R^\fd \to \R^\fd$ with standard gradients of the risk function $\cL_{ \infty } \colon \R^\fd \to \R$ by demonstrating that there exists an open set $U \subseteq \R^\fd$ with full Lebesgue measure such that $\cL_\infty$ restricted to $U$ is continuously differentiable with $\cG |_U \colon U \to \R^\fd$ being the gradient of $(\cL_\infty )|_U \colon U \to \R$.

\Cref{theo:intro:gen:item5} in \cref{theo:intro:general} establishes that the risk of every bounded GF trajectory converges in the training of the considered ANNs to the risk of a critical point. 
\Cref{theo:intro:gen:item6} in \cref{theo:intro:general} reveals that the risk of every bounded GF trajectory with sufficiently small initial risk converges in the training of the considered ANNs to the risk of the global minima of $\cL_{ \infty } \colon \R^\fd \to \R$.
We now present the precise statement of \cref{theo:intro:general}.

\begin{theorem} \label{theo:intro:general}
Let $d, \width, \fd \in \N$, $ \scra \in \R$, $\scrb \in ( \scra, \infty)$, $f \in C ( [\scra , \scrb ]^d , \R)$ satisfy $\fd = d\width + 2 \width + 1$,
let $\Rect_r \in C ( \R , \R )$, $r \in \N \cup \cu{ \infty } $, satisfy for all $x \in \R$ that $( \bigcup_{r \in \N} \cu{ \Rect_r } ) \subseteq C^1( \R , \R)$, $\Rect_\infty ( x ) = \max \cu{ x , 0 }$,
 $\sup_{r \in \N} \sup_{y \in [- \abs{x}, \abs{x} ] } \rbr*{ \abs{\Rect_r(y)} + \abs{ ( \Rect_r)'(y)}} < \infty$, and
\begin{equation} \label{theo:intro:general:eq1}
    \limsup\nolimits_{r \to \infty}  \rbr*{ \abs { \Rect_r ( x ) - \Rect _\infty ( x ) } + \abs { (\Rect_r)' ( x ) - \indicator{(0, \infty)} ( x ) } } = 0,
\end{equation}
let $\mu \colon \cB ( [\scra , \scrb]^d ) \to [0, \infty]$ be a finite measure,
let $\cL_r \colon \R^\fd \to \R$, $r \in \N \cup \cu{ \infty }$,
satisfy for all $r \in \N \cup \cu{ \infty }$, $\theta = (\theta_1, \ldots, \theta_\fd) \in \R^{\fd}$ that
\begin{multline} \label{theo:intro:general:eq2}
   \cL_r ( \theta ) 
   = \int_{[\scra , \scrb]^d} \bigl( f ( x_1, \ldots, x_d )  \\ - \theta_{\fd} - \smallsum_{i=1}^\width \theta_{\width ( d + 1 ) + i } \br[\big]{ \Rect_r ( \theta_{\width d + i}  + \smallsum_{j=1}^d\theta_{(i-1)d + j } x_j ) } \bigr)^2  \, \mu (\d (x_1, \ldots, x_d ) ),
    \end{multline}
    let $\norm{ \cdot } \colon \R^\fd  \to \R$ satisfy for all $x=(x_1, \ldots, x_\fd) \in \R^\fd $ that $\norm{ x } = [ \sum_{i=1}^\fd \abs*{ x_i } ^2 ] ^{1/2}$,
let $\cG  \colon \R^\fd \to \R^\fd$ satisfy for all
$\theta \in \cu{ \vartheta \in \R^\fd \colon ( ( \nabla \cL_r ) ( \vartheta ) ) _{r \in \N} \text{ is convergent} }$
that $\cG ( \theta ) = \lim_{r \to \infty} (\nabla \cL_r) ( \theta )$,
and assume that $\mu$ is absolutely continuous with respect to the Lebesgue measure on $[\scra , \scrb ] ^d$. Then
\begin{enumerate} [label = (\roman*)]
    \item \label{theo:intro:gen:item1} it holds that $\R^\fd \ni \theta \mapsto \cG ( \theta ) \in \R^\fd$ is locally bounded and measurable,
    \item \label{theo:intro:gen:item2} it holds that $\R^\fd \ni \theta \mapsto \norm{\cG ( \theta ) } \in \R$ is lower semicontinuous,
    \item \label{theo:intro:gen:item3} there exists an open $U \subseteq \R^\fd $ which satisfies $\int_{\R^\fd \backslash U } 1 \, \d x = 0$, $(\cL_\infty )|_U \in C^1 ( U , \R)$, and $\nabla ( ( \cL_\infty ) |_U ) = \cG |_U$,
    \item \label{theo:intro:gen:item5} it holds for all $\Theta \in C([0, \infty) , \R^\fd)$ with $\sup_{t \in [0, \infty)} \norm{\Theta_t} < \infty$ and $\forall \, t \in [0, \infty) \colon \Theta_t = \Theta_0 - \int_0^t \cG ( \Theta_s ) \, \d s$ that there exists $\vartheta \in \cG^{-1} ( \cu{ 0 } )$ such that $\limsup_{t \to \infty} \cL _\infty ( \Theta_t ) = \cL_\infty ( \vartheta )$, and
    \item \label{theo:intro:gen:item6} it holds for all $\Theta \in C([0, \infty) , \R^\fd)$ with $\sup_{t \in [0, \infty)} \norm{\Theta_t } < \infty$, $\forall \, t \in [0, \infty) \colon \Theta_t = \Theta_0 - \int_0^t \cG ( \Theta_s ) \, \d s$, and $\forall \, \theta \in \cG^{-1} ( \cu{ 0 } ) \cap ( \cL_\infty )^{-1} ( ( \inf\nolimits_{\vartheta \in \R^{ \fd }} \cL_\infty ( \vartheta ) , \infty ) ) \colon \cL _\infty ( \Theta_0 ) < \cL_\infty ( \theta )$
that 
\begin{equation}
    \limsup\nolimits_{t \to \infty} \cL_\infty ( \Theta_t ) = \inf\nolimits_{\vartheta \in \R^{ \fd }} \cL_\infty ( \vartheta ).
\end{equation}
\end{enumerate}
\end{theorem}

\Cref{theo:intro:gen:item1} in \cref{theo:intro:general} is a direct consequence of \cref{cor:gradient:measurable} below,
\cref{theo:intro:gen:item2} in \cref{theo:intro:general} is a direct consequence of \cref{cor:gradient.lsc} below,
\cref{theo:intro:gen:item3} in \cref{theo:intro:general} is a direct consequence of \cref{cor:loss:c1} below,
\cref{theo:intro:gen:item5} in \cref{theo:intro:general} is a direct consequence of \cref{theo:flow:conditional} below, and
\cref{theo:intro:gen:item6} in \cref{theo:intro:general} is a direct consequence of \cref{cor:flow:conditional} below.

In \cref{theo:intro:affine:conditional} below we specialise the setup in \cref{theo:intro:general} to the specific situation where there the input is $1$-dimensional (where there is only one neuron on the input layer), where the measure $ \mu \colon \cB ( [ \scra , \scrb ] ) \to [0, \infty] $ coincides with the Lebesgue--Borel measure, and where the target function $f \colon [\scra , \scrb ] \to \R$ is affine linear in the sense that there exist $\alpha, \beta \in \R$ such that for all $x \in [\scra , \scrb]$ it holds that 
\begin{equation} \label{eq:intro:affine:function}
  f(x) = \alpha x + \beta 
\end{equation}
to establish that the risk of every (bounded) GF trajectory with sufficiently small initial risk converges to zero. 
Specifically, 
in the specific situation of \cref{eq:intro:affine:function} we prove in \cref{theo:intro:affine:conditional} that for every continuous 
GF trajectory $\Theta \colon [0,\infty) \to \R^{3 \width + 1 }$
with
\begin{equation} \label{eq:intro:gf:bounded}
  \sup\nolimits_{ t \in [0,\infty) } 
  ( ( \width - 1 ) \norm{ \Theta_t } ) < \infty
\end{equation}
and
\begin{equation}
  \cL_{ \infty }( \Theta_0 ) < \frac{\alpha^2 ( \scrb - \scra ) ^3 }{12 ( 2 \lfloor \width/2 \rfloor + 1 )^4}
\end{equation}
we have that 
$\limsup_{ t \to \infty } \cL_{ \infty }( \Theta_t ) = 0$. 
In this specific situation of a $1$-dimensional input $ d = 1 $ 
(in this specific situation where there is only one neuron on the input layer) 
we observe that the ANN parameter space $\R^{ \mathfrak{d} } $ 
simplifies to 
$\R^{ \mathfrak{d} } = \R^{ d \width + 2 \width + 1 } =  \R^{ 3 \width + 1 }$. 
Moreover, we note that in \cref{theo:intro:affine:conditional} below and in \cref{eq:intro:gf:bounded} above, respectively, we 
assume in the case where the number $\width  \in \N$ of neurons on the hidden layer is strictly bigger than $1$ (in the case where $\width > 1$) that the GF trajectory is bounded.
We now present the precise statement of \cref{theo:intro:affine:conditional}.

\begin{theorem} \label{theo:intro:affine:conditional}
Let $\width, \fd \in \N$, $\alpha, \beta, \scra \in \R$, $\scrb \in ( \scra, \infty)$ satisfy $\fd = 3 \width + 1$,
let $\Rect_r \in C ( \R , \R )$, $r \in \N \cup \cu{ \infty } $, satisfy for all $x \in \R$ that $( \bigcup_{r \in \N} \cu{ \Rect_r } ) \subseteq C^1( \R , \R)$, $\Rect_\infty ( x ) = \max \cu{ x , 0 }$,
 $\sup_{r \in \N} \sup_{y \in [- \abs{x}, \abs{x} ] } \rbr*{ \abs{\Rect_r(y)} + \abs{ ( \Rect_r)'(y)}} < \infty$, and
\begin{equation}
    \limsup\nolimits_{r \to \infty}  \rbr*{ \abs { \Rect_r ( x ) - \Rect _\infty ( x ) } + \abs { (\Rect_r)' ( x ) - \indicator{(0, \infty)} ( x ) } } = 0,
\end{equation}
let $\cL_r \colon \R^\fd \to \R$, $r \in \N \cup \cu{ \infty }$,
satisfy for all $r \in \N \cup \cu{ \infty }$, $\theta = (\theta_1, \ldots, \theta_\fd) \in \R^{\fd}$ that
\begin{equation}
    \cL_r ( \theta ) = \int_{\scra}^\scrb \rbr[\big]{ \alpha x + \beta - \theta_{\fd} - \smallsum_{i=1}^\width \theta_{2 \width + i } \br[\big]{ \Rect_r ( \theta_{\width  + i}  +\theta_{i } x ) } }^2  \, \d x,
\end{equation}
let $\norm{ \cdot } \colon  \R^\fd \to \R$ satisfy for all $x=(x_1, \ldots, x_\fd ) \in \R^\fd $ that $\norm{ x } = [ \sum_{i=1}^n \abs*{ x_i } ^2 ] ^{1/2}$,
let $\cG \colon \R^\fd \to \R^\fd$ satisfy for all
$\theta \in \cu{ \vartheta \in \R^\fd \colon ( ( \nabla \cL_r ) ( \vartheta ) ) _{r \in \N} \text{ is convergent} }$
that $\cG ( \theta ) = \lim_{r \to \infty} (\nabla \cL_r) ( \theta )$,
and let $\Theta \in C([0, \infty) , \R^\fd)$ satisfy $\sup_{t \in [0, \infty)} ( (\width - 1 ) \norm{\Theta_t} ) < \infty$, $\forall \, t \in [0, \infty) \colon \Theta_t = \Theta_0 - \int_0^t \cG ( \Theta_s ) \, \d s$, and $\cL_\infty ( \Theta_0 ) < \frac{\alpha^2 ( \scrb - \scra ) ^3 }{12 ( 2 \lfloor \width/2 \rfloor + 1 )^4}$.
Then $\limsup_{t \to \infty} \cL_\infty ( \Theta_t ) = 0$.
\end{theorem}

\cref{theo:intro:affine:conditional} is a direct consequence of \cref{cor:gradient:flow:convergence:affine} (in the case $\width > 1$) and \cref{theo:intro:1d:affine} (in the case $\width = 1 $) below.
The remainder of this article is organized as follows. In \cref{section:risk:gradient} below 
we establish certain regularity properties for the 
generalized gradient function 
$\cG \colon \R^\fd \to \R^\fd$
in \cref{theo:intro:general} above. 
In \cref{section:gf:convergence} below we employ 
the regularity properties for the 
generalized gradient function 
$\cG \colon \R^\fd \to \R^\fd$
from \cref{section:risk:gradient} to prove \cref{theo:intro:gen:item5,theo:intro:gen:item6} in \cref{theo:intro:general} 
and to prove \cref{theo:intro:affine:conditional} under the more restrictive assumption 
that $\sup_{ t \in [0, \infty) } \norm{ \Theta_t } < \infty$; cf.~\cref{eq:intro:gf:bounded} above. 
In \cref{section:apriori:bounds} below we establish suitable a priori bounds for GF trajectories. 
In \cref{section:small:risk,section:gf:convergence:1neuron} we employ the a priori bounds from \cref{section:apriori:bounds} to prove 
\cref{theo:intro:affine:conditional} under the more general assumption 
that $\sup_{ t \in [0,\infty) } ( ( \width - 1 ) \norm{ \Theta_t } ) < \infty$; cf.~\cref{eq:intro:gf:bounded} above.

\section{Properties of the risk function and its gradient} \label{section:risk:gradient}

In this section we establish several regularity properties for the risk function associated to the considered supervised learning problem;
 see \cref{theo:intro:general:eq2} above. In particular, in \cref{prop:loss:differentiable} in \cref{subsection:risk:diff} below we provide in \cref{eq:assumption:diff} a sufficient condition to ensure that the risk function is differentiable 
and in \cref{cor:gradient.lsc} in \cref{subsection:gradient:lsc} below we prove that the standard norm of the generalized gradient function $\cG \colon \R^\fd \to \R^\fd$ associated to the risk function is lower semicontinuous. 
In the scientific literature results similar to \cref{prop:loss:differentiable} can, e.g., be found in Cheridito et al.~\cite{CheriditoJentzenRossmannek2021}. 
In particular, in the case of only one neuron on the input layer (in the case of a $1$-dimensional input) results similar to \cref{prop:loss:differentiable} have been shown in \cite[Lemma 3.4 and Lemma 3.7]{CheriditoJentzenRossmannek2021}. 

Our proof of \cref{prop:loss:differentiable} employs the known representation result for the generalized gradient function $\cG \colon \R^\fd \to \R^\fd$ in \cref{prop:loss:gradient} in \cref{subsection:loss:gradient} below, the well known local Lipschitz continuity result for the risk function in \cref{lem:realization:lip} in \cref{subsection:risk:diff}, the elementary Lipschitz type estimate for certain affine linear functions in \cref{lem:diff:1} in \cref{subsection:risk:diff}, and the fact that appropriate active neuron regions depend continuously on the ANN parameters which we establish in \cref{cor:interval:continuous} in \cref{subsection:regions:continuous} below. 
Our proof of \cref{cor:gradient.lsc} employs the fact that the absolute value of every component of the generalized gradient function $\cG \colon \R^\fd \to \R^\fd$ is lower semicontinuous which we establish in \cref{cor:gradient:components:lsc} in \cref{subsection:gradient:lsc}. Our proof of \cref{cor:gradient:components:lsc} uses the regularity results for the absolute values of the components of the generalized gradient function $\cG \colon \R^\fd \to \R^\fd$ in \cref{lem:vc:derivatives:cont}, \cref{lem:gradient:wbcomponents:cont}, and \cref{lem:gradient:components:lsc} in \cref{subsection:gradient:lsc}. Our proof of \cref{cor:interval:continuous} uses the appropriate continuity result for active neuron regions in \cref{lem:diff:2} and the well-known results on absolutely continuous measures in \cref{lem:abs:continuous} and \cref{cor:abs:continuous}. 
In the scientific literature \cref{lem:abs:continuous} can, e.g., be found in Rudin~\cite[Theorem 6.11]{Rudin1987}. 

In \cref{setting:snn} in \cref{subsection:description:anns} below we present the mathematical framework which we frequently employ in \cref{section:risk:gradient,section:gf:convergence,section:apriori:bounds} to formulate ANNs with one hidden layer and ReLU activation and the corresponding risk functions (see \cref{eq:setting:realization,eq:setting:approxrisk} in \cref{setting:snn}), in the elementary regularity result in \cref{lem:gradient:est} in \cref{subsection:loss:gradient} we establish an elementary a priori bound for the norm of the generalized gradient function $\cG \colon \R^\fd \to \R^\fd$, and in the elementary regularity result in \cref{cor:gradient:measurable} in \cref{subsection:loss:gradient} we demonstrate that the generalized gradient function $\cG \colon \R^\fd \to \R^\fd$ is locally bounded and measurable. \cref{lem:gradient:est} is used in the proof of \cref{cor:gradient:measurable} in \cref{subsection:loss:gradient} and \cref{cor:gradient:measurable} is employed in \cref{section:gf:convergence} and in \cref{theo:intro:gen:item1} in \cref{theo:intro:general}. 
Only for completeness we include in this section detailed proofs for \cref{prop:loss:gradient}, \cref{lem:gradient:est}, \cref{cor:gradient:measurable}, \cref{lem:abs:continuous},  \cref{cor:abs:continuous}, and \cref{lem:realization:lip}.

\subsection{Mathematical description of artificial neural networks (ANNs)}
\label{subsection:description:anns}
\begin{setting} \label{setting:snn}
Let $d, \width, \fd \in \N$, $ \scra \in \R$, $\scrb \in ( \scra, \infty)$, $f \in C ( [\scra , \scrb ]^d , \R)$ satisfy $\fd = d\width + 2 \width + 1$,
let $\fw  = (( \w{\theta} _ {i,j}  )_{(i,j) \in \cu{1, \ldots, \width } \times \cu{1, \ldots, d } })_{ \theta \in \R^{\fd}} \colon \R^{\fd} \to \R^{\width \times d}$,
$\fb =  (( \b{\theta} _ 1 , \ldots, \b{\theta} _ \width ))_{ \theta \in \R^{\fd}} \colon \R^{\fd} \to \R^{\width}$,
$\fv = (( \v{\theta} _ 1 , \ldots, \v{\theta} _ \width ))_{ \theta \in \R^{\fd}} \colon \R^{\fd} \to \R^{\width}$, and
$\fc = (\c{\theta})_{\theta \in \R^{\fd }} \colon \R^{\fd} \to \R$
 satisfy for all $\theta  = ( \theta_1 ,  \ldots, \theta_{\fd}) \in \R^{\fd}$, $i \in \cu{1, 2, \ldots, \width }$, $j \in \cu{1, 2, \ldots, d }$ that 
 \begin{equation}
     \w{\theta}_{i , j} = \theta_{ (i-1) d + j}, \qquad \b{\theta}_i = \theta_{\width d + i}, \qquad 
\v{\theta}_i = \theta_{ \width (d+1) + i}, \qandq \c{\theta} = \theta_{\fd},
 \end{equation}
let $\Rect_r \in C^1 ( \R , \R )$, $r \in \N $, satisfy for all $x \in \R$ that 
\begin{equation}
    \limsup\nolimits_{r \to \infty}  \rbr*{ \abs { \Rect_r ( x ) - \max \cu{ x , 0 } } + \abs { (\Rect_r)' ( x ) - \indicator{(0, \infty)} ( x ) } } = 0
\end{equation}
and $\sup_{r \in \N} \sup_{y \in [- \abs{x}, \abs{x} ] }  \abs{(\Rect_r)'(y)} < \infty$,
let $\mu \colon \cB ( [\scra , \scrb]^d ) \to [0, \infty]$ be a finite measure,
let $\scrN = (\realization{\theta})_{\theta \in \R^{\fd } } \colon \R^{\fd } \to C(\R^d , \R)$ and $\cL \colon \R^{\fd  } \to \R$
satisfy for all $\theta \in \R^{\fd}$, $x = (x_1, \ldots, x_d) \in \R^d$ that 
\begin{equation} \label{eq:setting:realization}
\realization{\theta} (x) = \c{\theta} + \smallsum_{i=1}^\width \v{\theta}_i \max \cu[\big]{ \b{\theta}_i + \smallsum_{j=1}^d \w{\theta}_{i,j} x_j , 0 }
\end{equation} 
and $\cL (\theta) = \int_{[\scra , \scrb]^d} ( f ( y ) - \realization{\theta} (y) )^2 \, \mu ( \d y )$,
let $\fL_r \colon \R^\fd \to \R$, $r \in \N$,
satisfy for all $r \in \N$, $\theta \in \R^{\fd}$ that
\begin{equation} \label{eq:setting:approxrisk}
    \fL_r ( \theta ) = \int_{[\scra , \scrb]^d} \rbr[\big]{f(y_1, \ldots, y_2) - \c{\theta} - \smallsum_{i = 1}^\width \v{\theta}_i \br[\big]{ \Rect_r \rbr{ \b{\theta}_i + \smallsum_{j = 1}^d \w{\theta}_{i,j} y_j } } }^2 \, \mu(\d (y_1, \ldots, y_d)),
\end{equation}
 let $\lambda \colon \cB ( [\scra , \scrb]^d ) \to [0, \infty]$ be the Lebesgue--Borel measure on $[\scra , \scrb]^d$,
let $\norm{ \cdot } \colon \rbr*{  \bigcup_{n \in \N} \R^n  } \to \R$ and $\spro{  \cdot , \cdot } \colon \rbr*{  \bigcup_{n \in \N} (\R^n \times \R^n )  } \to \R$ satisfy for all $n \in \N$, $x=(x_1, \ldots, x_n)$, $y=(y_1, \ldots, y_n ) \in \R^n $ that $\norm{ x } = [ \sum_{i=1}^n \abs*{ x_i } ^2 ] ^{1/2}$ and $\spro{  x , y } = \sum_{i=1}^n x_i y_i$,
let $I_i^\theta \subseteq \R^d$, $\theta \in \R^{\fd }$, $i \in \cu{1, 2, \ldots, \width }$, satisfy for all 
$\theta \in \R^{\fd}$, $i \in \cu{1, 2, \ldots, \width }$ that 
\begin{equation}
    I_i^\theta = \cu[\big]{ x = (x_1, \ldots, x_d) \in [\scra , \scrb ]^d \colon \b{\theta}_i + \smallsum_{j = 1}^d \w{\theta}_{i,j} x_j > 0 },
\end{equation}
and let $\cG = ( \cG_1 , \ldots, \cG_\fd ) \colon \R^\fd \to \R^\fd$ satisfy for all $\theta \in  \cu{ \vartheta \in \R^\fd \colon ((\nabla \fL_r)(\vartheta ) )_{r \in \N} \text{ is convergent} }$ that $\cG ( \theta ) = \lim_{r \to \infty} ( \nabla \fL_r) ( \theta ) $.

\end{setting}

\subsection{An upper bound for the norm of the gradient of the risk function}
\label{subsection:loss:gradient}

\begin{prop} \label{prop:loss:gradient}
Assume \cref{setting:snn} and let $\theta \in \R^\fd$, $i \in \cu{1, 2, \ldots, \width }$, $j \in \cu{1, 2, \ldots, d }$. Then 
\begin{enumerate} [label = (\roman*)]
\item \label{prop:loss:gradient:item0} it holds for all $r \in \N$ that $\fL_r \in C^1 ( \R^\fd , \R)$,
\item \label{prop:loss:gradient:item1} it holds that $\limsup_{r \to \infty} \abs{\fL_r ( \theta ) - \cL ( \theta ) } = 0$,
\item \label{prop:loss:gradient:item2} it holds that $\limsup_{r \to \infty} \norm{ ( \nabla \fL_r ) ( \theta ) - \cG ( \theta ) } = 0$, and
    \item \label{prop:loss:gradient:item3} it holds that
    \begin{equation} \label{eq:loss:gradient}
\begin{split}
        \cG_{ (i-1) d + j } ( \theta) &= 2 \v{\theta}_i \int_{I_i^\theta} x _j ( \realization{\theta} (x) - f ( x ) ) \, \mu ( \d x ), \\
        \cG_{\width d  + i} ( \theta) &= 2 \v{\theta}_i \int_{I_i^\theta} (\realization{\theta} (x) - f ( x ) ) \, \mu ( \d x ), \\
        \cG_{\width ( d  + 1 ) + i} ( \theta) &= 2 \int_{[\scra , \scrb]^d} \br[\big]{\max \cu[\big]{\b{\theta}_i + \smallsum_{k = 1}^d \w{\theta}_{i, k} x_k , 0 } } ( \realization{\theta}(x) - f ( x ) ) \, \mu ( \d x ), \\
       \text{and} \qquad \cG_{\fd} ( \theta) &= 2 \int_{[\scra , \scrb]^d} (\realization{\theta} (x) - f ( x ) ) \, \mu ( \d x ).
        \end{split}
\end{equation}
\end{enumerate}
\end{prop}
\begin{cproof}{prop:loss:gradient}
Throughout this proof we assume without loss of generality that $\mu ( [ \scra , \scrb ] ^d ) > 0$.
\Nobs that \cite[Proposition 2.3]{JentzenRiekert2021} (applied with $a \with \scra$, $b \with \scrb$, $\mu \with ( \cB ( [ \scra , \scrb ] ^ d ) \ni A \mapsto \mu ( A ) [ \mu ( [ \scra , \scrb ] ^d ) ]^{-1} \in [0, 1 ] )$ in the notation of \cite[Proposition 2.3]{JentzenRiekert2021}) establishes items \ref{prop:loss:gradient:item0}, \ref{prop:loss:gradient:item1}, \ref{prop:loss:gradient:item2}, and \ref{prop:loss:gradient:item3}.
\end{cproof}

\begin{lemma} \label{lem:gradient:est}
Assume \cref{setting:snn} and let $\bfa \in \R$, $\theta \in \R^{\fd}$ satisfy $\bfa = \max \cu{\abs{\scra}, \abs{\scrb}, 1 }$. Then
\begin{equation} \label{lem:gradient:est:eq}
    \norm{ \cG ( \theta ) } ^2 \leq 4 \cL ( \theta ) \rbr*{ \bfa^2 ( d + 1) \norm{ \theta } ^2 + 1 } \mu ( [\scra , \scrb ] ^d ) .
\end{equation}
\end{lemma}
\begin{cproof} {lem:gradient:est}
Throughout this proof assume without loss of generality that $\mu ( [ \scra , \scrb ] ^d ) > 0$. \Nobs that \cref{prop:loss:gradient}, \cite[Proposition 2.3]{JentzenRiekert2021} (applied with $a \with \scra$, $b \with \scrb$, $\mu \with ( \cB ( [ \scra , \scrb ] ^ d ) \ni A \mapsto \mu ( A ) [ \mu ( [ \scra , \scrb ] ^d ) ]^{-1} \in [0, 1 ] ) $ in the notation of \cite[Proposition 2.3]{JentzenRiekert2021}),
and \cite[Lemma 2.5]{JentzenRiekert2021} (applied with $a \with \scra$, $b \with \scrb$, $\mu \with ( \cB ( [ \scra , \scrb ] ^ d ) \ni A \mapsto \mu ( A ) [ \mu ( [ \scra , \scrb ] ^d ) ]^{-1} \in [0, 1 ] ) $ in the notation of \cite[Lemma 2.5]{JentzenRiekert2021}) establish \cref{lem:gradient:est:eq}.
\end{cproof}

\begin{cor} \label{cor:gradient:measurable}
Assume \cref{setting:snn}. Then it holds that $\cG$ is locally bounded and measurable.
\end{cor}
\begin{cproof2}{cor:gradient:measurable}
\Nobs that \cref{prop:loss:gradient:item1} in \cref{prop:loss:gradient} ensures that for all $r \in \N$ it holds that $\R^\fd \ni \theta \mapsto (\nabla \fL_r ) ( \theta ) \in \R^\fd$ is measurable. Combining this with \cref{prop:loss:gradient:item2} in \cref{prop:loss:gradient} demonstrates that $\cG$ is measurable. Moreover, \nobs that \cref{lem:realization:lip} and \cref{lem:gradient:est} assure that $\cG$ is locally bounded.
\end{cproof2}

\subsection{Continuous dependence of active neuron regions on ANN parameters}
\label{subsection:regions:continuous}

\cfclear
\begin{lemma} \label{lem:diff:2} 
Let $d \in \N$, $\scra \in \R$, $\scrb \in (\scra , \infty)$, let $I^u \subseteq [\scra , \scrb ]^d$, $u \in \R^{d + 1}$, satisfy for all $u = (u_1, \ldots, u_{d+1} ) \in \R^{d+1}$ that $I^u = \cu{ x = (x_1, \ldots, x_d) \in [\scra , \scrb]^d \colon u_{d+1} + \sum_{i=1}^d u_i x_i > 0 }$, 
for every $n \in \N$ let $\lambda_n \colon \cB ( \R^n ) \to [0, \infty]$ be the Lebesgue--Borel measure on $\R^n$,
and let $v \in \R^{d+1} \backslash \cu{ 0 }$. Then 
\begin{equation} \label{lem:diff:2:eq}
    \limsup\nolimits_{\R^{d+1} \ni u  \to v } \lambda_d ( I^u \Delta I ^v ) = 0.
\end{equation}
\end{lemma}
\begin{cproof} {lem:diff:2}
Throughout this proof let $\norm{ \cdot } \colon \rbr*{  \bigcup_{n \in \N} \R^n  } \to \R$ satisfy for all $n \in \N$, $x=(x_1, \ldots, x_n) \in \R^n $ that $\norm{ x } = [ \sum_{i=1}^n \abs*{ x_i } ^2 ] ^{1/2}$.
\Nobs that the fact that for all $y \in \R$ it holds that $y \geq - \abs{y}$ ensures that for all $u = (u_1, \ldots, u_{d+1} ) \in \R^{d+1}$, $i \in \cu{1, 2, \ldots, d+1 }$ with $\norm{u-v} < \abs{v_i}$ it holds that
\begin{equation} \label{lem:diff:2:eq:help1}
    u_i v_i = (v_i)^2 + ( u_i - v_i ) v_i \geq \abs{v_i}^2 - \abs{u_i - v_i} \abs{v_i} \geq \abs{v_i}^2 - \norm{u-v} \abs{v_i} > 0.
\end{equation}
In the following we distinguish between the case $\max_{i \in \cu{1, 2, \ldots, d } } \abs{v_i} = 0$,
the case $(\max_{i \in \cu{1, 2, \ldots, d } } \abs{v_i} , \allowbreak d ) \in (0, \infty ) \times [ 2 , \infty)$, 
and the case $(\max_{i \in \cu{1, 2, \ldots, d } } \abs{v_i} , d ) \in (0, \infty ) \times \cu{1}$.
We first prove \cref{lem:diff:2:eq} in the case \begin{equation} \label{lem:diff:2:eq:case1}
    \max\nolimits_{i \in \cu{1, 2, \ldots, d } } \abs{v_i} = 0.
\end{equation} 
\Nobs that \cref{lem:diff:2:eq:case1} and the assumption that $v \in \R^{d+1} \backslash \cu{ 0 }$ imply that $v_{d+1} \not= 0$.
Moreover, \nobs that \cref{lem:diff:2:eq:case1} shows that for all $u = (u_1, \ldots, u_{d+1} ) \in \R^{d+1}$, $x \in I^u \Delta I^v$ we have that
\begin{equation} \label{proof:diff:2:eq1}
\begin{split}
    &\abs[\big]{\rbr[\big]{ \br[\big]{\smallsum_{i=1}^d u_i x_i} + u_{d+1} } - \rbr[\big]{ \br[\big]{\smallsum_{i=1}^d v_i x_i} + v_{d+1} } } \\
    &=\abs[\big]{\br[\big]{ \smallsum_{i=1}^d u_i x_i} + u_{d+1} }  + \abs[\big]{  \br[\big]{ \smallsum_{i=1}^d v_i x_i} + v_{d+1} } 
    \geq \abs[\big]{ \br[\big]{\smallsum_{i=1}^d v_i x_i} + v_{d+1} } = \abs{v_{d+1} }.
\end{split}
\end{equation}
In addition, \nobs that for all $u = (u_1, \ldots, u_{d+1} ) \in \R^{d+1}$, $x \in [ \scra , \scrb ] ^d$ it holds that
\begin{equation} \label{proof:diff:2:eq2}
    \begin{split}
      &   \abs[\big]{ \rbr[\big]{ \br[\big]{\smallsum_{i=1}^d u_i x_i} + u_{d+1} } - \rbr[\big]{ \br[\big]{\smallsum_{i=1}^d v_i x_i} + v_{d+1} } } \leq \br[\big]{\smallsum_{i=1}^d  \abs{u_i - v_i} \abs{x_i}  } + \abs{ u_{d+1} - v_{d+1} } \\
         & \leq \max \cu{ \abs{\scra}, \abs{\scrb} } \br[\big]{\smallsum_{i=1}^d  \abs{u_i - v_i}  } + \abs{ u_{d+1} - v_{d+1} } 
         \leq ( 1 + d \max \cu{ \abs{\scra , \scrb } } ) \norm{u - v }.
    \end{split}
\end{equation}
Combining this with \cref{proof:diff:2:eq1} shows that for all $u \in \R^{d+1} $ with $\norm{ u - v } < \frac{\abs{v_{d+1}}}{1 + d \max \cu{ \abs{\scra , \scrb } }}$ it holds that $I^u \Delta I^v = \emptyset$. 
Hence, we obtain that $\limsup_{\R^{d+1} \ni u  \to v } \lambda_d ( I^u \Delta I ^v ) = 0$.
This establishes \cref{lem:diff:2:eq} in the case $\max_{i \in \cu{1, 2, \ldots, d } } \abs{v_i} = 0$.
In the next step we prove \cref{lem:diff:2:eq} in the case 
\begin{equation} \label{lem:diff:2:eq:case2}
    (\max\nolimits_{i \in \cu{1, 2, \ldots, d } } \abs{v_i} , d ) \in (0, \infty ) \times [ 2 , \infty ).
\end{equation}
For this we assume without loss of generality that $v_1 \not=0$.
In the following let $J_x^{u, w} \subseteq \R$, $x \in [\scra , \scrb]^{d-1}$,
$u, w \in \R^{d+1}$,
satisfy for all $x =(x_2, \ldots, x_d) \in [\scra , \scrb] ^{d-1}$, $u, w \in \R^{d+1}$ that $J_x^{u , w} = \cu{ y \in [\scra , \scrb ] \colon (y, x_2, \ldots, x_d) \in I^u \backslash I^w }$.
Next \nobs that Fubini's theorem and the fact that for all $u \in \R^{d+1}$ it holds that $I^u$ is measurable show that for all $u \in \R^{d+1}$ we have that
\begin{equation} \label{eq:est:fubini}
\begin{split}
    &\lambda_d ( I^u \Delta I^v ) = \int_{[\scra , \scrb ]^d } \indicator{I^u \Delta I^v} ( x ) \, \lambda_d ( \d x )
    =
    \int_{[\scra , \scrb ] ^d } \rbr[\big]{ \indicator{I^u \backslash I^v} ( x ) + \indicator{I^v \backslash I^u } ( x ) } \, \lambda_d ( \d x) \\
    &= \int_{[\scra , \scrb]^{d-1} } \int_{[\scra , \scrb ]} \rbr[\big]{ \indicator{I^u \backslash I^v} ( y, x_2, \ldots, x_d ) + \indicator{I^v \backslash I^u } ( y, x_2, \ldots, x_d ) } \, \lambda_1 ( \d y ) \, \lambda_{d-1} (\d ( x_2, \ldots, x_d ) ) \\
    &= \int_{[\scra , \scrb ] ^{d-1} } \int_{[\scra , \scrb ] } \rbr[\big]{ \indicator{J_x^{u,v}} ( y ) + \indicator{J_x^{v,u}} ( y ) } \, \lambda_1 ( \d y ) \, \lambda_{d-1} ( \d x ) \\
    &=
 \int_{[\scra , \scrb ]^{d-1} } ( \lambda_1 ( J_x^{u , v} ) + \lambda_1 ( J_x^{v , u } ) ) \, \lambda_{d-1} ( \d x).
\end{split}
\end{equation}
Moreover,
\nobs that for all $x = (x_2, \ldots, x_d) \in [ \scra , \scrb ] ^{d-1}$, $u = (u_1, \ldots, u_{d+1})$, $w = (w_1, \ldots, w_{d+1} ) \in \R^{d+1}$, $\fs \in \cu{ - 1 , 1 }$ with $\min \cu{\fs u_1, \fs w_1 } > 0$ it holds that
\begin{equation}
     \begin{split}
     J_x^{u , w} &= \cu*{y \in [\scra , \scrb ] \colon
        (y, x_2, \ldots, x_d) \in I^u \backslash I^w }\\
        &= \cu*{ y \in [ \scra , \scrb ] \colon  u_1 y + \br[\big]{ \smallsum_{i=2}^d u_i x_i } + u_{d+1} > 0 \geq w_1 y + \br[\big]{\smallsum_{i = 2}^d w_i x_i } + w_{d+1} }\\
        & = \cu*{y \in [\scra , \scrb ] \colon - \tfrac{\fs }{u_1} \rbr[\big]{ \br[\big]{  \smallsum_{i=2}^d u_i x_i } + u_{d+1} } < \fs y \leq -\tfrac{\fs }{w_1} \rbr[\big]{ \br[\big]{ \smallsum_{i = 2}^d w_i x_i } + w_{d+1} } } .
    \end{split}
\end{equation}
Hence, we obtain for all $x = (x_2, \ldots, x_d) \in [ \scra , \scrb ] ^{d-1}$, $u = (u_1, \ldots, u_{d+1})$, $w = (w_1, \ldots, w_{d+1} ) \in \R^{d+1}$, $\fs \in \cu{ - 1 , 1 }$ with $\min \cu{\fs u_1, \fs w_1 } > 0$ that
\begin{equation} \label{lem:diff:2:eq:help2}
\begin{split}
      \lambda_1 ( J_x^{u , w} ) &\leq \abs*{\tfrac{\fs }{u_1} \rbr[\big]{ \br[\big]{ \smallsum_{i=2}^d u_i x_i } + u_{d+1} } -\tfrac{\fs }{w_1} \rbr[\big]{ \br[\big]{ \smallsum_{i=2}^d w_i x_i } + w_{d+1} } } \\
        &\leq \br*{ \smallsum_{i=2}^d \abs[\big]{\tfrac{u_i}{u_1} - \tfrac{w_i}{w_1}} \abs{x_i} } + \abs*{\tfrac{u_{d+1}}{u_1} - \tfrac{w_{d+1}}{w_1} } \\
        &\leq \max \cu{ \abs{\scra} , \abs{\scrb} } \br*{ \smallsum_{i=2}^d  \abs[\big]{\tfrac{u_i}{u_1} - \tfrac{w_i}{w_1}} } + \abs*{\tfrac{u_{d+1}}{u_1} - \tfrac{w_{d+1}}{w_1} }.
    \end{split}
\end{equation}
Furthermore, \nobs that \cref{lem:diff:2:eq:help1} demonstrates for all $u = (u_1, \ldots, u_{d+1}) \in \R^{d+1}$ with $\norm{u - v} < \frac{\abs{v_1}}{2}$ that $u_1 v_1 > 0$. This implies that for all $u = (u_1, \ldots, u_{d+1}) \in \R^{d+1}$ with $\norm{u - v} < \frac{\abs{v_1}}{2}$ there exists $\fs \in \cu{-1 , 1 }$ such that  $\min \cu{\fs u_1, \fs v_1 } > 0$.
Combining this with \cref{lem:diff:2:eq:help2} proves that there exists $\fC \in \R$ such that for all $x \in [\scra , \scrb]^{d-1}$,
$u  \in \R^{d+1}$ with $\norm{u - v} < \frac{ \abs{ v_1 } }{2}$ we have that $\lambda_1( J_x^{u , v} ) + \lambda_1 ( J_x^{v , u} ) \leq \fC \norm{ u - v }$. This and \cref{eq:est:fubini} establish \cref{lem:diff:2:eq} in the case $(\max_{i \in \cu{1, 2, \ldots, d } } \abs{v_i} , d ) \in (0, \infty ) \times [ 2 , \infty )$.
Finally, we prove \cref{lem:diff:2:eq} in the case 
\begin{equation} \label{lem:diff:2:eq:case3}
    (\max\nolimits_{i \in \cu{1, 2, \ldots, d } } \abs{v_i} , d ) \in (0, \infty ) \times \cu{ 1 }.
\end{equation}
\Nobs that \cref{lem:diff:2:eq:case3} assures that $\abs{v_1} > 0$.
In addition, \nobs that for all $u = (u_1, u_2)$, $w = (w_1, w_2 ) \in \R^{2}$, $\fs \in \cu{ - 1 , 1 }$ with $\min \cu{\fs u_1, \fs w_1 } > 0$ it holds that
\begin{equation}
    \begin{split}
        I^w \backslash I^u &= \cu*{ y \in [ \scra , \scrb ] \colon w_1 y + w_2 > 0 \geq u_1 y + u_2} = \cu*{ y \in [\scra , \scrb ] \colon - \tfrac{\fs w_2}{w_1} < \fs y \leq - \tfrac{s u_2}{u_1} } \\
        & \subseteq \cu*{ y \in \R \colon - \tfrac{\fs w_2}{w_1} < \fs y \leq - \tfrac{s u_2}{u_1} }.
    \end{split}
\end{equation}
Hence, we obtain for all $u = (u_1, u_2)$, $w = (w_1, w_2 ) \in \R^{2}$, $\fs \in \cu{ - 1 , 1 }$ with $\min \cu{\fs u_1, \fs w_1 } > 0$ that 
\begin{equation} \label{lem:diff:2:eq:help3}
    \lambda_1 ( I^w \backslash I^u  ) \leq \abs*{ \rbr*{ - \tfrac{s u_2}{u_1}} - \rbr*{ - \tfrac{\fs w_2}{w_1} } } = \abs*{\tfrac{u_2}{u_1} - \tfrac{w_2}{w_1} }.
\end{equation}
Furthermore, \nobs that \cref{lem:diff:2:eq:help1} ensures for all $u = (u_1, u_2) \in \R^2$ with $\norm{u - v } < \abs{v_1}$ that $u_1 v_1 > 0$. This proves that for all $u = (u_1, u_2) \in \R^2$ with $\norm{u - v } < \abs{v_1}$ there exists $\fs \in \cu{-1 , 1 }$ such that $\min \cu{\fs u_1, \fs v_1 } > 0$.
Combining this with \cref{lem:diff:2:eq:help3} demonstrates for all $u = (u_1, u_2) \in \R^2$ with $\norm{u - v } < \abs{v_1}$ that
\begin{equation}
    \lambda_1 ( I^u \Delta I^v ) = \lambda_1 ( I^u \backslash I^v ) + \lambda_1 ( I^v \backslash I^u ) \leq 2 \abs*{\tfrac{u_2}{u_1} - \tfrac{v_2}{v_1} }.
\end{equation}
Hence, we obtain that 
\begin{equation}
    \limsup\nolimits_{\R^{2} \ni u \to v } \lambda_1 ( I^v \Delta I^u ) = 0.
\end{equation}
This establishes \cref{lem:diff:2:eq} in the case $(\max_{i \in \cu{1, 2, \ldots, d } } \abs{v_i} , d ) \in (0, \infty ) \times \cu{ 1 }$.
\end{cproof}

\begin{lemma} \label{lem:abs:continuous}
Let $(E, \cE)$ be a measurable space, let $\mu \colon \cE \to [0, \infty]$ and $ \nu \colon \cE \to [0, \infty]$ be measures, assume $\mu \ll \nu$ and $\mu ( E ) < \infty$, and let $\varepsilon \in (0, \infty)$. Then there exists $\delta \in (0, \infty)$ such that for all $A \in \cE$ with $\nu ( A ) < \delta$ it holds that $\mu ( A ) < \varepsilon$.
\end{lemma}
\begin{cproof}{lem:abs:continuous}
Throughout this proof assume for the sake of contradiction that there exists $A = (A_n)_{n \in \N} \colon \N \to \cE$ which satisfies for all $n \in \N$ that $\nu ( A_n) < 2^{-n}$ and $\mu ( A _n ) \ge \varepsilon$ and let $B_n \in \cE$, $n \in \N$, and $ C \in \cE$ satisfy for all $n \in \N$ that $B_n = \bigcup_{k=n}^\infty A_k$ and $C = \bigcap_{k=1}^\infty B_k$. \Nobs that the fact that for all $n \in \N$ it holds that $\nu ( A_n) < 2^{-n}$ ensures that for all $n \in \N$ we have that
\begin{equation}
    \nu ( B_n) = \nu \rbr*{\textstyle\bigcup_{k=n}^\infty A_k } \leq \smallsum_{k=n}^\infty \nu (A_k) \leq \smallsum_{k=n}^\infty 2^{-k} = 2^{-n} \rbr*{ \smallsum_{k=0}^\infty 2^{-k} } = 2^{1-n}.
\end{equation}
This implies that
\begin{equation}
    \nu ( C ) = \nu \rbr*{\textstyle\bigcap_{k=1}^\infty B_k } \leq \inf\nolimits_{k \in \N} \nu ( B_k ) \leq \inf\nolimits_{k \in \N} (2^{1-k} ) = 0.
\end{equation}
The assumption that $\mu \ll \nu$ hence shows that 
\begin{equation} \label{lem:abs:continuous:eq1}
\mu ( C ) = 0.    
\end{equation}
Moreover, \nobs that the fact that for all $n \in \N$ it holds that $\mu ( A _n ) \ge \varepsilon$ proves that for all $n \in \N$ we have that $\mu ( B_n) = \mu (\bigcup_{k=n}^\infty A_k ) \geq \varepsilon$. Combining this and \cref{lem:abs:continuous:eq1} with the fact that for all $n \in \N$ it holds that $B_n \supseteq B_{n + 1 }  $ and the fact that $\mu ( B_1) \leq \mu ( E ) < \infty$ demonstrates that
\begin{equation}
   0 = \mu ( C ) = \mu \rbr*{ \textstyle\bigcap_{k=1}^\infty B_k} = \lim\nolimits_{k \to \infty} \mu ( B_k ) \geq \varepsilon > 0.
\end{equation}
This is a contradiction.
\end{cproof}

\begin{cor} \label{cor:abs:continuous}
Let $(E, \cE)$ be a measurable space, let $\mu \colon \cE \to [0, \infty]$ and $\nu \colon \cE \to [0, \infty]$ be measures, assume $\mu \ll \nu$ and $\mu ( E ) < \infty$, and let $A_n \in \cE$, $n \in \N$, satisfy $\limsup_{n \to \infty} \nu ( A_n) = 0$. Then $\limsup_{n \to \infty} \mu ( A_n) = 0$.
\end{cor}
\begin{cproof} {cor:abs:continuous}
Throughout this proof let $\varepsilon \in (0, \infty)$.
\Nobs that \cref{lem:abs:continuous} proves that there exists $\delta \in (0, \infty)$ such that for all $B \in \cE$ with $\nu ( B ) < \delta$ it holds that $\mu ( B ) < \varepsilon$. Furthermore, \nobs that the assumption that $\limsup_{n \to \infty} \nu ( A_n) = 0$ ensures that there exists $N \in \N$ such that for all $n \in \N \cap [ N, \infty)$ it holds that $\nu ( A_n) < \delta$. Hence, we obtain for all $n \in \N \cap [ N, \infty)$ that $\mu ( A_n) < \varepsilon$.
\end{cproof}

\cfclear
\begin{cor} \label{cor:interval:continuous} 
Assume \cref{setting:snn}, let $\theta \in \R^\fd$, $i \in \cu{1, 2, \ldots, \width }$ satisfy $\abs{\b{\theta}_i} + \sum_{j = 1}^d \abs{\w{\theta}_{i,j}}  > 0$, and assume $\mu \ll \lambda $. Then $\limsup_{\R^\fd \ni \vartheta \to \theta} \mu ( I_i^\theta \Delta I_i^{\vartheta  } ) = 0$.
\end{cor}
\begin{cproof} {cor:interval:continuous}
Throughout this proof let $ \vartheta = (\vartheta_n)_{n \in \N} \colon \N \to \R^{\fd}$ satisfy $\limsup_{n \to \infty} \norm{\vartheta_n - \theta} = 0$. 
\Nobs that \cref{lem:diff:2} and the assumption that $\abs{\b{\theta}_i} + \sum_{j = 1}^d \abs{\w{\theta}_{i,j}}  > 0$ establish that $\limsup_{n \to \infty} \lambda ( I_i^\theta \Delta I_i^{\vartheta_n } ) = 0$.
Combining this, the assumption that $\mu \ll \lambda $, the fact that $\mu ( [ \scra , \scrb]^d ) < \infty$, and \cref{cor:abs:continuous} implies that $\limsup_{n \to \infty} \mu ( I_i^\theta \Delta I_i^{\vartheta _n} ) = 0$. 
\end{cproof}

\subsection{Differentiability of the risk function}
\label{subsection:risk:diff}

\cfclear
\begin{lemma} \label{lem:realization:lip}
Let $d, \width, \fd \in \N$, $ \scra \in \R$, $\scrb \in (\scra, \infty)$, $f \in C ( [\scra , \scrb]^d , \R)$ satisfy $\fd = d\width + 2 \width + 1$,
let $\scrN = (\realization{\theta})_{\theta \in \R^{\fd } } \colon \R^{\fd } \to C(\R ^d , \R)$
satisfy for all $\theta = ( \theta_1 , \ldots, \theta_\fd) \in \R^{\fd}$, $x = (x_1, \ldots, x_d) \in \R^d$ that
\begin{equation}
\realization{\theta} (x) = \theta_\fd + \smallsum_{i = 1}^\width \theta_{\width ( d+1 )  + i} \max \cu[\big]{\theta_{ \width d + i} + \smallsum_{j = 1}^d \theta_{ (i - 1 ) d + j} x_j , 0 } ,
\end{equation}
let $\mu \colon \cB ( [\scra , \scrb]^d ) \to [0, \infty  ]$ be a finite measure,
let $\norm{ \cdot } \colon \R^\fd \to \R$ and $\cL \colon \R^\fd \to \R$ satisfy for all $\theta =(\theta_1 , \ldots, \theta_{\fd} ) \in \R^\fd$ that $\norm{\theta} = [ \sum_{i=1}^\fd \abs*{ \theta_i } ^2 ] ^{1/2}$ and $\cL (\theta) = \int_{[\scra , \scrb]^d} (\realization{\theta} ( x ) - f ( x ) )^2 \, \mu ( \d x )$,
and let $K \subseteq \R^{\fd }$ be compact. Then there exists $\scrL \in \R$ such that for all $\theta, \vartheta \in K$ it holds that 
\begin{equation} \label{lem:realization:lip:eq}
\rbr[\big]{ \sup\nolimits_{x \in [\scra , \scrb]^d} \abs{ \realization{\theta} ( x ) - \realization{ \vartheta} ( x ) } } + \abs{ \cL( \theta ) - \cL ( \vartheta ) } \leq \scrL \norm{ \theta - \vartheta }.
\end{equation}
\end{lemma}

\begin{cproof}{lem:realization:lip}
Throughout this proof we distinguish between the case $\mu ( [ \scra , \scrb ] ^d ) = 0$ and the case $\mu ( [ \scra , \scrb ] ^d ) > 0$.
We first prove \cref{lem:realization:lip:eq} in the case 
\begin{equation} \label{eq:realization:lip:proof}
\mu ( [ \scra , \scrb ] ^d ) = 0.     
\end{equation}
\Nobs that \cref{eq:realization:lip:proof} ensures that for all $\theta \in \R^\fd$ it holds that $\cL ( \theta ) = 0$. Furthermore, \nobs that
 \cite[Lemma 2.4]{JentzenRiekert2021} (applied with $a \with \scra$, $b \with \scrb$, $\mu \with ( \cB ( [ \scra , \scrb ] ^ d ) \ni A \mapsto \indicator{A} ( a , a , \ldots, a ) \in [0, 1 ] ) $ in the notation of \cite[Lemma 2.4]{JentzenRiekert2021}) proves that there exists $\scrL \in \R$ such that for all $\theta, \vartheta \in K$ it holds that $\rbr{ \sup\nolimits_{x \in [\scra , \scrb]^d} \abs{ \realization{\theta} ( x ) - \realization{ \vartheta} ( x ) } } \leq \scrL \norm{\theta - \vartheta } $. 
 This establishes \cref{lem:realization:lip:eq} in the case $\mu ( [ \scra , \scrb ] ^d ) = 0$.
In the next step we prove \cref{lem:realization:lip:eq} in the case $\mu ( [ \scra , \scrb ] ^d ) > 0$. \Nobs that  \cite[Lemma 2.4]{JentzenRiekert2021} (applied with $a \with \scra$, $b \with \scrb$, $\mu \with ( \cB ( [ \scra , \scrb ] ^ d ) \ni A \mapsto \mu ( A ) [ \mu ( [ \scra , \scrb ] ^d ) ]^{-1} \in [0, 1 ] ) $ in the notation of \cite[Lemma 2.4]{JentzenRiekert2021}) establishes \cref{lem:realization:lip:eq} in the case $\mu ( [ \scra , \scrb ] ^d ) > 0$.
\end{cproof}

\begin{lemma} \label{lem:diff:1}
Let $d \in \N$, $w_1, w_2 \in \R^d$, $b_1, b_2, \scra \in \R$, $\scrb \in (\scra , \infty)$,
let $\norm{ \cdot } \colon  \R^d \to \R$ and $\spro{  \cdot , \cdot } \colon \R^d \times \R^d \to \R$ satisfy for all  $x=(x_1, \ldots, x_d)$, $y=(y_1, \ldots, y_d ) \in \R^d $ that $\norm{ x } = [ \sum_{i=1}^d \abs*{ x_i } ^2 ] ^{1/2}$ and $\spro{  x , y } = \sum_{i=1}^d x_i y_i$,
let $I_k \subseteq [\scra , \scrb ]^d$, $k \in \cu{1, 2 }$, satisfy for all $k \in \cu{1,2 }$ that $I_k = \cu{ x \in [\scra , \scrb]^d \colon \spro{  w_k , x }  + b_k > 0 }$, and let $x \in I_1 \Delta I_2$. Then
\begin{equation} 
\max_{k \in \cu{1, 2 }} \abs{ \spro{   w_k , x } + b_k } \leq \abs{ \spro{  w_1 - w_2 , x } + b_1 - b_2 } \leq \max \cu{ \abs{\scra} , \abs{\scrb} } \sqrt{d} \norm{w_1 - w_2 } + \abs{ b_1 - b_2 }.
\end{equation}
\end{lemma}
\begin{cproof2} {lem:diff:1}
Throughout this proof assume without loss of generality that $x \in I_1 \backslash I_2$.
\Nobs that the fact that
$  \spro{  w_2 , x } + b_2  \leq 0 < \spro{  w _1,  x } + b_1$ demonstrates that
\begin{equation}
 \spro{  w_2 - w_1 , x } + b_2 - b_1 <  \spro{  w_2 , x } + b_2  \leq  0 < \spro{  w_1 , x } + b_1 \leq  \spro{  w_1 - w_2 , x } + b_1 - b_2 .
\end{equation}
Hence, we obtain that $\max_{k \in \cu{1, 2 } } \abs{  \spro{  w _k , x } + b_k } \leq \abs{  \spro{  w_1 - w_2 , x } + b_1 - b_2 }$. Furthermore, \nobs that the Cauchy-Schwarz inequality and the fact that $x \in [\scra , \scrb]^d$ assure that 
\begin{equation}
    \abs{  \spro{  w_1 - w_2 , x } + b_1 - b_2 } \leq \norm{x} \norm{w_1 - w_2 } + \abs{b_1 - b_2 } \leq \max \cu{ \abs{\scra} , \abs{\scrb} } \sqrt{d} \norm{w_2 - w_1 } + \abs{b_2 - b_1 }.
\end{equation}
\end{cproof2}

\cfclear
\begin{prop} \label{prop:loss:differentiable}
Assume \cref{setting:snn}, assume $\mu \ll \lambda$, and let $\theta \in \R^{\fd }$ satisfy
\begin{equation} \label{eq:assumption:diff} 
  \cL ( \theta )\rbr[\big]{ \smallsum_{i = 1}^\width \abs{\v{\theta}_i } \indicator{ \cu{ 0 } } \rbr[\big]{\abs{\b{\theta}_i } + \smallsum_{j = 1}^d \abs{ \w{\theta}_{i,j} } } }  = 0.
\end{equation}
Then
\begin{enumerate} [label=(\roman*)]
    \item \label{prop:loss:diff:item1} it holds that $\cL$ is differentiable at $\theta$ and
    \item \label{prop:loss:diff:item2} it holds that $( \nabla \cL )( \theta ) = \cG ( \theta ) $. 
\end{enumerate} 
\end{prop}

\cfclear
\begin{cproof} {prop:loss:differentiable}
Throughout this proof let $M \in \R$ satisfy
\begin{equation} \label{prop:diff:proof:eq:defm}
    M = \inf \cu[\big]{ m \in \R \colon \mu \rbr[\big]{ \cu{ x \in [\scra , \scrb ] ^d \colon \abs{ \realization{\theta} ( x ) - f ( x ) } > m } } = 0 }
\end{equation}
and let $\fC \in \R$ satisfy
\begin{equation}\label{prop:diff:proof:eq:defc}
     \fC = 1 + d \max \cu{ \abs{\scra},  \abs{\scrb} }.
\end{equation}
We will prove \cref{prop:loss:diff:item1,prop:loss:diff:item2} by showing that
\begin{equation} \label{prop:diff:proof:eqclaim}
    \limsup\nolimits_{\R^\fd \backslash \cu{ 0 } \ni h \to 0} \br*{  \norm{ h }^  {-1} \abs{ \cL ( \theta + h ) - \cL ( \theta ) - \spro{  \cG ( \theta ) , h } } } = 0.
\end{equation}
\Nobs that \cref{prop:loss:gradient} ensures that for all $h \in \R^\fd$ it holds that
\begin{equation}
    \begin{split}
        \spro{\cG ( \theta ) , h } &= 2 \br*{ \sum_{i = 1}^\width \int_{I_i^\theta} \rbr[\big]{  \b{h}_i + \smallsum_{j = 1}^d \w{h}_{i,j} x_j } \v{\theta}_i ( \realization{\theta} ( x ) - f ( x ) ) \, \mu ( \d x ) } \\
        &+ 2 \br*{ \displaystyle\sum_{i = 1} ^\width \v{h}_i \int_{[\scra , \scrb]^d} \max \cu[\big]{ \b{\theta}_i + \smallsum_{j = 1}^d \w{\theta}_{i,j} x_j , 0 } ( \realization{\theta} ( x ) - f ( x ) ) \, \mu ( \d x ) } \\
        &+ 2 \c{h} \int_{[\scra , \scrb]^d} ( \realization{\theta} ( x ) - f ( x ) ) \, \mu ( \d x ) . 
    \end{split}
\end{equation}
Combining this and the fact that for all $\fx, \fy, \fz \in \R$ it holds that
\begin{equation}
\begin{split}
    (\fx - \fz ) ^2 - (\fy - \fz ) ^2 &= (\fx - \fy ) ( \fx + \fy - 2 \fz ) = (\fx - \fy ) ( ( \fx - \fy) + 2 ( \fy - \fz ) ) \\
    &= (\fx - \fy )^2 + 2 ( \fx - \fy ) ( \fy - \fz )
    \end{split}
\end{equation}
demonstrates that for all $h \in \R^\fd$ it holds that
\begin{equation}
    \begin{split}
        &\cL ( \theta + h ) - \cL ( \theta ) - \spro{  \cG ( \theta ) , h } \\
        &= \int_{[\scra , \scrb]^d} ( \realization{\theta + h} ( x ) - \realization{\theta} ( x ) ) ^2 \, \mu ( \d x ) \\
        &\quad + 2 \int_{[\scra , \scrb]^d} ( \realization{\theta + h} ( x ) - \realization{\theta} ( x ) ) ( \realization{\theta} ( x ) - f(x) ) \, \mu ( \d x ) - \spro{  \cG ( \theta ) , h } \\
        &= \int_{[\scra , \scrb]^d} ( \realization{\theta + h} ( x ) - \realization{\theta} ( x ) ) ^2 \, \mu ( \d x ) \\
        &\quad + 2 \int_{[\scra , \scrb]^d} \Bigl( \c{h} + \smallsum_{i = 1}^\width \bigl[  ( \v{\theta}_i + \v{h}_i ) \max  \cu[\big]{\b{\theta}_i + \b{h}_i + \smallsum_{j = 1}^d ( \w{\theta}_{i,j} + \w{h}_{i,j}) x_j , 0 } \bigr. \Bigr. \\
        & \qquad \Bigl. \bigl.- \v{\theta}_i \max \cu[\big]{ \b{\theta}_i + \smallsum_{j = 1}^d \w{\theta}_{i,j} x_j , 0 } \bigr] \Bigr) 
        ( \realization{\theta} ( x ) - f ( x ) ) \, \mu ( \d x ) \\
        &\quad - 2 \br*{ \sum_{i = 1}^\width \int_{I_i^\theta} \rbr[\big]{ \b{h}_i + \smallsum_{j = 1}^d \w{h}_{i,j} x_j } \v{\theta}_i ( \realization{\theta} ( x ) - f ( x ) ) \, \mu ( \d x ) } \\
        & \quad - 2 \br*{ \displaystyle\sum_{i = 1} ^\width \v{h}_i \int_{[\scra , \scrb]^d} \max \cu[\big]{ \b{\theta}_i + \smallsum_{j = 1}^d \w{\theta}_{i,j} x_j , 0 } ( \realization{\theta} ( x ) - f ( x ) ) \, \mu ( \d x ) } \\
        & \quad - 2 \c{h} \int_{[\scra , \scrb]^d} ( \realization{\theta} ( x ) - f ( x ) ) \, \mu ( \d x ) .
    \end{split}
\end{equation}
This shows for all $h \in \R^\fd$ that
\begin{equation}
    \begin{split}
          &\cL ( \theta + h ) - \cL ( \theta ) - \spro{  \cG ( \theta ) , h } = \int_{[\scra , \scrb]^d} ( \realization{\theta + h} ( x ) - \realization{\theta} ( x ) ) ^2 \, \mu ( \d x ) \\
          & +2 \br*{\sum_{i=1}^\width \int_{[\scra , \scrb ] ^d } ( \v{\theta}_i + \v{h}_i ) \rbr[\big]{\b{\theta}_i + \b{h}_i + \smallsum_{j = 1}^d ( \w{\theta}_{i,j} + \w{h}_{i,j}) x_j } ( \realization{\theta} ( x ) - f ( x ) ) \indicator{I_i^{\theta + h } } ( x ) \, \mu ( \d x ) } \\
          & - 2 \br*{ \sum_{i=1}^\width \int_{[\scra , \scrb]^d } \v{\theta}_i \rbr[\big]{\b{\theta}_i + \smallsum_{j=1}^d \w{\theta}_{i,j} x_j } ( \realization{\theta} ( x ) - f ( x ) ) \indicator{I_i^\theta} ( x ) \, \mu ( \d x ) } \\
          & - 2 \br*{ \sum_{i=1}^\width \int_{[\scra , \scrb]^d } \v{\theta}_i \rbr[\big]{\b{h}_i + \smallsum_{j=1}^d \w{h}_{i,j} x_j } ( \realization{\theta} ( x ) - f ( x ) ) \indicator{I_i^\theta} ( x ) \, \mu ( \d x ) } \\
          & - 2 \br*{ \sum_{i=1}^\width \int_{[\scra , \scrb]^d } \v{h}_i \rbr[\big]{\b{\theta}_i + \smallsum_{j=1}^d \w{\theta}_{i,j} x_j } ( \realization{\theta} ( x ) - f ( x ) ) \indicator{I_i^\theta} ( x ) \, \mu ( \d x ) } .
    \end{split}
\end{equation}
Hence, we obtain for all $h \in \R^\fd$ that
\begin{equation} 
    \begin{split}
         &\cL ( \theta + h ) - \cL ( \theta ) - \spro{  \cG ( \theta ) , h } = \int_{[\scra , \scrb]^d} ( \realization{\theta + h} ( x ) - \realization{\theta} ( x ) ) ^2 \, \mu ( \d x ) \\
        & + 2 \br*{\sum_{i = 1} ^\width \int_{[\scra , \scrb]^d}  \v{h}_i \rbr[\big]{ \b{h}_i + \smallsum_{j = 1}^d \w{h}_{i,j} x_j } (\realization{\theta} ( x ) - f ( x ) ) \indicator{I_i^{\theta +h}}(x) \, \mu ( \d x ) } \\
        &+ 2 \br*{ \sum_{i = 1}^\width \int_{[\scra , \scrb]^d} \v{\theta}_i \rbr[\big]{ \b{h}_i + \smallsum_{j = 1}^d \w{h}_{i,j} x_j  } (\realization{\theta} ( x ) - f ( x ) ) ( \indicator{I_i^{\theta + h}} ( x ) - \indicator{I_i^\theta} (x ) ) \, \mu ( \d x ) } \\
        &+ 2 \br*{ \sum_{i = 1}^\width \int_{[\scra , \scrb]^d} (\v{\theta}_i + \v{h}_i ) \rbr[\big]{ \b{\theta}_i + \smallsum_{j = 1}^d \w{\theta}_{i,j} x_j  } (\realization{\theta} ( x ) - f ( x ) ) ( \indicator{I_i^{\theta + h}} ( x ) - \indicator{I_i^\theta} (x ) ) \, \mu ( \d x ) } .
    \end{split}
\end{equation}
Combining this with the triangle inequality and \cref{prop:diff:proof:eq:defm} proves that for all $h \in \R^{ \fd} \backslash \cu{ 0 }$ we have that
\begin{equation} \label{prop:diff:proof:eq2}
    \begin{split}
        &\frac{ \abs{ \cL ( \theta + h ) - \cL ( \theta ) - \spro{  \cG ( \theta ) , h } } }{\norm{ h }} \leq \norm{ h }^{-1} \int_{[\scra , \scrb]^d} ( \realization{\theta + h} ( x ) - \realization{\theta} ( x ) ) ^2 \, \mu ( \d x ) \\
        &+ 2M\norm{ h } ^{-1} \br*{ \sum_{i = 1}^\width \int_{[\scra , \scrb]^d} \abs[\big]{ \v{h}_i \rbr[\big]{ \b{h}_i + \smallsum_{j = 1}^d \w{h}_{i,j} x_j } } \indicator{I_i^{\theta + h}}(x) \, \mu ( \d x ) } \\
        &+ 2  M \br*{ \displaystyle \sum_{i = 1}^\width \abs{ \v{\theta}_i } \int_{[\scra , \scrb]^d} \norm{h}^{-1} \abs[\big]{ \b{h}_i + \smallsum_{j = 1}^d \w{h}_{i,j} x_j }  \indicator{I_i^\theta \Delta I_i^{\theta + h} } ( x ) \, \mu ( \d x ) } \\
        &+ 2  M \br*{ \sum_{i = 1}^\width \abs{ \v{\theta}_i + \v{h}_i } \int_{[\scra , \scrb]^d} \norm{h}^{-1} \abs[\big]{ \b{\theta}_i + \smallsum_{j = 1}^d \w{\theta}_{i,j} x_j }  \indicator{I_i^\theta \Delta I_i^{\theta +h} } ( x ) \, \mu ( \d x ) } .
    \end{split}
\end{equation}
 Next \nobs that \cref{lem:realization:lip} ensures that there exists $\scrL \in \R$ such that for all $x \in [ \scra , \scrb ]^d$, $h \in \R^{\fd}$ with $\norm{ h } \leq 1$ it holds that 
 \begin{equation} \label{prop:diff:proof:eq:lipschitz}
 \abs{ \realization{\theta + h }  ( x ) - \realization{\theta} ( x ) } \leq \scrL \norm{ h }.    
 \end{equation}
 Furthermore, \nobs that \cref{lem:diff:1} (applied for every $i \in \cu{1, 2, \ldots, \width }$, $h \in \R^\fd$, $x \in I_i^{\theta + h} \Delta I_i^\theta$ with $d \with d$, $w_1 \with (\w{\theta + h}_{i,1}, \ldots, \w{\theta + h}_{i,d} )$,
 $w_2 \with (\w{\theta}_{i,1}, \ldots, \w{\theta}_{i,d} )$,
 $b_1 \with \b{\theta + h}_i$, $b_2 \with \b{\theta}_i$,
 $\scra \with \scra$, $\scrb \with \scrb$, $I_1 \with I_i^{\theta + h}$, $I_2 \with I_i^\theta$, $x \with x$ in the notation of \cref{lem:diff:1})
 and \cref{prop:diff:proof:eq:defc} show that for all $i \in \cu{1, 2, \ldots, \width }$, $h \in \R^\fd$,
$x \in I_i^{\theta + h} \Delta I_i^\theta$ it holds that 
\begin{equation} \label{prop:diff:proof:eq1}
\begin{split}
        \abs[\big]{ \b{\theta}_i + \smallsum_{j = 1}^d \w{\theta}_{i,j} x_j } 
        &\leq \abs[\big]{ \b{h}_i + \smallsum_{j = 1}^d \w{h}_{i,j} x_j }
        \leq \abs{ \b{h}_i } + \max \cu{ \abs{\scra} , \abs{\scrb } } \br[\big]{ \smallsum_{j = 1} ^d \abs{ \w{h}_{i,j} } } \\
        & \leq \norm{h} + d \max \cu{ \abs{\scra}, \abs{\scrb} } \norm{h} = \fC \norm{ h }.
\end{split}
\end{equation}
Moreover, \nobs that \cref{prop:diff:proof:eq:defc} implies that for all $i \in \cu{1, 2, \ldots, \width }$, $h \in \R^\fd \backslash \cu{ 0 }$,
$x \in [\scra , \scrb]^d$ it holds that 
\begin{equation}  \label{prop:diff:proof:eq4}
\begin{split}
    \norm{h}^{-1} \abs[\big]{ \v{h}_i \rbr[\big]{ \b{h}_i + \smallsum_{j = 1}^d \w{h}_{i,j} x_j } }
        &\leq \abs{ \b{h}_i } + \max \cu{ \abs{\scra} , \abs{\scrb } } \br[\big]{ \smallsum_{j = 1} ^d \abs{ \w{h}_{i,j} } } \\
        & \leq \norm{h} + d \max \cu{ \abs{\scra}, \abs{\scrb} } \norm{h} = \fC \norm{ h }.
    \end{split}
\end{equation}
This, \cref{prop:diff:proof:eq2}, \cref{prop:diff:proof:eq:lipschitz}, \cref{prop:diff:proof:eq1}, and the triangle inequality
demonstrate that for all $h \in \R^\fd \backslash \cu{ 0 }$ with $\norm{h} \leq 1$ it holds that
\begin{equation} 
    \begin{split}
    &\frac{ \abs{ \cL ( \theta + h ) - \cL ( \theta ) - \spro{  \cG ( \theta ) , h } } }{\norm{ h }} \\
        &\leq \scrL^2 \norm{ h } \br[\big]{ \mu ( [ \scra , \scrb ] ^d ) }
        + 2 M \br*{ \sum_{i = 1} ^\width \int_{[\scra , \scrb]^d} \norm{h}^{-1} \abs[\big]{ \v{h}_i \rbr[\big]{ \b{h}_i + \smallsum_{j = 1}^d \w{h}_{i,j} x_j } }  \, \mu ( \d x ) } \\
        & \qquad + 2 \fC M \br*{ \displaystyle \sum_{i = 1}^\width ( \abs{ \v{\theta} _i } + \abs{ \v{\theta}_i + \v{h}_i } ) \br[\big]{ \mu ( I_i^{\theta + h} \Delta I_i^\theta) } } \\
        & \leq \scrL^2 \norm{ h } \br[\big]{ \mu ( [ \scra , \scrb ] ^d ) } + 2 M \width \fC \norm{h} \br[\big]{ \mu ( [ \scra , \scrb ] ^d ) } + 2 \fC M \br*{ \displaystyle \sum_{i = 1}^\width ( 2 \abs{ \v{\theta} _i } + \abs{ \v{h}_i } ) \br[\big]{ \mu(I_i^{\theta + h} \Delta I_i^\theta) } } \\
        & \leq ( \scrL^2 + 2 M \width \fC ) \norm{h}  \br[\big]{ \mu ( [ \scra , \scrb ] ^d ) } + 2 \fC M \br*{ \sum_{i=1}^\width \rbr*{ 2 \abs{\v{\theta}_i }  \br[\big]{ \mu(I_i^{\theta + h} \Delta I_i^\theta) } + \norm{h}  \br[\big]{ \mu ( [ \scra , \scrb ] ^d ) } } } \\
        &= ( \scrL^2 + 4 M \width \fC ) \norm{h}  \br[\big]{ \mu ( [ \scra , \scrb ] ^d ) } + 4 \fC M \br*{\sum_{i=1}^\width  \abs{\v{\theta}_i }  \br[\big]{ \mu(I_i^{\theta + h} \Delta I_i^\theta) } }.
    \end{split}
\end{equation}
Hence, we obtain that
\begin{equation} \label{prop:diff:proof:eq:mainest}
    \begin{split}
        \limsup_{\R^\fd \backslash \cu{ 0 } \ni h \to 0 } \br*{ \frac{ \abs{ \cL ( \theta + h ) - \cL ( \theta ) - \spro{  \cG ( \theta ) , h } } }{\norm{ h } } } \leq 4 \fC M \br*{ \sum_{i=1}^\width \abs{\v{\theta}_i } \rbr[\bigg]{ \limsup_{\R^\fd \backslash \cu{ 0 } \ni h \to 0 }  \mu(I_i^{\theta + h} \Delta I_i^\theta) } }.
    \end{split}
\end{equation}
In the following we distinguish between the case $\cL ( \theta ) = 0$ and the case $\cL ( \theta ) > 0$.
We first prove \cref{prop:diff:proof:eqclaim} in the case 
\begin{equation} \label{prop:diff:proof:eq:loss0}
    \cL (\theta ) = 0.
\end{equation}
\Nobs that \cref{prop:diff:proof:eq:loss0} implies that for $\mu$-almost every $x \in [\scra , \scrb]^d$ it holds that $\realization{\theta}(x) = f(x)$. 
This and \cref{prop:diff:proof:eq:defc} show that $M = 0$. Combining this with \cref{prop:diff:proof:eq:mainest} establishes \cref{prop:diff:proof:eqclaim} in the case $\cL (\theta ) = 0$.
In the next step we prove \cref{prop:diff:proof:eqclaim} in the case 
\begin{equation} \label{prop:diff:proof:eq:losspos}
    \cL ( \theta ) > 0.
\end{equation}
\Nobs that \cref{eq:assumption:diff,prop:diff:proof:eq:losspos} ensure that for all $i \in \cu{1, 2, \ldots, \width }$ with $\abs{\v{\theta}_i} > 0$ 
it holds that
$ \abs{ \b{\theta}_i } + \sum_{j = 1}^d \abs{ \w{\theta} _{i,j} } > 0$. \cref{cor:interval:continuous} hence proves that for all $i \in \cu{1, 2, \ldots, \width }$ with $\abs{\v{\theta}_i } > 0$ we have that $\limsup_{\R^\fd \ni h \to 0 }  \mu(I_i^{\theta + h} \Delta I_i^\theta) = 0$.
Combining this with \cref{prop:diff:proof:eq:mainest}
establishes \cref{prop:diff:proof:eqclaim} in the case $\cL ( \theta ) > 0$.
\end{cproof}

\subsection{Lower semicontinuity of the norm of the gradient of the risk function}
\label{subsection:gradient:lsc}

\begin{lemma} \label{lem:vc:derivatives:cont}
Assume \cref{setting:snn} and let $j \in \N \cap (\width ( d+1) , \fd]$. Then it holds that $\R^\fd \ni \theta \mapsto \cG_j ( \theta ) \in \R$ is continuous.
\end{lemma}
\begin{cproof}{lem:vc:derivatives:cont}
Throughout this proof let $\vartheta \in \R^\fd$ and let $\theta = (\theta_n)_{n \in \N} \colon \N \to \R^\fd$ satisfy $\limsup_{n \to \infty} \norm{\theta_n - \vartheta } = 0$. 
\Nobs that \cref{lem:realization:lip} and the fact that $\R \ni x \mapsto \max \cu{x,0 } \in \R$ is continuous prove that for all $i \in \cu{1, 2, \ldots, \width }$, $x = (x_1, \ldots, x_d) \in [\scra , \scrb]^d$ it holds that
\begin{equation}
\begin{split} 
    &\lim_{n \to \infty} \rbr[\big]{ \br[\big]{\max \cu[\big]{ \b{\theta_n}_i + \smallsum_{k = 1}^d \w{\theta_n}_{i, k} x_k , 0 } } ( \realization{\theta_n}(x) - f ( x ) ) } \\
    &= \br[\big]{\max \cu[\big]{\b{\vartheta}_i + \smallsum_{k = 1}^d \w{\vartheta}_{i, k} x_k , 0 } } ( \realization{\vartheta}(x) - f ( x ) )
    \end{split}
\end{equation}
and
\begin{equation}
      \lim_{n \to \infty} ( \realization{\theta_n}(x) - f ( x ) ) = \realization{\vartheta}(x) - f ( x ).
\end{equation}
Combining \cref{eq:loss:gradient} and Lebesgue's dominated convergence theorem therefore establishes that $\limsup_{n \to \infty} \abs{ \cG_j ( \theta_n ) - \cG_j ( \vartheta) } = 0$. 
\end{cproof}

\cfclear
\begin{lemma} \label{lem:gradient:wbcomponents:cont} 
Assume \cref{setting:snn}, assume $\mu \ll \lambda$, and let $\theta \in \R^\fd$, $i \in \cu{1, 2, \ldots, \width }$ satisfy $\abs{\b{\theta}_i} + \sum_{j=1}^d \abs{\w{\theta}_{i,j}} > 0$. Then
\begin{enumerate} [ label = (\roman*)]
    \item it holds for all $j \in \cu{1, 2, \ldots, d }$ that $\R^{ \fd } \ni \vartheta \mapsto \cG _{(i-1) d + j} ( \vartheta )  \in \R$ is continuous at $\theta$ and
    \item it holds that $\R^{ \fd } \ni \vartheta \mapsto \cG _{\width d + i} ( \vartheta )  \in \R$ is continuous at $\theta$.
\end{enumerate}
\end{lemma}

\begin{cproof} {lem:gradient:wbcomponents:cont}
Throughout this proof let $j \in \cu{1, 2, \ldots, d }$. \Nobs that \cref{eq:loss:gradient} implies that for all $\vartheta \in \R^\fd$, $v \in \cu{0, 1 }$ we have that
\begin{equation} \label{lem:lsc:eq1}
    \begin{split}
     &  \abs*{ \br*{ \int_{I_i^{ \vartheta } } (x _j)^v ( \realization{\vartheta } (x) - f ( x ) ) \, \mu ( \d x ) } - \br*{ \int_{I_i^{\theta } } (x _j)^v ( \realization{\theta } (x) - f ( x ) ) \, \mu ( \d x ) } } \\
     &\leq \abs*{\int_{[\scra , \scrb]^d} (x_j)^v (\realization{\vartheta } ( x ) - \realization{\theta } ( x ) ) \indicator{I_i^{\vartheta } } ( x ) \, \mu ( \d x ) } \\
     & \quad + \abs*{\int_{[\scra , \scrb]^d} (x_j)^v ( \realization{\theta} ( x ) - f ( x ) ) ( \indicator{I_i^\theta} ( x ) - \indicator{I_i^{\vartheta  } } ( x ) ) \, \mu ( \d x )  } \\
     &\leq \br[\big]{ \sup\nolimits_{x \in [\scra , \scrb]^d } \abs{ (x_j)^v (\realization{\vartheta } ( x ) - \realization{\theta } ( x ) ) } } \mu ( [ \scra , \scrb ] ^d ) \\
     & \quad + \br[\big]{ \sup\nolimits_{x \in [\scra , \scrb]^d} \abs{(x_j)^v (\realization{\theta } ( x ) - f ( x ) ) } } \mu ( I_i^\theta \Delta I_i^{\vartheta } )  .
    \end{split}
\end{equation}
Next \nobs that \cref{lem:realization:lip} establishes that for all $v \in \cu{0,1 }$ it holds that 
\begin{equation} \label{lem:lsc:eq2}
\limsup\nolimits_{\R^\fd \ni \vartheta \to \theta} \rbr[\big]{ \sup\nolimits_{x \in [\scra , \scrb]^d } \abs{(x_j)^v (\realization{\vartheta  } ( x ) - \realization{\theta } ( x ) ) } } = 0.    
\end{equation}
Moreover, \nobs that the assumption that $\mu \ll \lambda$, the assumption that $\abs{\b{\theta}_i} + \sum_{k=1}^d \abs{\w{\theta}_{i,k}} > 0$, and \cref{cor:interval:continuous} imply that $\limsup_{\R^\fd \ni \vartheta \to \theta } \mu ( I_i^\theta \Delta I_i^{\vartheta } ) = 0$. Combining this with \cref{lem:lsc:eq1,lem:lsc:eq2} shows that for all $v \in \cu{0,1 }$ it holds that
\begin{equation}
    \R^\fd \ni \vartheta \mapsto \int_{I_i^{\vartheta } } (x _j)^v ( \realization{\vartheta } (x) - f ( x ) ) \, \mu ( \d x ) \in \R 
\end{equation}
is continuous at $\theta$. This and \cref{eq:loss:gradient} establish that $\cG_{(i-1) d + j}$ and $\cG_{\width d + i}$ are continuous at $\theta$.
\end{cproof}

\cfclear
\begin{lemma} \label{lem:gradient:components:lsc}
Assume \cref{setting:snn}
and let $\theta \in \R^\fd$, $i \in \cu{1, 2, \ldots, \width }$ satisfy $\abs{\b{\theta}_i} + \sum_{j=1}^d \abs{\w{\theta}_{i,j}} = 0$. Then 
\begin{enumerate} [label = (\roman*)]
    \item it holds for all $j \in \cu{1, 2, \ldots, d }$ that $\R^{ \fd } \ni \vartheta \mapsto \abs{\cG _{(i-1) d + j} ( \vartheta ) } \in \R$ is lower semicontinuous at $\theta$ and
    \item it holds that $\R^{ \fd } \ni \vartheta \mapsto \abs{\cG _{\width d + i} ( \vartheta ) } \in \R$ is lower semicontinuous at $\theta$.
\end{enumerate}  
\end{lemma}
\begin{cproof} {lem:gradient:components:lsc}
\Nobs that the assumption that $\abs{\b{\theta}_i} + \sum_{j=1}^d \abs{\w{\theta}_{i,j}} = 0$ proves that $I_i^\theta = \emptyset$. Combining this with \cref{eq:loss:gradient} shows that for all $j \in \cu{1, 2, \ldots, d }$ it holds that $\cG_{(i-1) d + j} ( \theta ) = \cG_{ \width d + i} ( \theta ) = 0$. Therefore, we obtain for all $j \in \cu{1, 2, \ldots, d }$ and all $\vartheta = (\vartheta_n )_{n \in \N} \colon \N \to \R^{ \fd }$ with $\limsup_{n \to \infty} \norm{\vartheta_n - \theta } = 0$ that 
\begin{equation}
    \abs{\cG_{(i-1) d + j} ( \theta)} = 0 \leq \liminf\nolimits_{n \to \infty} \abs{\cG_{(i-1) d + j} ( \vartheta_n)}
\end{equation}
and 
\begin{equation}
\abs{\cG_{\width d + i} ( \theta)} = 0 \leq \liminf\nolimits_{n \to \infty} \abs{\cG_{\width d + i} ( \vartheta_n)}.    
\end{equation}
Hence, we have for all $j \in \cu{1, 2, \dots, d } $ that $\R^{ \fd } \ni \vartheta \mapsto \abs{\cG _{(i-1) d + j} ( \vartheta ) } \in \R$ and  $\R^{ \fd } \ni \vartheta \mapsto \abs{\cG _{\width d + i} ( \vartheta ) } \in \R$ are lower semicontinuous at $\theta$.
\end{cproof}

\cfclear 
\begin{cor} \label{cor:gradient:components:lsc} 
Assume \cref{setting:snn}, assume $\mu \ll \lambda$, and let $k \in \cu{1, 2, \ldots, \fd }$. Then it holds that $\R^\fd \ni \theta \mapsto \abs{\cG_k ( \theta ) } \in \R$ is lower semicontinuous.
\end{cor}
\begin{cproof} {cor:gradient:components:lsc}
\Nobs that \cref{lem:vc:derivatives:cont} assures that for all $k \in \N \cap (\width ( d+1), \fd]$ it holds that $\R^\fd \ni \theta \mapsto \abs{\cG_k ( \theta ) } \in \R$ is lower semicontinuous.
Moreover, \nobs that \cref{lem:gradient:wbcomponents:cont,lem:gradient:components:lsc} prove that for all $k \in \N \cap [1, \width ( d+1)]$ it holds that $\R^\fd \ni \theta \mapsto \abs{\cG_k ( \theta ) } \in \R$ is lower semicontinuous.
\end{cproof}

\cfclear
\begin{cor} \label{cor:gradient.lsc} 
Assume \cref{setting:snn} and assume $\mu \ll \lambda$. Then it holds that $\R^{\fd } \ni \theta \mapsto \norm{\cG ( \theta ) }  \in \R$ is lower semicontinuous.
\end{cor}
\begin{cproof} {cor:gradient.lsc}
Throughout this proof let $\vartheta \in \R^\fd$ and let $\theta = (\theta_n)_{n \in \N} \colon \N \to \R^\fd$ satisfy $\limsup_{n \to \infty} \norm{\theta_n - \vartheta } = 0$. \Nobs that
\cref{cor:gradient:components:lsc} and the fact that for all $v = (v_{k,n})_{(k,n) \in \cu{1,2 } \times \N } \colon \cu{1, 2 } \times \N \to [0, \infty)$ it holds that
\begin{equation}
    \liminf\nolimits_{n \to \infty} (v_{1,n} + v_{2,n} ) \geq ( \liminf\nolimits_{n \to \infty} v_{1,n} ) + ( \liminf\nolimits_{n \to \infty} v_{2,n} )
\end{equation}
ensure that
\begin{equation}
    \begin{split}
        \liminf_{n \to \infty} \norm{\cG(\theta_n ) } ^2 &= \liminf_{n \to \infty} \br*{\smallsum_{j=1}^{ \fd } \abs{\cG_j(\theta_n)}^2 }
        \geq \smallsum_{j=1}^{ \fd } \br*{ \liminf_{n \to \infty}\abs{ \cG_j ( \theta_n ) }^2 } \\
        &\geq \smallsum_{j=1}^{ \fd } \abs{\cG_j ( \vartheta ) }^2 = \norm{ \cG ( \vartheta ) } ^2.
    \end{split}
\end{equation}
Hence, we obtain that $\norm{\cG ( \vartheta ) } \leq  \liminf_{n \to \infty} \norm{\cG(\theta_n ) } $. 
\end{cproof}

\begin{cor} \label{cor:loss:c1}
Assume \cref{setting:snn} and assume $\mu \ll \lambda$. Then there exists an open $U \subseteq \R^\fd$ such that $\int_{\R^\fd \backslash U } 1 \, \d x = 0$, $\cL | _U \in C^1 ( U , \R)$, and $\nabla (\cL |_U ) = \cG | _U$.
\end{cor}
\begin{cproof2} {cor:loss:c1}
Throughout this proof let $U \subseteq \R^\fd$ satisfy
\begin{equation} \label{cor:loss:c1:eq:defu}
    U = \cu[\big]{ \theta \in \R^\fd \colon \br[\big]{ \forall \, i \in \cu{1, 2, \ldots, \width } \colon  \rbr[\big]{ \abs{\b{\theta}_i } + \smallsum_{j=1}^d \abs{\w{\theta}_{i,j} } > 0  } } } .
\end{equation}
\Nobs that \cref{cor:loss:c1:eq:defu} ensures that $U \subseteq \R^\fd$ is open. Moreover, \nobs that the fact that $\R^\fd \backslash U \subseteq \rbr[\big]{ \bigcup_{i=1}^\width \cu{ \theta \in \R^\fd \colon \b{\theta}_i = 0 } }$ assures that $\int_{\R^\fd \backslash U } 1 \, \d x = 0$.
Furthermore, \nobs that \cref{prop:loss:differentiable} demonstrates that for all $\theta \in U$ it holds that $\cL$ is differentiable at $\theta$ with $( \nabla \cL )( \theta ) = \cG ( \theta ) $. In addition, \nobs that \cref{lem:vc:derivatives:cont} and \cref{lem:gradient:wbcomponents:cont} prove that for all $\theta \in U$, $i \in \cu{1, 2, \ldots, \fd }$ it holds that $\R^\fd \ni \vartheta \mapsto \cG_i ( \vartheta ) \in \R$ is continuous at $\theta$. Hence, we obtain that $\cL | _U \in C^1 ( U , \R)$.
\end{cproof2}

\section{Convergence of the risk of gradient flows (GFs) in the training of ANNs}
\label{section:gf:convergence}

In this section we establish in \cref{theo:flow:conditional} in \cref{subsection:gf:critical} below, 
in \cref{cor:flow:conditional} in \cref{subsection:gf:convergence:minimum} below,
 and in \cref{cor:gradient:flow:convergence:affine} in \cref{subsection:gf:convergence:affine} below convergence results for the risk of GFs.
 In particular, in \cref{theo:flow:conditional} we establish that the risk of every bounded GF trajectory converges to the risk of a critical point. 
Our proof of \cref{theo:flow:conditional} employs the fundamental theorem of calculus type result for the risk of GFs in \cref{lem:loss:integral} in \cref{subsection:gf:critical} and the fundamental fact that the standard norm of the generalized gradient function $\cG \colon \R^\fd \to \R^\fd$ is lower semicontinuous, which we established in \cref{cor:gradient.lsc} above.
 The proof of \cref{lem:loss:integral} is entirely analogous to the proof of \cite[Lemma 3.5]{CheriditoJentzenRiekert2021}.
 In \cref{cor:flow:conditional} we establish that the risk of every bounded GF trajectory with sufficiently small initial risk converges to the risk of the global minima of the risk function.  In \cref{cor:gradient:flow:convergence:affine} we employ the characterization result for criticial points for affine linear target functions in Cheridito et al.~\cite{CheriditoJentzenRossmannek2021} to specialize \cref{cor:flow:conditional} to the situation of affine linear target functions. 

\subsection{Convergence of the risk of GFs to the risk of a critical point}
\label{subsection:gf:critical}

\cfclear
\begin{lemma} \label{lem:loss:integral} 
Assume \cref{setting:snn}, let $T \in (0, \infty)$, and let $\Theta \in C([0, T ] , \R^{ \fd } )$ satisfy for all $t \in [0,T] $ that $\Theta_t = \Theta_0 - \int_0^t \cG ( \Theta_s ) \, \d s$ \cfadd{cor:gradient:measurable}\cfload.
Then it holds for all $t \in [0,T]$ that $\cL(\Theta_t) = \cL (\Theta_0) - \int_0^t \norm{ \cG( \Theta_s ) } ^2 \, \d s$.
\end{lemma}
\begin{proof}[Proof of \cref{lem:loss:integral}]
The proof of \cref{lem:loss:integral} is entirely analogous to the proof of \cite[Lemma 3.5]{CheriditoJentzenRiekert2021}.
\end{proof}

\cfclear
\begin{theorem} \label{theo:flow:conditional}
Assume \cref{setting:snn}, assume $\mu \ll \lambda$, and let $\Theta \in C([0, \infty), \R^{\fd})$ satisfy for all $t \in [0, \infty)$ that $\sup_{s \in [0, \infty ) } \norm{\Theta_s } < \infty$ and 
\begin{equation}
 \Theta_t = \Theta_0 - \int_0^t \cG ( \Theta_s ) \, \d s
\end{equation}
\cfadd{cor:gradient:measurable}\cfload.
 Then there exists $\vartheta \in \cG^{-1} ( \cu{ 0 } )$ such that $\limsup_{t \to \infty} \cL ( \Theta_t ) = \cL ( \vartheta )$.
\end{theorem}
\begin{cproof}{theo:flow:conditional}
\Nobs that \cref{lem:loss:integral} implies that $\int_0^\infty \norm{\cG ( \Theta_s ) } ^2 \, \d s < \infty$. Hence, we have that $\liminf_{t \to \infty} \norm{\cG ( \Theta_t ) } = 0$. This proves that there exists $\tau = (\tau_n)_{n \in \N} \colon \N \to [0, \infty)$ which satisfies $\liminf_{n \to \infty} \tau_n = \infty$ and 
\begin{equation} \label{theo:flow:conditional:proof:eq1}
\limsup\nolimits_{n \to \infty} \norm{\cG ( \Theta_{\tau_n} ) } = 0.    
\end{equation}
\Nobs that the fact that $\sup_{n \in \N} \norm{\Theta_{\tau_n} } \leq \sup_{t \in [0, \infty)} \norm{\Theta_t }  < \infty$ ensures that there exist $\vartheta \in \R^\fd$ and a strictly increasing $n \colon \N \to \N$ which satisfies 
\begin{equation} \label{theo:flow:conditional:proof:eq2}
    \limsup\nolimits_{k \to \infty} \norm{ \Theta_{\tau_{n (k)}} - \vartheta } = 0.
\end{equation}
\Nobs that \cref{theo:flow:conditional:proof:eq1}, \cref{theo:flow:conditional:proof:eq2},
and \cref{cor:gradient.lsc} demonstrate that 
\begin{equation}
    \norm{\cG ( \vartheta ) } \leq \liminf\nolimits_{k \to \infty} \norm{\cG ( \Theta_{\tau_{n ( k )}} ) } = 0.
\end{equation}
Furthermore, \nobs that \cref{lem:loss:integral} assures that $[0, \infty) \ni t \mapsto \cL ( \Theta_t) \in \R$ is non-increasing. 
Combining this and \cref{theo:flow:conditional:proof:eq2} with \cref{lem:realization:lip} proves that $\limsup_{t \to \infty} \cL ( \Theta_t ) =\lim_{k \to \infty} \cL ( \Theta_{\tau_{n ( k ) } } ) = \cL ( \vartheta )$.
\end{cproof}

\subsection{Convergence of the risk of GFs to the minimal risk}
\label{subsection:gf:convergence:minimum}

\cfclear
\begin{cor} \label{cor:flow:conditional} 
Assume \cref{setting:snn}, assume $\mu \ll \lambda$, 
let $\mathbf{m} \in \R$ satisfy $\mathbf{m} = \inf\nolimits_{\theta \in \R^{ \fd }} \cL ( \theta) $,
and let $\Theta \in C([0, \infty), \R^{\fd})$ satisfy $\sup_{t \in [0, \infty ) } \norm{\Theta_t } < \infty$, 
 $\forall \, t \in [0, \infty) \colon \Theta_t = \Theta_0 - \int_0^t \cG ( \Theta_s ) \, \d s$,
 and $\forall \, \theta \in \cG^{-1} ( \cu{ 0 } ) \cap \cL^{-1} ( ( \mathbf{m} , \infty ) ) \colon \inf_{t \in [0, \infty ) } \cL ( \Theta_t ) < \cL ( \theta )$
\cfadd{cor:gradient:measurable}\cfload.
Then 
\begin{equation} \label{eq:cor:flow:conditional}
    \limsup\nolimits_{t \to \infty} \cL ( \Theta_t ) = \mathbf{m}.
\end{equation}
\end{cor}
\begin{cproof}{cor:flow:conditional}
\Nobs that \cref{theo:flow:conditional} assures that there exists $\vartheta \in \R^\fd$ which satisfies $\cG ( \vartheta ) = 0 $ and $\limsup_{t \to \infty} \cL ( \Theta_t ) = \cL ( \vartheta )$.
In the following we prove \cref{eq:cor:flow:conditional} by contradiction.
We thus assume that
\begin{equation} \label{proof:flow:conditional:eq1}
    \cL ( \vartheta ) > \mathbf{m}.
\end{equation}
\Nobs that \cref{proof:flow:conditional:eq1} and the assumption that $\forall \, \theta \in \cG^{-1} ( \cu{ 0 } ) \cap \cL^{-1} ( ( \mathbf{m} , \infty ) ) \colon \inf_{t \in [0, \infty ) } \cL ( \Theta_t ) < \cL ( \theta )$ imply that 
\begin{equation}  \label{proof:flow:conditional:eq2}
    \inf\nolimits_{t \in [0, \infty ) } \cL ( \Theta_t ) < \cL ( \vartheta) .
\end{equation}
Moreover, \nobs that \cref{lem:loss:integral} proves that $[0, \infty) \ni t \mapsto \cL ( \Theta_t) \in \R$ is non-increasing. 
Combining this with \cref{proof:flow:conditional:eq2} shows that
\begin{equation}
     \cL ( \vartheta ) = \limsup\nolimits_{t \to \infty} \cL ( \Theta_t ) = \inf\nolimits_{t \in [0, \infty ) } \cL ( \Theta_t ) < \cL ( \vartheta ) .
\end{equation}
This is a contradiction.
\end{cproof}

\subsection{Risks of critical points for affine linear target functions}
\label{subsection:critical:points:affine}

\begin{prop} \label{prop:loss:bounded:affine}
Assume \cref{setting:snn},
assume $d=1$,  
and let $\alpha , \beta \in \R$, $\rho \in (0, \infty)$ satisfy for all $E \in \cB ( [ \scra , \scrb ] )$, $x \in [\scra , \scrb]$
that $\mu ( E ) = \rho \lambda_1 ( E )$ and $ f(x)= \alpha x + \beta$. 
Then
\begin{enumerate} [label=(\roman*)]
    \item \label{cor:loss:bounded:affine:item1} there exists $\vartheta \in \R^\fd$ such that $ \cL ( \vartheta ) = \inf_{\theta \in \R^\fd} \cL ( \theta ) = 0$ and
    \item \label{cor:loss:bounded:affine:item2} it holds for all $\theta \in \cG^{-1} ( \cu{ 0 } ) \cap \cL^{-1} ( ( 0 , \infty ) ) $ that $\cL ( \theta )  \geq  \frac{ \rho \alpha^2 (\scrb - \scra)^3 }{12 (2 \lfloor \width / 2 \rfloor + 1)^4}$.
\end{enumerate}
\end{prop}

\begin{cproof} {prop:loss:bounded:affine}
\Nobs that the assumption that $d=1$ implies that $\fd = 3 \width + 1 $.
Let $\psi \in \R^{3 \width + 1 }$ satisfy $ \w{\psi}_{1 , 1} = 1$, $\b{\psi} _1  = - \scra$, $\v{\psi}_1  = \alpha$, $\c{\psi} _1 = \beta + \alpha \scra $, and $\forall \, i \in \N \cap ( 1 , \width ] \colon \w{\psi}_{i,1}=\b{\psi}_i = \v{\psi}_i = 0$. \Nobs that for all $x \in [\scra , \scrb ]$ we have that
\begin{equation}
    \realization{\psi} ( x ) = \alpha \max \cu{x -\scra , 0 } + \beta + \alpha \scra
    = \alpha ( x - \scra ) + \alpha \scra + \beta = \alpha x + \beta = f ( x ).
\end{equation}
This shows that $\cL ( \psi ) = 0$. Combining this with the fact that for all $\theta \in \R^\fd$ it holds that $\cL ( \theta ) \ge 0$ establishes \cref{cor:loss:bounded:affine:item1}. 
We now prove \cref{cor:loss:bounded:affine:item2}. 
For this assume in the following without loss of generality that $\alpha \not= 0$ and let $\fG = ( \fG_1 , \ldots, \fG_\fd) \colon \R^\fd \to \R^\fd$ satisfy for all $\theta \in \R^\fd$, $i \in \cu{1, 2, \ldots, \width }$ that
\begin{equation}
    \begin{split}
        \fG_i ( \theta ) &= 2 \v{\theta}_i \int_\scra ^\scrb x ( \realization{\theta} ( x ) - f ( x ) ) \indicator{[0, \infty)} ( \w{\theta}_{i,1} x + \b{\theta}_i ) \, \d x, \\
        \fG_{\width + i } ( \theta ) &= 2 \v{\theta}_i \int_\scra ^\scrb ( \realization{\theta} ( x ) - f ( x ) ) \indicator{[0, \infty)} ( \w{\theta}_{i,1} x + \b{\theta}_i ) \, \d x, \\
        \fG_{2 \width + i} ( \theta ) &= 2 \int_{\scra}^{\scrb} \br[\big]{ \max \cu{ \w{\theta}_{i,1} x + \b{\theta}_i , 0 } } ( \realization{\theta} ( x ) - f ( x ) ) \, \d x, \\
         \text{and} \qquad \fG_\fd ( \theta ) &= 2 \int_{\scra}^{\scrb} ( \realization{\theta} ( x ) - f ( x ) ) \, \d x
    \end{split}
\end{equation}
(cf., e.g., \cite[Lemma 3.5]{CheriditoJentzenRossmannek2021}).
\Nobs that \cref{eq:loss:gradient} and the assumption that for all $E \in \cB ( [ \scra , \scrb ] )$ it holds that $\mu ( E ) = \rho \lambda_1 ( E )$ show that for all $\theta \in \R^\fd$, $i \in \cu{1, 2, \ldots, \width }$ it holds that $\fG_{2 \width + i } ( \theta ) = \rho^{-1} \cG _{2 \width + i } ( \theta )$ and $\fG_\fd ( \theta ) = \rho^{-1} \cG_\fd ( \theta )$.
In the following let $\theta \in \R^{3 \width + 1 }$ satisfy $\cL ( \theta ) > 0 = \norm{ \cG ( \theta ) }$ and let $\vartheta \in \R^{3 \width + 1 }$ satisfy for all $i \in \cu{1, 2, \ldots, \width }$ that 
\begin{equation} \label{cor:loss:bounded:affine:eq:defvartheta}
    \w{\vartheta}_{i,1} = \w{\theta}_{i,1}, \qquad \b{\vartheta}_i = \b{\theta}_i,
    \qquad \v{\vartheta}_i = \v{\theta}_i \indicator{(0, \infty ) } ( \abs{ \w{\theta}_{i,1 } } + \abs{\b{\theta}_i } ) , \qandq
\c{\vartheta} = \c{\theta}.
\end{equation}
\Nobs that \cref{cor:loss:bounded:affine:eq:defvartheta} ensures that
\begin{equation}
    \realization{\vartheta} = \realization{\theta}, \qquad  
\cL ( \vartheta ) = \cL ( \theta ) , \qandq \cG ( \vartheta ) = \cG ( \theta ) = 0.
\end{equation}
Furthermore, \nobs that the fact that for all $i \in \cu{1, 2, \ldots, \width }$ it holds that $\fG_{2 \width + i } ( \vartheta ) = \rho^{-1} \cG _{2 \width + i } ( \vartheta )$ and $\fG_\fd ( \vartheta ) = \rho^{-1} \cG_\fd ( \vartheta )$ assures that for all $i \in \cu{1, 2, \ldots, \width }$ it holds that 
\begin{equation} \label{proof:loss:bounded:affine:eq1}
    \fG_{2 \width + i } ( \vartheta ) = \fG_\fd ( \vartheta ) = 0.
\end{equation}
Next \nobs that the fact that for all $i \in \cu{ 1, 2, \ldots, \width }$ with $\abs{\w{\vartheta}_{i,1}} + \abs{\b{\vartheta}_i } = 0$ it holds that $\v{\vartheta}_i = 0$ implies that for all $i \in \cu{ 1, 2, \ldots, \width }$ with $\abs{\w{\vartheta}_{i,1}} + \abs{\b{\vartheta}_i } = 0$ we have that
\begin{equation} \label{proof:loss:bounded:affine:eq2}
\fG_i ( \vartheta ) = \fG_{\width + i} ( \vartheta ) = 0.    
\end{equation}
In addition, \nobs that for all $i \in \cu{ j  \in \cu{ 1, 2, \ldots, \width } \colon \abs{ \w{\vartheta}_{j,1 } } + \abs{\b{\vartheta}_j } > 0}$ and almost all $x \in [\scra, \scrb]$
it holds that $\indicator{[0, \infty)} ( \w{\vartheta}_{i,1} x + \b{\vartheta}_i ) = \indicator{(0, \infty)} ( \w{\vartheta}_{i,1} x + \b{\vartheta}_i ) = \indicator{I_i^\vartheta} ( x )$. This shows that for all $i \in \cu{ 1, 2, \ldots, \width }$ with $\abs{ \w{\vartheta}_{i,1 } } + \abs{\b{\vartheta}_i } > 0$ it holds that
\begin{equation} \label{proof:loss:bounded:affine:eq3}
    \fG_i ( \vartheta ) = \rho^{-1} \cG_i ( \vartheta ) = 0 \qandq \fG_{\width + i} ( \vartheta ) = \rho^{-1} \cG_{\width + i } ( \vartheta ) = 0.
\end{equation}
Combining \cref{proof:loss:bounded:affine:eq1,proof:loss:bounded:affine:eq2,proof:loss:bounded:affine:eq3} demonstrates that $\fG ( \vartheta ) = 0$.
Cheridito et al.~\cite[Corollary 2.7]{CheriditoJentzenRossmannek2021} hence proves that there exists $n \in \cu{0, 2, 4,  \ldots } \cap (0, \width]$ which satisfies
\begin{equation} \label{eq:loss:explicit:value}
\cL ( \theta ) = \cL ( \vartheta ) = \rho \int_\scra^\scrb ( \realization{\vartheta} ( x ) - (\alpha x + \beta ))^2 \, \d x = \frac{\rho \alpha^2 (\scrb - \scra ) ^3 }{12 (n+1)^4}  .   
\end{equation}
\Nobs that the fact that $\frac{n}{2} \in \Z$ and the fact that $n \leq \width$ assure that $n \leq 2 \lfloor H/2 \rfloor$.
Combining this with \cref{eq:loss:explicit:value} shows that
\begin{equation}
    \cL ( \theta ) = \frac{\rho \alpha^2 (\scrb - \scra ) ^3 }{12 (n+1)^4} \geq \frac{ \rho \alpha^2 (\scrb - \scra ) ^3 }{12 (2 \lfloor \width / 2 \rfloor + 1)^4}.
\end{equation}
This establishes \cref{cor:loss:bounded:affine:item2}.
\end{cproof}

\subsection{Convergence of the risk of GFs to the minimal risk for affine linear target functions}
\label{subsection:gf:convergence:affine}

\cfclear
\begin{cor} \label{cor:gradient:flow:convergence:affine}
Assume \cref{setting:snn}, assume $d=1$,
let $\alpha , \beta \in \R$, $\rho \in (0, \infty)$ satisfy 
for all $E \in \cB ( [ \scra , \scrb ] )$, $x \in [\scra , \scrb]$
that $\mu ( E ) = \rho \lambda_1 ( E )$
and $ f(x)= \alpha x + \beta$,
and let $\Theta \in C([0, \infty), \R^{\fd})$ satisfy $\sup_{t \in [0, \infty ) } \norm{\Theta_t } < \infty$, $\forall \, t \in [0, \infty ) \colon  \Theta_t = \Theta_0 - \int_0^t \cG ( \Theta_s ) \, \d s$, and $\inf_{t \in [0, \infty ) } \cL ( \Theta_t ) < \frac{\rho \alpha^2 ( \scrb - \scra ) ^3}{12(2 \lfloor \width / 2 \rfloor + 1 )^4}  $
\cfadd{cor:gradient:measurable}\cfload.
Then 
\begin{equation}
    \limsup\nolimits_{t \to \infty} \cL ( \Theta_t ) = 0.
\end{equation}
\end{cor}
\begin{cproof}{cor:gradient:flow:convergence:affine}
\Nobs that \cref{cor:loss:bounded:affine:item1} in \cref{prop:loss:bounded:affine} implies
that $\inf_{\theta \in \R^\fd} \cL ( \theta ) = 0$.
Moreover, \nobs that \cref{cor:loss:bounded:affine:item2} in \cref{prop:loss:bounded:affine} demonstrates that for all $\theta \in \cG^{-1} ( \cu{ 0 } ) \cap \cL^{-1} ( (0 , \infty ) )$ we have that
\begin{equation}
\cL ( \theta ) \geq \frac{\rho \alpha^2 ( \scrb - \scra ) ^3}{12(2 \lfloor \width / 2 \rfloor + 1 )^4} > \inf_{t \in [0, \infty ) } \cL ( \Theta_t ).
\end{equation}
Combining this and \cref{cor:flow:conditional} (applied with $\bfm \with 0$ in the notation of \cref{cor:flow:conditional}) establishes that $\limsup_{t \to \infty} \cL ( \Theta_t ) = 0$.
\end{cproof}

\section{A priori estimates for GFs in the training of ANNs}
\label{section:apriori:bounds}

In this section we establish in \cref{prop:gen:apriori} in \cref{subsection:lyapunov:gf} below, in \cref{cor:flow:limsup:loss} in \cref{subsection:lyapunov:gf}, in \cref{cor:flow:bounded} in \cref{subsection:apriori:largerisk} below, and in \cref{prop:flow:invariant} in \cref{subsection:gf:invariant} several general a priori estimates for GF trajectories.
 In particular, \cref{cor:flow:limsup:loss} demonstrates that the limit value of the risk of every GF trajectory is bounded by the squared $L^2$-error $\inf_{\xi \in \R} \br{ \int_{[\scra , \scrb]^d} ( f(x) - \xi ) ^2 \, \mu ( \d x )}$ of constant approximations of the target function $f \colon [\scra , \scrb ] ^d \to \R$.
 Our proof of \cref{cor:flow:limsup:loss} is based on an application of the a priori estimate in \cref{prop:gen:apriori}. 
\cref{cor:flow:bounded}, in particular, proves that the norm of every GF trajectory is bounded until the first time where the risk is smaller than $\inf_{\xi \in \R} \br{ \int_{[\scra , \scrb]^d} ( f(x) - \xi ) ^2 \, \mu ( \d x )}$.
 Our proof of \cref{cor:flow:bounded} also employs an application of \cref{prop:gen:apriori}.
A result similar to \cref{prop:gen:apriori} has been obtained in \cite[Lemma 3.2]{CheriditoJentzenRiekert2021} in the special situation where the measure $\mu$ is the Lebesgue--Borel measure on $[0 , 1]$
 and where the target function $f$ is a constant function,
and our proof of \cref{prop:gen:apriori} uses similar ideas as the proof of \cite[Lemma 3.2]{CheriditoJentzenRiekert2021}.

 In \cref{prop:flow:invariant} we identify appropriate invariant quantities for the GF dynamics.
 In the scientific literature \cref{prop:flow:invariant} has already been asserted and proved in Williams et al.~\cite[Lemma 3]{WilliamsTragerPanozzo2019} in the case where the measure $ \mu $ is chosen in a way so that the function $\cL \colon \R^\fd \to \R$ describes the empirical risk and where the input is $1$-dimensional (where $ d = 1 $).
 Moreover, a result similar to \cref{prop:flow:invariant} has also been established in Du et al.~\cite[Theorem 2.1]{DuWeiLee2018} in the situation of deep ANNs without biases.

\subsection{Lyapunov type functions for GFs}
\label{subsection:lyapunov:gf}

\cfclear
\begin{prop} \label{prop:gen:apriori}
Assume \cref{setting:snn}, let $\xi \in \R$, let $V \colon \R^\fd \to \R$ satisfy for all $\theta \in \R^\fd $ that $V ( \theta ) = \norm{\theta} ^2 + \abs{ \c{\theta} - 2 \xi } ^2$, and let $\Theta \in C([0, \infty) , \R^{\fd})$ satisfy for all $t \in [0, \infty)$ that $\Theta_t = \Theta_0 - \int_0^t \cG ( \Theta_s ) \,  \d s$ \cfadd{cor:gradient:measurable}\cfload.
Then 
 it holds for all $t \in [ 0, \infty)$ that
\begin{equation} \label{prop:gen:apriori:eq1}
\begin{split}
V( \Theta_t ) &= V(\Theta_0) - 8 \int_0^t \cL ( \Theta_s ) \, \d s - 8 \int_0^t \br*{ \int_{[\scra , \scrb]^d} (f(x) - \xi ) ( \realization{\Theta_s} ( x ) - f ( x ) ) \, \mu ( \d x )  } \, \d s \\
&\leq  V ( \Theta_0 ) + 4 \int_0^t \br*{ \int_{[\scra , \scrb]^d} (f(x) - \xi ) ^2 \, \mu ( \d x ) - \cL ( \Theta_s ) } \, \d s.
\end{split}
\end{equation}
\end{prop}

\begin{cproof}{prop:gen:apriori}
\Nobs that for all $\theta \in \R^{\fd}$ it holds that
\begin{equation}
\begin{split}
    &(\nabla V) ( \theta ) \\
    &= 2 \rbr[\big]{  \w{\theta}_{1,1}, \ldots, \w{\theta}_{1,d}, \w{\theta}_{2 , 1} , \ldots, \w{\theta}_{2, d}, \ldots, \w{\theta}_{\width , 1}, \ldots, \w{\theta}_{\width, d} , \b{\theta} _1 , \ldots, \b{\theta}_\width , \v{\theta}_1 , \ldots, \v{\theta}_{\width},  2 \c{\theta} - 2 \xi } .
    \end{split}
\end{equation} 
This and \cref{eq:loss:gradient} imply that for all $\theta \in \R^{\fd}$ it holds that
\begin{equation} 
    \begin{split}
        &\spro{  ( \nabla V ) ( \theta) , \cG(\theta) } \\
        &=  4 \br[\Bigg]{  \sum_{i = 1}^\width \v{\theta}_i \int_{[\scra , \scrb]^d} \rbr[\big]{ \b{\theta}_i + \smallsum_{j=1}^d \w{\theta}_{i,j} x_j  } (\realization{\theta} (x) - f (x)) \indicator{(0, \infty )} \rbr[\big]{ \b{\theta}_i + \smallsum_{j=1}^d \w{\theta}_{i,j} x_j } \, \mu ( \d x ) } \\
        &+ 4 \br[\Bigg]{ \sum_{i=1}^\width \v{\theta}_i \int_{[\scra , \scrb]^d} \br[\big]{ \max \cu[\big]{ \b{\theta}_i + \smallsum_{j=1}^d \w{\theta}_{i,j} x_j , 0 } } ( \realization{\theta} (x) - f (x) ) \, \mu ( \d x ) } \\
        &+ 8 ( \c{\theta} - \xi ) \br*{ \displaystyle\int_{[\scra , \scrb]^d} (\realization{\theta} (x) - f (x) ) \, \mu ( \d x ) } .
\end{split}
 \end{equation}
 Hence, we obtain for all $\theta \in \R^{\fd}$ that
\begin{equation} \label{apriori:gen:eq1}
    \begin{split}
        &\spro{  ( \nabla V ) ( \theta) , \cG(\theta) } \\
        & = 8 \br*{\int_{[\scra , \scrb]^d} \rbr*{  \smallsum_{i=1}^\width \v{\theta}_i \br[\big]{ \max \cu[\big]{ \b{\theta}_i + \smallsum_{j=1}^d \w{\theta}_{i,j}  x_j , 0 } }  }(\realization{\theta} ( x ) -  f (x) ) \, \mu ( \d x ) } \\
        &+ 8 ( \c{\theta} - \xi ) \br*{ \int_{[\scra , \scrb]^d} (\realization{\theta} (x) - f (x) ) \, \mu ( \d x ) } \\
        &= 8 \int_{[\scra , \scrb]^d}  (\realization{\theta}(x) - \xi ) (\realization{\theta} (x) - f (x) ) \, \mu ( \d x ) \\
        &= 8 \int_{[\scra , \scrb]^d} ( \realization{\theta} ( x ) - f ( x ) ) ^2 \, \mu ( \d x ) + 8 \int_{[\scra , \scrb]^d} ( f ( x ) - \xi ) ( \realization{\theta} ( x ) - f ( x ) ) \, \mu ( \d x ) \\
        &= 8 \cL ( \theta ) + 8 \int_{[\scra , \scrb]^d} ( f ( x ) - \xi ) ( \realization{\theta} ( x ) - f ( x ) ) \, \mu ( \d x ) .
    \end{split}
\end{equation}
Next \nobs that the Cauchy-Schwarz inequality implies that for all $\theta \in \R^{\fd}$ it holds that
\begin{equation}
\begin{split}
  &\int_{[\scra , \scrb]^d} ( f ( x ) - \xi ) ( \realization{\theta} ( x ) - f ( x ) ) \, \mu ( \d x ) \\
  &\geq - \br*{ \int_{[\scra , \scrb]^d} (f(x) - \xi ) ^2 \, \mu ( \d x ) }^{1/2}  \br*{ \int_{[\scra , \scrb]^d} ( \realization{\theta} ( x ) - f ( x ) ) ^2 \, \mu ( \d x ) } ^{1/2}
  \\ &= - \br*{ \int_{[\scra , \scrb]^d} (f(x) - \xi ) ^2 \, \mu ( \d x ) }^{1/2}  \sqrt{\cL ( \theta ) }.
  \end{split}
\end{equation}
Combining this with the fact that for all $a , b \in \R$ it holds that $ab \leq \frac{a^2 + b^2}{2}$ demonstrates that for all $\theta \in \R^\fd$ we have that
\begin{equation}
    \int_{[\scra , \scrb]^d} ( f ( x ) - \xi ) ( \realization{\theta} ( x ) - f ( x ) ) \, \mu ( \d x ) \geq - \frac{1}{2}\br*{ \int_{[\scra , \scrb]^d} (f(x) - \xi ) ^2 \, \mu ( \d x ) } - \frac{\cL ( \theta )}{2} .
\end{equation}
This, \eqref{apriori:gen:eq1}, the fact that $V \in C^\infty ( \R^{\fd }, \R)$, and, e.g.,~\cite[Lemma 3.1]{CheriditoJentzenRiekert2021} show for all $t \in [ 0, \infty)$ that
\begin{equation}
\begin{split}
    V ( \Theta_t ) - V ( \Theta_0 ) 
    &= - \int_0^t \spro{  ( \nabla V ) ( \Theta_s ) , \cG ( \Theta_s) } \, \d s \\
    &= - 8 \int_0^t \cL ( \Theta_s ) \, \d s - 8 \int_0^t \br*{ \int_{[\scra , \scrb]^d} (f(x) - \xi ) ( \realization{\Theta_s} ( x ) - f ( x ) ) \, \mu ( \d x )  } \, \d s \\
    &\leq - 8 \int_0^t \cL ( \Theta_s ) \, \d s + 4 \int_0^t \br*{ \int_{[\scra , \scrb]^d} (f(x) - \xi ) ^2 \, \mu ( \d x ) + \cL ( \Theta_s ) } \, \d s \\
    &= 4 \int_0^t \br*{ \int_{[\scra , \scrb]^d} (f(x) - \xi ) ^2 \, \mu ( \d x ) - \cL ( \Theta_s ) } \, \d s.
    \end{split}
\end{equation}
\end{cproof}

\cfclear
\begin{cor} \label{cor:flow:limsup:loss}
Assume \cref{setting:snn}
and let $\Theta \in C([0, \infty) , \R^{ \fd })$ satisfy for all $t \in [0, \infty)$ that $\Theta_t = \Theta_0 - \int_0^t \cG ( \Theta_s ) \, \d s$ \cfadd{cor:gradient:measurable}\cfload.
Then 
\begin{equation} \label{cor:flow:limsup:loss:eq}
    \limsup_{t \to \infty} \cL ( \Theta_t ) \leq \inf_{\xi \in \R} \br*{ \int_{[\scra , \scrb]^d} ( f(x) - \xi ) ^2 \, \mu ( \d x )}.
\end{equation}
\end{cor}
\begin{cproof} {cor:flow:limsup:loss}
Throughout this proof let $\bfm, \xi, \nu \in \R$ satisfy $\bfm = \limsup_{t \to \infty} \cL ( \Theta_t)$ and $\nu = \int_{[\scra , \scrb]^d} ( f(x) - \xi ) ^2 \, \mu ( \d x )$.
\Nobs that \cref{lem:loss:integral} implies that $[0, \infty) \ni t \mapsto \cL ( \Theta_t ) \in \R$ is non-increasing. 
This assures that $ \inf_{t \in [0, \infty) } \cL ( \Theta_t ) = \bfm$. 
\cref{prop:gen:apriori} hence demonstrates that for all $t \in [0, \infty)$ it holds that
\begin{equation}
    \begin{split}
      0 & \leq V ( \Theta_t ) \leq V ( \Theta_0 ) + 4 \int_0^t \rbr*{ \nu - \cL ( \Theta_s ) } \, \d s \\
      &\leq V ( \Theta_0 ) + 4 \int_0^t (\nu - \bfm ) \, \d s = V ( \Theta_0 ) - 4 t ( \bfm - \nu ).
    \end{split}
\end{equation}
Therefore, we obtain for all $t \in (0, \infty )$ that $ \bfm - \nu \leq \tfrac{V(\Theta_0) }{4 t}$. This shows that
\begin{equation}
   \bfm \leq \limsup\nolimits_{t \to \infty} \br*{ \tfrac{V(\Theta_0) }{4 t} + \nu } = \nu .
\end{equation}
\end{cproof}

\subsection{A priori estimates for GFs with large risk}
\label{subsection:apriori:largerisk}

\cfclear
\begin{cor}  \label{cor:flow:bounded}
Assume \cref{setting:snn}, let $\nu , \xi \in \R$ satisfy $\nu = \int_{[\scra , \scrb]^d} ( f(x) - \xi ) ^2 \, \mu ( \d x )$,
and let $\Theta \in C([0, \infty) , \R^{ \fd })$ satisfy for all $t \in [0, \infty)$ that $\Theta_t = \Theta_0 - \int_0^t \cG ( \Theta_s ) \, \d s$ \cfadd{cor:gradient:measurable}\cfload.
Then 
\begin{equation}
    \sup\nolimits _{ t \in [0, \infty), \, \cL ( \Theta_t ) \geq \nu \indicator{(0, \infty ) } ( t ) } \norm{\Theta_t } \leq 3 \norm{\Theta_0 } ^2 + 8 \abs{ \xi } ^2 < \infty.
\end{equation}
\end{cor}

\begin{cproof2} {cor:flow:bounded}
Throughout this proof let $V \colon \R^{\fd } \to [0, \infty)$ satisfy for all $\theta \in \R^{\fd }$ that $V ( \Theta ) = \norm{\theta}^2 + \abs{\c{\theta} - 2 \xi } ^2$ and let $t \in (0, \infty )$ satisfy $\cL ( \Theta_t ) \ge \nu$.
\Nobs that \cref{lem:loss:integral} implies that $[0, \infty) \ni s \mapsto \cL ( \Theta_s ) \in \R$ is non-increasing. This shows that for all $s \in [0,t]$ it holds that $\cL ( \Theta_s ) \geq \cL ( \Theta_t ) \geq \nu$.
Combining this with \cref{prop:gen:apriori} demonstrates that 
\begin{equation} \label{cor:flow:bounded:eq}
    \begin{split}
       \norm{\Theta_t } \leq V ( \Theta_t ) \leq V ( \Theta_0 ) + 4 \int_0^t \rbr*{ \nu - \cL ( \Theta_s ) } \, \d s \leq V ( \Theta_0 ).
    \end{split}
\end{equation}
Furthermore, \nobs that the fact that for all $x,y \in \R$ it holds that $(x+y)^2 \le 2 ( x^2 + y^2)$ ensures that for all $\theta \in \R^\fd$ it holds that
\begin{equation}
    V ( \theta ) = \norm{\theta}^2 + \abs{\c{\theta} - 2 \xi } ^2 \leq \norm{\theta}^2 + 2 ( \abs{\c{\theta}}^2 + \abs{2 \xi } ^2 ) \leq 3 \norm{\theta} ^2 + 8 \abs{\xi } ^2.
\end{equation}
Combining this with \cref{cor:flow:bounded:eq} proves that $\norm{\Theta_t } \leq 3 \norm{\Theta_0 } ^2 + 8 \abs{ \xi } ^2 < \infty$. 
\end{cproof2}

\subsection{Invariant quantities for GFs}
\label{subsection:gf:invariant}

\cfclear
\begin{prop} \label{prop:flow:invariant}
Assume \cref{setting:snn}, let $W_i \colon \R^\fd \to \R$, $i \in \cu{1, 2, \ldots, \width }$, satisfy for all $\theta \in \R^\fd $, $i \in \cu{1, 2, \ldots, \width }$ that $W_i ( \theta ) = \br[\big]{ \sum_{j=1}^d (\w{\theta}_{i,j})^2 } + (\b{\theta}_i)^2 - (\v{\theta}_i)^2$, and let $\Theta \in C([0, \infty) , \R^{\fd})$ satisfy for all $t \in [0, \infty)$ that $\Theta_t = \Theta_0 - \int_0^t \cG ( \Theta_s ) \,  \d s$ \cfadd{cor:gradient:measurable}\cfload.
Then 
\begin{enumerate} [label = (\roman*)]
    \item \label{prop:flow:invariant:item1} it holds for all $t \in [ 0, \infty)$, $i \in \cu{1, 2, \ldots, \width }$ that $W_i ( \Theta_t ) = W_i ( \Theta_0 ) $ and
\item \label{prop:flow:invariant:item2} it holds for all $t \in [ 0, \infty)$ that $\sum_{i=1}^\width W_i ( \Theta_t ) = \sum_{i=1}^\width W_i ( \Theta_0 )$.
\end{enumerate}

\end{prop}

\begin{cproof} {prop:flow:invariant}
\Nobs that the assumption that for all $\theta \in \R^\fd$, $i \in \cu{1, 2, \ldots, \width }$ it holds that $W_i ( \theta ) = \br[\big]{ \sum_{j=1}^d (\w{\theta}_{i,j})^2 } + (\b{\theta}_i)^2 - (\v{\theta}_i)^2$ and \cref{eq:loss:gradient} demonstrate that for all $\theta \in \R^\fd$, $i \in \cu{1, 2, \ldots, \width }$ we have that
\begin{equation} 
    \begin{split}
          &\spro{  ( \nabla W_i ) ( \theta) , \cG(\theta) } \\
        &=  4 \v{\theta}_i \int_{[\scra , \scrb]^d} \rbr[\big]{  \b{\theta}_i + \smallsum_{j=1}^d \w{\theta}_{i,j} x_j } (\realization{\theta} (x) - f (x)) \indicator{(0, \infty )} \rbr[\big]{  \b{\theta}_i + \smallsum_{j=1}^d \w{\theta}_{i,j} x_j } \, \mu ( \d x )  \\
        &- 4  \v{\theta}_i \int_{[\scra , \scrb]^d} \br[\big]{ \max \cu[\big]{ \b{\theta}_i + \smallsum_{j=1}^d \w{\theta}_{i,j} x_j , 0 } } ( \realization{\theta} (x) - f (x) ) \, \mu ( \d x ) = 0.
    \end{split}
\end{equation}
This, the fact that for all $i \in \cu{1, 2, \ldots, \width }$ it holds that $W_i \in C^\infty ( \R^{\fd }, \R)$, and, e.g.,~\cite[Lemma 3.1]{CheriditoJentzenRiekert2021} show for all $i \in \cu{1, 2, \ldots, \width }$, $t \in [ 0, \infty)$ that
\begin{equation}
    W_i ( \Theta_t ) = W_i ( \Theta_0 ) - \int_0^t \spro{  (\nabla W_i ) ( \Theta_s ) , \cG ( \Theta_s ) } \, \d s = W_i ( \Theta_0 ).
\end{equation}
This proves \cref{prop:flow:invariant:item1}. Next \nobs that \cref{prop:flow:invariant:item1} establishes \cref{prop:flow:invariant:item2}.
\end{cproof}

\section{Properties of ANN parametrizations with small risk and one hidden neuron}
\label{section:small:risk}

In \cref{theo:flow:convergence:small:loss} in \cref{section:gf:convergence:1neuron} below we establish in the case where the measure $\mu$ (see~\cref{setting:snn}) is up to a constant the Lebesgue--Borel measure on $[\scra, \scrb]$, 
where the hidden layer consists of only one neuron (where $\width = 1$), and where the target function $f \colon [\scra , \scrb] \to \R$ is affine linear 
that the risk of every not necessarily bounded GF trajectory converges to zero.
 Our proof of \cref{theo:flow:convergence:small:loss} employs, among other things,
 the a priori bounds for GF trajectories with sufficiently small initial risk in \cref{lem:flow:product:vw:bounded} in \cref{subsection:gf:apriopri:1neuron} below, the well known mean square approximation results in \cref{lem:constant:approx} and \cref{cor:constant:approx:affine} in \cref{subsection:mean:square} below, 
the lower bound for the product of the slope of the target function and its ANN approximations in \cref{cor:prod:vw:positive} in \cref{subsection:param:small:risk} below, and appropriate lower bounds for the transformation between the input and hidden layer of the considered ANN in \cref{lem:neuron:active} in \cref{subsection:param:small:risk}.

In \cref{lem:constant:approx} in \cref{subsection:mean:square} we recall the elementary fact that the mean value of a given continuous function on a compact real interval is the best constant mean square approximation of the considered continuous function. \cref{cor:constant:approx:affine} in \cref{subsection:mean:square} specializes \cref{lem:constant:approx} to the case where the considered continuous function is affine linear. \cref{lem:constant:approx} follows, e.g., from~\cite[Lemma 2.1]{BeckBeckerGrohs2018} and only for completeness we include in this section detailed proofs for \cref{lem:constant:approx} and \cref{cor:constant:approx:affine}.

In \cref{setting:snn:width1} in \cref{subsection:ann:one:neuron} below we specialize \cref{setting:snn} from \cref{subsection:description:anns} above and present the mathematical framework which we frequently employ in \cref{section:small:risk,section:gf:convergence:1neuron} to formulate ANNs with ReLU activation, one hidden layer, one neuron on the input layer (corresponding to the case $d = 1$ in \cref{setting:snn}), and one neuron on the hidden layer (corresponding to the case $\width = 1 $ in \cref{setting:snn}) and the corresponding risk functions (see \cref{setting:snn:width1:eq:realization} in \cref{setting:snn:width1}).

In \cref{subsection:param:small:risk} we study realizations of ANNs whose risk is strictly smaller than the risk which can be achieved by the best constant approximation (cf.~\cref{lem:constant:approx}). Our proof of the a priori bound result for GF trajectories with sufficiently small initial risk in \cref{lem:flow:product:vw:bounded} in \cref{subsection:gf:apriopri:1neuron} employs \cref{lem:loss:integral} from Subsection 3.1, \cref{lem:constant:approx} and \cref{cor:constant:approx:affine} from \cref{subsection:mean:square} and \cref{lem:realization:lipschitz:const}, \cref{cor:prod:vw:positive}, and \cref{prop:product:vw:bounded} from \cref{subsection:param:small:risk}. The elementary result in \cref{lem:realization:lipschitz:const} in \cref{subsection:param:small:risk} shows that for every ANN 
with parameter vector $\theta = ( \theta_1, \ldots, \theta_4 ) \in \R^4$ we have that the realization associated to $\theta$ is Lipschitz continuous with the Lipschitz constant $\abs{ \theta_1 \theta_3 } $.

\cref{cor:prod:vw:positive} in \cref{subsection:param:small:risk} demonstrates 
in the case where there exist $\alpha, \beta \in \R$ such that the target function satisfies for all $x \in [\scra , \scrb ]$ that $f(x) = \alpha x + \beta$ 
that for every 
ANN whose risk is strictly smaller than the risk which can be achieved by the best constant approximation (cf.~\cref{lem:constant:approx}) with parameter vector $\theta = ( \theta_1, \ldots , \theta_4 ) \in \R^4$ we have that the slope $\alpha$ of the target function and the slope $\theta_1 \theta_3$ of the realization of the ANN must have the same sign in the sense that $\alpha \theta_1 \theta_3 > 0$. Our proof of \cref{cor:prod:vw:positive} employs an application of \cref{lem:prod:vw:positive} in \cref{subsection:param:small:risk}. \cref{lem:prod:vw:positive}, in turn, establishes the statement of \cref{cor:prod:vw:positive} in the special case where the slope $\alpha$ of the target function is assumed to be strictly positive in the sense that $\alpha > 0$.

\cref{lem:neuron:active} in \cref{subsection:param:small:risk} establishes that for every 
ANN whose risk is strictly smaller than the risk which can be achieved 
by the best constant approximation (cf.~\cref{lem:constant:approx}) with parameter vector $\theta = ( \theta_1, \ldots , \theta_4 ) \in \R^4$ 
we have that the hidden neuron of this ANN cannot be inactive and we must have that 
$\max \cu{ \theta_1 \scra + \theta_2, \theta_1 \scrb + \theta_2 } > 0$. This simply follows from the fact that if the neuron was
inactive in the sense that $\max \cu{ \theta_1 \scra + \theta_2, \theta_1 \scrb + \theta_2 } \leq 0$, then the realization function associated to $\theta$ would be constant 
which would result in a larger risk.

Finally, \cref{prop:product:vw:bounded} in \cref{subsection:param:small:risk}, the main result of \cref{section:small:risk}, loosely speaking, reveals that 
for every ANN whose risk is strictly smaller than the risk which can be achieved by 
the best constant approximation (cf.~\cref{lem:constant:approx}) with parameter vector $\theta = ( \theta_1, \ldots , \theta_4 ) \in \R^4$ 
we have that the slope of the realization of the ANN $\theta$ is uniformly bounded from below and from above.

\subsection{Mean square approximations through constant functions}
\label{subsection:mean:square}

\begin{lemma} \label{lem:constant:approx}
Let $\xi , \scra \in \R$, $\scrb \in ( \scra , \infty)$, $f \in C( [\scra , \scrb ] , \R )$. Then 
\begin{equation}
    \int_\scra^\scrb ( f(x) - \xi ) ^2 \, \d x \geq \int_\scra^\scrb \rbr*{f ( x ) -\tfrac{1}{\scrb - \scra } \br[\big]{ \textstyle\int_\scra^\scrb f(y) \, \d y } }^{\! 2} \, \d x.
\end{equation}
\end{lemma}
\begin{cproof2}{lem:constant:approx}
Throughout this proof let $\mu \in \R$ satisfy $\mu = (\scrb - \scra)^{-1} \int_\scra^\scrb f(y) \, \d y$. \Nobs that for all $u \in \R$ it holds that
\begin{equation}
      \int_\scra^\scrb ( f(x) - u ) ^2 \, \d x = \int_\scra^\scrb ( f(x) )^2 \, \d x - 2 u \mu ( \scrb - \scra ) + u^2 ( \scrb - \scra ).
\end{equation}
Hence, we obtain that
\begin{equation}
\begin{split}
        &\int_\scra^\scrb ( f(x) - \xi ) ^2 \, \d x -  \int_\scra^\scrb \rbr*{f ( x ) -\tfrac{1}{\scrb - \scra} \br[\big]{ \textstyle\int_\scra^\scrb f(y) \, \d y } }^{\! 2} \, \d x \\
        &=  \int_\scra^\scrb ( f(x) - \xi ) ^2 \, \d x -  \int_\scra^\scrb \rbr*{f ( x ) - \mu }^{ 2} \, \d x \\
        &= - 2 \xi \mu (\scrb - \scra ) + \xi ^2 ( \scrb - \scra)  + 2 \mu ^2 ( \scrb - \scra ) - \mu^2 ( \scrb - \scra ) \\
        &= (\scrb - \scra ) (\xi^2 - 2 \xi \mu + \mu^2 ) = ( \scrb - \scra ) ( \xi - \mu ) ^2 \geq 0.
   \end{split}
\end{equation}
\end{cproof2}

\begin{cor} \label{cor:constant:approx:affine}
Let $\xi, \alpha, \beta, \scra \in \R$, $\scrb \in (\scra , \infty)$. Then
\begin{equation}
    \int_\scra^\scrb ( \alpha x + \beta - \xi ) ^2 \, \d x \geq \int_\scra^\scrb \rbr[\big]{ ( \alpha x + \beta ) - ( \alpha \br{ \tfrac{\scra + \scrb }{2} } + \beta ) }^2 \, \d x = \frac{\alpha^2 (\scrb - \scra)^3}{12}.
\end{equation}
\end{cor}

\begin{cproof}{cor:constant:approx:affine}
\Nobs that $\int_\scra^\scrb ( \alpha x + \beta  )  \, \d x = \frac{\alpha(\scrb^2 - \scra^2)}{2} + \beta ( \scrb - \scra ) =  ( \scrb - \scra ) ( \alpha \br{ \frac{\scrb + \scra}{2} } + \beta )$. \cref{lem:constant:approx} hence shows that
\begin{equation}
\begin{split}
      \int_\scra^\scrb ( \alpha x + \beta - \xi ) ^2 \, \d x &\geq \int_\scra^\scrb \rbr[\big]{ ( \alpha x + \beta ) - ( \alpha \br{ \tfrac{\scrb + \scra}{2} } + \beta ) }^2 \, \d x 
      = \int_\scra^\scrb \alpha^2(  x - \br{ \tfrac{\scrb + \scra}{2} } )^2 \, \d x \\
      &=\br*{ \tfrac{\alpha^2}{3} \rbr[\big]{x- \br{ \tfrac{\scrb + \scra}{2} }}^3 }_{x=\scra}^{x = \scrb} = \frac{\alpha^2}{3} \br*{ \rbr*{\frac{\scrb - \scra}{2}}^{ \! 3 } - \rbr*{ \frac{\scra - \scrb}{2} }^{ \! 3 } } \\
      &= \frac{\alpha^2}{24} \br[\big]{ (\scrb - \scra)^3-(\scra - \scrb)^3} 
      =\frac{\alpha^2 (\scrb - \scra)^3}{12}.
\end{split}
\end{equation}
\end{cproof}

\subsection{Mathematical description of ANNs with one hidden neuron}
\label{subsection:ann:one:neuron}

\begin{setting} \label{setting:snn:width1}
Let $\scra \in \R$, $\scrb \in (\scra , \infty)$,
 $\rho \in (0, \infty)$,
$f \in C ( [\scra , \scrb] , \R)$,
 $\fw, \fb, \fv, \fc \in C(\R^4, \R )$ 
 satisfy for all $\theta  = ( \theta_1 ,  \ldots, \theta_{4}) \in \R^{4}$ that $\w{\theta} = \theta_{ 1}$, $\b{\theta} = \theta_{2}$, 
$\v{\theta} = \theta_{3}$, and $\c{\theta} = \theta_{4}$,
let $\scrN = (\realization{\theta})_{\theta \in \R^{4 } } \colon \R^{4 } \to C(\R , \R)$ and $\cL \colon \R^{4  } \to \R$
satisfy for all $\theta \in \R^{4}$, $x \in \R$ that
\begin{equation} \label{setting:snn:width1:eq:realization}
    \realization{\theta} (x) =  \v{\theta} \max \cu[\big]{ \w{\theta} x + \b{\theta} , 0 } + \c{\theta}
\end{equation}
and $\cL (\theta) = \rho \int_{\scra}^\scrb (\realization{\theta} (y) - f ( y ) )^2 \, \d y $,
let $\Rect_r \in C^1 ( \R , \R )$, $r \in \N $, satisfy for all $x \in \R$ that
\begin{equation}
    \limsup\nolimits_{r \to \infty}  \rbr*{ \abs { \Rect_r ( x ) - \max \cu{ x , 0 } } + \abs { (\Rect_r)' ( x ) - \indicator{(0, \infty)} ( x ) } } = 0
\end{equation}
and
$\sup_{r \in \N} \sup_{y \in [- \abs{x}, \abs{x} ] } \rbr*{ \abs{\Rect_r(y)} + \abs{(\Rect_r)'(y)}} < \infty$,
let $\fL_r \colon \R^4 \to \R$, $r \in \N$,
satisfy for all $r \in \N$, $\theta \in \R^{4}$ that
\begin{equation}
    \fL_r ( \theta ) = \rho \int_\scra^\scrb \rbr[\big]{  \v{\theta} \br[\big]{ \Rect_r ( \w{\theta} x + \b{\theta} ) } + \c{\theta} - f(x)} ^2 \, \d x,
\end{equation}
let $\norm{ \cdot } \colon \rbr*{  \bigcup_{n \in \N} \R^n  } \to \R $ and
$\spro{  \cdot , \cdot } \colon \rbr*{  \bigcup_{n \in \N} (\R^n \times \R^n )  } \to \R$ satisfy for all
$n \in \N$, $x=(x_1, \ldots, x_n)$, $y=(y_1, \ldots, y_n ) \in \R^n $ that
$\norm{ x } = [ \sum_{i=1}^n \abs*{ x_i } ^2 ] ^{1/2}$ and $\spro{  x , y } = \sum_{i=1}^n x_i y_i$,
let $\lambda \colon \cB ( \R ) \to [0, \infty]$ be the Lebesgue--Borel measure on $\R$,
let $I^\theta \subseteq \R$, $\theta \in \R^{4 }$, satisfy for all 
$\theta \in \R^{4}$ that $I^\theta = \cu{ x \in [\scra  , \scrb ] \colon \w{\theta} x + \b{\theta}  > 0 }$,
and let $\cG = ( \cG_1 , \ldots, \cG_4 ) \colon \R^4 \to \R^4$ satisfy for all
$\theta \in \cu{ \vartheta \in \R^4 \colon ( ( \nabla \fL_r ) ( \vartheta ) ) _{r \in \N} \text{ is convergent} }$
that $\cG ( \theta ) = \lim_{r \to \infty} (\nabla \fL_r) ( \theta )$. 
\end{setting}

\subsection{Properties of ANNs with small risk and one hidden neuron}
\label{subsection:param:small:risk}

\begin{lemma} \label{lem:realization:lipschitz:const}
Assume \cref{setting:snn:width1} and let $\theta \in \R^{4}$. Then it holds for all $x,y \in \R$ that
\begin{equation}
    \abs{ \realization{\theta} ( x ) - \realization{\theta} ( y )} \leq \abs{\w{\theta}\v{\theta} } \abs{x-y}.
\end{equation}
\end{lemma}
\begin{cproof}{lem:realization:lipschitz:const}
\Nobs that \cref{setting:snn:width1:eq:realization} ensures that for all $x , y\in \R$ it holds that
\begin{equation}
    \begin{split}
        \abs{ \realization{\theta} ( x ) - \realization{\theta} ( y )} & 
        = \abs{\v{\theta} \max \cu{ \w{\theta} x + \b{\theta}, 0 } - \v{\theta} \max \cu{ \w{\theta} y + \b{\theta}, 0 } }  \\
        &= \abs{\v{\theta}} \abs{ \max \cu{ \w{\theta} x + \b{\theta} , 0 } - \max \cu{ \w{\theta} y + \b{\theta} , 0 } } \\
        & \leq \abs{\v{\theta} } \abs{ ( \w{\theta} x + \b{\theta} ) - (\w{\theta} y + \b{\theta} ) } = \abs{\w{\theta} \v{\theta} } \abs{ x-y }.
    \end{split}
\end{equation}
\end{cproof}

\begin{lemma} \label{lem:prod:vw:positive}
Assume \cref{setting:snn:width1}, let $\alpha \in (0, \infty)$, 
$\beta \in \R$ satisfy for all $ x \in [\scra,\scrb]$ that $f(x)=\alpha x + \beta$, and let $\theta \in \R^4$ satisfy $\cL ( \theta ) < \frac{\rho \alpha^2 (\scrb - \scra)^3}{12}$. Then 
\begin{equation} \label{lem:prod:vw:pos:eq}
    \w{\theta}  \v{\theta}  > 0.
\end{equation}
\end{lemma}
\begin{cproof}{lem:prod:vw:positive}
We prove \cref{lem:prod:vw:pos:eq} by contradiction. We thus assume that
\begin{equation} \label{lem:vw:positive:eq:cont}
    \w{\theta} \v{\theta} \leq 0.
\end{equation}
\Nobs that \cref{lem:vw:positive:eq:cont} ensures that for all $x,y \in [\scra , \scrb ]$ with $x \le y$ it holds that
\begin{equation} \label{lem:vw:positive:eq:noninc}
   \realization{\theta} ( x ) \ge \realization{\theta} ( y ).
\end{equation}
In the following we distinguish between the case $\realization{\theta} ( \scrb ) \ge f ( \scrb )$, the case $\realization{\theta} ( \scra ) \le f ( \scra )$,
and the case $\min \cu{f ( \scrb ) - \realization{\theta} ( \scrb ) , \realization{\theta} ( \scra ) - f ( \scra ) } > 0$.
We first establish the contradiction in the case
\begin{equation} \label{lem:vw:positive:eq:case1}
    \realization{\theta} ( \scrb ) \ge f ( \scrb ).
\end{equation}
\Nobs that \cref{lem:vw:positive:eq:noninc,lem:vw:positive:eq:case1} imply for all $x \in [\scra , \scrb]$ that $\realization{\theta} ( x ) \geq \realization{\theta} ( \scrb ) \ge f(\scrb ) \geq f(x)$. 
Combining this with \cref{cor:constant:approx:affine} proves that $\frac{\rho \alpha^2 (\scrb - \scra)^3}{12} > \cL ( \theta ) \geq \rho \int_\scra^\scrb (f ( \scrb) - f ( x ) )^2 \, \d x \geq \frac{ \rho \alpha^2 (\scrb - \scra)^3 }{12}$, which is a contradiction.
In the next step we establish the contradiction in the case
\begin{equation} \label{lem:vw:positive:eq:case2}
    \realization{\theta} ( \scra ) \leq f(\scra).
\end{equation}
\Nobs that \cref{lem:vw:positive:eq:noninc,lem:vw:positive:eq:case2} show for all $x \in [\scra , \scrb ]$ that $\realization{\theta} ( x ) \leq \realization{\theta } ( \scra ) \leq f(\scra ) \leq f(x)$.
This and \cref{cor:constant:approx:affine} imply that $\frac{\rho \alpha^2 (\scrb - \scra)^3}{12} > \cL ( \theta ) \geq \rho \int_\scra^\scrb (f(\scra ) - f(x))^2 \, \d x\geq \frac{\rho \alpha^2 (\scrb - \scra)^3 }{12}$, which is a contradiction.
Finally, we establish the contradiction in the case
\begin{equation} \label{lem:vw:positive:eq:case3}
    \min \cu{f ( \scrb ) - \realization{\theta} ( \scrb ) , \realization{\theta} ( \scra ) - f ( \scra ) } > 0 .
\end{equation}
\Nobs that \cref{lem:vw:positive:eq:case3} and intermediate value theorem assure that there exists $u \in [\scra , \scrb]$ such that $\realization{\theta} ( u ) = f ( u ) $. 
This and \cref{lem:vw:positive:eq:noninc} prove that $\forall \, x \in [\scra, u] \colon \realization{\theta} ( x ) \geq \realization{\theta} ( u  ) = f(u) \geq f(x)$ and $\forall \, x \in [u , \scrb] \colon \realization{\theta} ( x ) \leq \realization{\theta} ( u ) = f ( u ) \leq f ( x )$. 
Combining this with \cref{cor:constant:approx:affine} demonstrates that
\begin{equation}
\begin{split}
       \frac{\rho \alpha^2 (\scrb - \scra)^3}{12} 
       &> \cL ( \theta ) 
        = \rho \int_\scra^u ( \realization{\theta} ( x ) - f ( x ) ) ^2 \, \d x + \rho \int_u^\scrb ( \realization{\theta} ( x ) - f ( x ) ) ^2 \, \d x \\
        &\geq \rho \int_\scra^u ( f(u) - f (x))^2 \, \d x + \rho \int_u^\scrb ( f ( x ) - f ( u ) )^2 \, \d x \\
        &= \rho \int_\scra^\scrb (f(x) - f(u))^2 \, \d x \geq \frac{\rho \alpha^2 (\scrb - \scra)^3 }{12}.
\end{split}
\end{equation}
This is a contradiction.
\end{cproof}

\begin{cor} \label{cor:prod:vw:positive}
Assume \cref{setting:snn:width1},
let $\alpha , \beta \in \R$ satisfy for all $ x \in [\scra,\scrb]$ that $f(x)=\alpha x + \beta$,
and let $\theta \in \R^4$ satisfy $\cL ( \theta ) < \frac{\rho \alpha^2 (\scrb - \scra)^3}{12}$.
Then 
\begin{equation} \label{cor:prod:vw:pos:eq}
    \alpha \w{\theta}  \v{\theta}  > 0.
\end{equation}
\end{cor}

\begin{cproof}{cor:prod:vw:positive}
\Nobs that the assumption that $\cL ( \theta ) < \frac{\rho \alpha^2 (\scrb - \scra)^3}{12}$ assures that $\alpha \not= 0$.
In the following we distinguish between the case $\alpha > 0$ and the case $\alpha < 0$.
First \nobs that \cref{lem:prod:vw:positive} establishes \cref{cor:prod:vw:pos:eq} in the case $\alpha > 0$.
In the next step we prove \cref{cor:prod:vw:pos:eq} in the case $\alpha < 0$.
\Nobs that
\begin{equation}
\begin{split}
   \frac{\rho \alpha^2 (\scrb - \scra)^3}{12} &> \cL ( \theta ) 
    = \rho \int_\scra^\scrb \rbr[\big]{ \v{\theta} \max \cu{\w{\theta} x + \b{\theta}, 0 } + \c{\theta} - ( \alpha x + \beta )}^2 \, \d x \\
    &= \rho \int_\scra^\scrb \rbr[\big]{ (- \v{\theta} ) \max \cu{\w{\theta} x + \b{\theta}, 0 } + ( - \c{\theta} ) - ( - \alpha x - \beta )}^2 \, \d x.
\end{split}
\end{equation}
Combining this, the fact that $-\alpha > 0$,
and \cref{lem:prod:vw:positive} (applied with $\theta \with (\w{\theta}, \b{\theta}, - \v{\theta}, - \c{\theta} )$, $\alpha \with - \alpha$, $\beta \with -\beta$ in the notation of \cref{lem:prod:vw:positive})
demonstrates that $\alpha \w{\theta} \v{\theta} = (- \alpha ) \w{\theta} ( - \v{\theta} ) > 0$.
This establishes \cref{cor:prod:vw:pos:eq} in the case $\alpha < 0$.
\end{cproof}

\begin{lemma} \label{lem:neuron:active}
Assume \cref{setting:snn:width1}, let $m \in \R$ satisfy $m = \rho \int_\scra^\scrb \rbr{f ( x ) - (\scrb - \scra)^{-1} \int_\scra^\scrb f(y) \, \d y }^2 \, \d x$, and let $\theta \in \R^4$ satisfy $\cL ( \theta ) < m$. Then $\max \cu{ \w{\theta} \scra + \b{\theta} , \w{\theta} \scrb  + \b{\theta}  } > 0$. 
\end{lemma}
\begin{cproof}{lem:neuron:active}
We prove \cref{lem:neuron:active} by contradiction. 
We thus assume that
\begin{equation} \label{lem:neuron:active:eq:cont}
    \max \cu{\w{\theta} \scra + \b{\theta} , \w{\theta} \scrb  + \b{\theta} } \leq 0. 
\end{equation}
\Nobs that \cref{lem:neuron:active:eq:cont} ensures that for all $x \in [\scra , \scrb]$ we have that \begin{equation}
    \w{\theta}  x + \b{\theta}  = \br*{ \tfrac{\scrb - x}{\scrb - \scra} } ( \w{\theta} \scra + \b{\theta} ) + \br*{ \tfrac{x - \scra}{\scrb - \scra} } (\w{\theta} \scrb + \b{\theta} ) \leq 0.
    \end{equation}
This implies for all $x \in [\scra , \scrb]$ that $\max \cu{ \w{\theta}  x + \b{\theta} , 0 } = 0$. Therefore, we obtain for all $x \in [\scra , \scrb]$ that $\realization{\theta} ( x ) = \c{\theta}$. Combining this with \cref{lem:constant:approx} proves that $\cL ( \theta ) \geq m$. 
This is a contradiction.
\end{cproof}

\begin{prop} \label{prop:product:vw:bounded}
Assume \cref{setting:snn:width1} and let $m \in \R$, $\varepsilon \in (0, \infty)$ satisfy $m = \rho \int_\scra^\scrb \rbr{f ( x ) - (\scrb - \scra)^{-1} \int_\scra^\scrb f(y) \, \d y }^2 \, \d x$.
Then there exists $\fC \in (0, \infty)$ such that for all $\theta \in \cu{ \vartheta \in \R^4 \colon \cL ( \vartheta ) \leq m - \varepsilon }$ it holds that $\fC^{-1} \leq \abs{ \w{\theta}  \v{\theta}  } \leq \fC$. 
\end{prop}

\begin{cproof}{prop:product:vw:bounded}
Throughout this proof assume without loss of generality that $\varepsilon \leq m$, assume without loss of generality that $\cu{ \vartheta \in \R^4 \colon \cL ( \vartheta ) \leq m - \varepsilon } \not= \emptyset$, and let $M \in \R$ satisfy $M = \max \cu{ 1 , \sup_{x \in [\scra,\scrb] } \abs{f ( x ) }  }$.
We first prove that there exists $\fC \in (0, \infty)$ such that for all $\theta \in \cu{ \vartheta \in \R^4 \colon \cL ( \vartheta ) \leq m - \varepsilon }$ it holds that
\begin{equation} \label{prop:vw:bounded:eq:lowb}
    \fC^{-1} \le \abs{\w{\theta} \v{\theta} } .
\end{equation}
\Nobs that \cref{lem:realization:lipschitz:const} implies for all $\theta \in \R^4$, $x \in [\scra , \scrb]$ that $\abs{\realization{\theta} ( x ) - \realization{\theta} ( \scra ) } \leq \abs{ \w{\theta}  \v{\theta}  } \abs{x - \scra} \leq \abs{ \w{\theta}  \v{\theta}  } (\scrb - \scra) $. Combining this, \cref{lem:constant:approx}, and Minkowski's inequality establishes for all $\theta \in \R^4$ that
\begin{equation}
    \begin{split}
        \sqrt{\cL ( \theta ) } 
        &= \br*{ \rho \int_\scra^\scrb (\realization{\theta} ( x ) - f ( x ) ) ^2 \, \d x}^{1/2} \\
        & \geq \br*{ \rho \int_\scra^\scrb (\realization{\theta} ( \scra ) - f ( x ) ) ^2 \, \d x}^{1/2} - \br*{ \rho \int_\scra^\scrb ( \realization{\theta} ( x ) - \realization{\theta} ( \scra ) ) ^2 \, \d x}^{1/2} \\
        & \geq \inf_{\xi \in \R} \br*{ \rho \int_\scra^\scrb ( f ( x ) - \xi ) ^2 \, \d x}^{1/2} - \br*{ \rho \int_\scra^\scrb \abs{ \w{\theta}  \v{\theta}  } ^2 \abs{\scrb - \scra}^2 \, \d x}^{1/2} \\
        & = \sqrt{m} - \abs{ \w{\theta}  \v{\theta}  } \sqrt{\rho (\scrb - \scra)^3 }.
    \end{split}
\end{equation}
This implies for all $\theta \in \R^4$ that 
\begin{equation}
    \abs{\w{\theta}\v{\theta}} \geq \tfrac{\sqrt{m} - \sqrt{\cL ( \theta ) }}{\sqrt{\rho ( \scrb - \scra ) ^3 }}.
\end{equation}
Hence, we obtain for all $\theta \in \cu{ \vartheta \in \R^4 \colon \cL ( \vartheta ) \leq m - \varepsilon }$ that
\begin{equation}
    \abs{ \w{\theta}  \v{\theta}  } \geq \tfrac{\sqrt{m} - \sqrt{\cL ( \theta ) }}{\sqrt{\rho ( \scrb - \scra ) ^3}} \geq \tfrac{ \sqrt{m} - \sqrt{m - \varepsilon} } { \sqrt{ \rho ( \scrb - \scra ) ^3 } }> 0.
\end{equation}
This establishes \cref{prop:vw:bounded:eq:lowb}.
In the next step we verify that there exists $\fC \in (0, \infty)$ such that for all $\theta \in \cu{ \vartheta \in \R^4 \colon \cL ( \vartheta ) \leq m - \varepsilon } $ it holds that
\begin{equation} \label{prop:vw:bounded:eq:upb}
    \abs{\w{\theta} \v{\theta} } \leq \fC.
\end{equation}
We prove \cref{prop:vw:bounded:eq:upb} by contradiction. In the following we thus assume that
\begin{equation} \label{prop:vw:bounded:eq:infty}
    \sup\nolimits_{\theta \in \cu{ \vartheta \in \R^4 \colon \cL ( \vartheta ) \leq m - \varepsilon } } \abs{ \w{\theta}  \v{\theta}  } = \infty.
\end{equation} 
\Nobs that \cref{prop:vw:bounded:eq:infty} ensures that there exist $\theta_n \in \cu{ \vartheta \in \R^4 \colon \cL ( \vartheta ) \leq m - \varepsilon }$, $n \in \N$, which satisfy for all $n \in \N$ that 
\begin{equation} \label{prop:vw:bounded:eq:defthetan}
    \abs{ \w{\theta_n} \v{\theta_n}} \geq 2 (n + 1 ) ^2 M >  0.
\end{equation}
Roughly speaking, 
we next establish that for all sufficiently large $n$ it holds that the function $[\scra , \scrb] \ni x \mapsto \realization{\theta_n} ( x ) \in \R$ is almost constant in the sense that $\limsup_{n \to \infty} \lambda ( I^{\theta_n } ) = 0$
and, thereafter, we use this to prove \cref{prop:vw:bounded:eq:upb}.
\Nobs that \cref{setting:snn:width1:eq:realization} ensures that for all $n \in \N$, $x \in I^{\theta_n}$ it holds that $\realization{\theta_n} ( x ) = \w{\theta_n} \v{\theta_n} x + (\b{\theta_n} \v{\theta_n} + \c{\theta_n } )$.
Combining this with \cref{prop:vw:bounded:eq:defthetan} and the fact that for all $\alpha, \beta ,c \in \R$ with $\alpha \not= 0$ it holds that
\begin{equation}
\begin{split}
       & \lambda ( \cu{ x \in [\scra , \scrb] \colon \abs{\alpha x + \beta } \leq \abs{c} } )
         \leq \lambda ( \cu{ x \in \R \colon \abs{\alpha x + \beta } \leq \abs{c} } )
         \\ 
         &= \lambda \rbr[\big]{ \cu[\big] { x \in \R \colon \abs[\big]{ x + \tfrac{\beta}{\alpha} } \leq \tfrac{\abs{c}}{\abs{\alpha} } } } 
         = \lambda \rbr[\big] {\br[\big] {-\tfrac{\beta}{\alpha} - \tfrac{\abs{c}}{\abs{\alpha} } , - \tfrac{\beta}{\alpha} + \tfrac{\abs{c}}{\abs{\alpha} } } }
         = \tfrac{2 \abs{c}}{\abs{\alpha}}
\end{split}
\end{equation}
implies that for all $n \in \N$ we have that
\begin{equation}
    \lambda \rbr*{ \cu[\big]{ x \in I ^{\theta_n} \colon \abs{\realization{\theta_n} ( x ) } \leq (n+1) M } } \leq \min \cu*{ \lambda(I ^{\theta_n}) , \tfrac{2 ( n+1) M}{\abs{ \w{\theta_n} \v{\theta_n}}} } \leq \min \cu[\big]{ \lambda(I ^{\theta_n}) , \tfrac{1}{n+1} }.
\end{equation}
Hence, we obtain for all $n \in \N$ that 
\begin{equation} \label{prop:product:vw:bounded:eq1}
\begin{split}
     \lambda \rbr*{ \cu[\big]{ x \in I ^{\theta_n} \colon \abs{\realization{\theta_n} ( x ) } > (n+1) M } } &= \lambda ( I^{\theta_n} ) - \lambda \rbr*{ \cu[\big]{ x \in I ^{\theta_n} \colon \abs{\realization{\theta_n} ( x ) } \leq (n+1) M } }\\
     &\geq \lambda ( I^{\theta_n} ) - \min \cu[\big]{ \lambda(I ^{\theta_n}) , \tfrac{1}{n+1} } \\
     & = \max \cu[\big]{ 0 , \lambda(I ^{\theta_n}) - \tfrac{1}{n+1} }. 
\end{split}
\end{equation}
Furthermore, \nobs that for all $x \in I ^{\theta_n}$ with $\abs{\realization{\theta_n} ( x ) } > (n+1) M$ it holds that 
\begin{equation}
\abs{ \realization{\theta_n} ( x ) - f ( x ) } \geq \abs{\realization{\theta_n}(x) } - \abs{f(x)} \geq \abs{\realization{\theta_n} (x) } - M > (n+1) M - M = n M.
\end{equation}
Combining this with \cref{prop:vw:bounded:eq:defthetan,prop:product:vw:bounded:eq1} implies that for all $n \in \N$ it holds that
\begin{equation} \label{prop:vw:bounded:eq3}
    m > m - \varepsilon \ge \cL ( \theta_n ) \geq n^2 M^2 \max \cu[\big]{ 0 , \lambda(I ^{\theta_n}) - \tfrac{1}{n+1} }.
\end{equation}
Hence, we obtain that
\begin{equation} \label{prop:product:vw:bounded:eq2}
\begin{split}
    0 &\leq \limsup_{n \to \infty} [\lambda ( I^{\theta_n})] = \limsup_{n \to \infty} \br[\big]{\lambda ( I^{\theta_n}) - \tfrac{1}{n+1} } \leq \limsup_{n \to \infty} \br[\big]{ \max \cu[\big]{ 0 , \lambda(I ^{\theta_n}) - \tfrac{1}{n+1} } } \\
    &\leq \limsup_{n \to \infty} \br[\big]{ \tfrac{m}{n^2 M^2}} = 0.
    \end{split}
\end{equation}
 Next \nobs that \cref{setting:snn:width1:eq:realization} ensures that for all $n \in \N$, $x \in [\scra , \scrb] \backslash I ^{\theta_n }$ it holds that $\realization{\theta_n} ( x ) = \c{\theta_n}$. This implies for all $n \in \N$ that
\begin{equation} \label{prop:vw:bounded:eq1}
    \cL ( \theta_n )
    \geq \rho \int_{[\scra , \scrb ] \backslash I^{\theta_n } } ( f(x) - \realization{\theta_n} ( x ) ) ^2 \, \d x
    \geq \inf_{\xi \in \R} \br*{ \rho \int_{ [\scra , \scrb] \backslash I ^{\theta_n }} ( f ( x ) - \xi ) ^2 \, \d x } .
\end{equation}
Furthermore, \nobs that for all $n \in \N$ it holds that
\begin{equation} \label{prop:vw:bounded:eq:interval}
\begin{split}
    [\scra , \scrb] \backslash I^{\theta_n} &= \cu{ x \in [\scra , \scrb] \colon \w{\theta_n} x + \b{\theta_n} \leq 0 } = \cu{ x \in [\scra , \scrb ] \colon \w{\theta_n} x \leq - \b{\theta_n} } \\
    &= \begin{cases}
    [\scra , \scrb] & \colon \b{\theta_n} \leq \w{\theta_n} = 0 \\
    \emptyset & \colon \b{\theta_n} > \w{\theta_n} = 0\\
    [\scra , \scrb] \cap (-\infty, - \tfrac{\b{\theta_n}}{\w{\theta_n}} ] &\colon \w{\theta_n} > 0 \\
    [\scra , \scrb] \cap [ - \tfrac{\b{\theta_n}}{\w{\theta_n}} , \infty) &\colon \w{\theta_n} < 0.
    \end{cases}
    \end{split}
\end{equation}
\cref{lem:constant:approx} hence proves that for all $n \in \N$ it holds that 
\begin{equation} \label{prop:product:vw:bounded:eq3}
\inf_{\xi \in \R} \br*{ \rho \int_{ [\scra , \scrb] \smallsetminus I ^{\theta_n }} ( f ( x ) - \xi ) ^2 \, \d x } = \inf_{\xi \in [-M , M]} \br*{ \rho \int_{ [\scra , \scrb] \smallsetminus I ^{\theta_n }} ( f ( x ) - \xi ) ^2 \, \d x } .
\end{equation}
This, \cref{prop:vw:bounded:eq1}, \cref{prop:vw:bounded:eq:interval},
and \cref{lem:constant:approx} demonstrate for all $n \in \N$ that
\begin{equation}
\begin{split}
     \cL ( \theta_n ) &\geq \inf_{\xi \in [-M , M]} \br*{ \rho \int_{ [\scra , \scrb] \smallsetminus I ^{\theta_n }} ( f ( x ) - \xi ) ^2 \, \d x } \\
     &= \inf_{\xi \in [-M, M]} \br*{ \rho \int_{ [\scra , \scrb ] } ( f ( x ) - \xi ) ^2 \, \d x - \rho \int_{ I ^{\theta_n }} ( f ( x ) - \xi ) ^2 \, \d x} \\
     &\geq \inf_{\xi \in [-M, M]} \br*{ \rho \int_{ [\scra , \scrb ] } ( f ( x ) - \xi ) ^2 \, \d x - \rho \int_{ I ^{\theta_n }} (\abs{ f ( x ) } + \abs{ \xi } ) ^2 \, \d x} \\
     &\geq \inf_{\xi \in [-M, M]} \br*{\rho \int_{ [\scra , \scrb ] } ( f ( x ) - \xi ) ^2 \, \d x - \rho \int_{ I ^{\theta_n }} (2 M ) ^2 \, \d x} \\
      &\geq \br*{ \inf_{\xi \in [-M , M ]} \rho \int_\scra^\scrb( f ( x ) - \xi ) ^2 \, \d x } - 4 \rho M^2 \lambda(I ^{\theta_n} ) \\
      &= m - 4 \rho M^2 \lambda(I ^{\theta_n} ) .
  \end{split}
\end{equation}
Combining this with \cref{prop:vw:bounded:eq3} and \cref{prop:product:vw:bounded:eq2} shows that \begin{equation}
    m > m - \varepsilon \geq  \liminf\nolimits_{n \to \infty} \cL ( \theta_n ) \geq m.
\end{equation} 
This is a contradiction.
\end{cproof}

\section{Convergence of the risk of GFs in the training of ANNs with one hidden neuron}
\label{section:gf:convergence:1neuron}

The main result of this section, \cref{theo:flow:convergence:small:loss} in \cref{subsection:gf:conv:affine} below, demonstrates in the special situation
where the measure $\mu$ (see \cref{setting:snn}) is up to a constant the Lebesgue--Borel measure on $[\scra , \scrb]$,
where the hidden layer consists of only one neuron (where $\width = 1$),
 and where the target function $f \colon [\scra , \scrb] \to \R$ is affine linear that the risk of every not necessarily bounded GF trajectory converges to zero.
 Our proof of \cref{theo:flow:convergence:small:loss} employs
 some of the results in \cref{section:gf:convergence,section:small:risk}, the a priori bounds for GF trajectories with sufficiently small initial risk in \cref{lem:flow:product:vw:bounded} in \cref{subsection:gf:apriopri:1neuron} below, the convergence properties of ANNs with uniformly convergent realization functions in \cref{lem:realizations:closed} in \cref{subsection:convergent:realizations}, and the well-known fact for integral equations in \cref{lem:affine:integral:zero} in \cref{subsection:gf:conv:affine}. Only for completeness we include in this section a detailed proof for \cref{lem:affine:integral:zero}. 

 In our proof of \cref{theo:flow:convergence:small:loss} we first employ \cref{lem:loss:integral} in \cref{subsection:gf:critical} to obtain
that $[0, \infty ) \ni t \mapsto \cG ( \Theta_t ) \in \R^4$ is $L^2$-integrable.
This allows us to extract a subsequence along which the standard norm of the generalized gradient converges to zero. 
In the next step \cref{lem:flow:product:vw:bounded} enables us to conclude
that the realization functions of the corresponding ANNs are uniformly equicontinuous. 
This, in turn, allows us to bring the Arzela-Ascoli theorem into play to obtain
that along some sub-subsequence the realization functions converge uniformly on $[\scra , \scrb]$. 
It then remains to prove that the limit of these uniformly convergent ANN realization functions coincides with the affine linear target function.
We verify this by employing \cref{lem:realizations:closed} in combination with a careful analysis of the gradient given by \cref{eq:loss:gradient:width1}.

As a consequence of \cref{theo:flow:convergence:small:loss}, we prove in \cref{cor:flow:convergence:uniform} in the special situation
 where the measure $\mu$ (see \cref{setting:snn}) is up to a constant the Lebesgue--Borel measure on $[\scra , \scrb]$,
 where the hidden layer consists of only one neuron (where $\width = 1$),
 and where the target function $f \colon [\scra , \scrb] \to \R$ is affine linear
 that the realization functions of the GF trajectory converge to the target function not only in $L^2$-sense (\cref{theo:flow:convergence:small:loss}) but even uniformly in the set of all continuous functions $C( [\scra , \scrb], \R )$ from $[\scra , \scrb ]$ to $\R$. 

Our formulations of the statements in \cref{lem:flow:product:vw:bounded}, \cref{cor:boundedness:partial}, \cref{theo:flow:convergence:small:loss}, and \cref{cor:flow:convergence:uniform} also exploit the elementary regularity result in \cref{lem:gradient:measurable:width1} in \cref{subsection:gf:apriopri:1neuron}. \cref{lem:gradient:measurable:width1} clarifies in the framework of \cref{setting:snn:width1} that the generalized gradient function $\cG \colon \R^4 \to \R^4$ is locally bounded and measurable and, thereby, in particular ensures for every continuous function $ \Theta \colon [0,\infty) \to \R^4 $ and every $t \in [0,\infty)$ that the Lebesgue integral $\int_0^t \cG ( \Theta_s ) \, \d s$ is well-defined. \cref{lem:gradient:measurable:width1} is an immediate consequence of the more general result in \cref{cor:gradient:measurable} from \cref{subsection:loss:gradient} above.

\subsection{A priori estimates for GFs}
\label{subsection:gf:apriopri:1neuron}

\begin{lemma} \label{lem:gradient:measurable:width1}
Assume \cref{setting:snn:width1}. Then it holds that $\cG$ is locally bounded and measurable.
\end{lemma}
\begin{cproof}{lem:gradient:measurable:width1}
\Nobs that \cref{cor:gradient:measurable} establishes that $\cG$ is locally bounded and measurable.
\end{cproof}

\cfclear
\begin{lemma} \label{lem:flow:product:vw:bounded}
Assume \cref{setting:snn:width1}, let $\Theta \in C([0, \infty) , \R^{4})$ satisfy for all $t \in [0, \infty)$ that $\Theta_t = \Theta_0 - \int_0^t \cG ( \Theta_s ) \,  \d s$, let $m \in \R$ satisfy $m = \rho \int_\scra^\scrb \rbr{f ( x ) - (\scrb - \scra)^{-1} \int_\scra^\scrb f(y) \, \d y }^2 \, \d x$,
and assume $\cL ( \Theta_0 ) < m$ \cfadd{lem:gradient:measurable:width1}\cfload. Then 
\begin{enumerate} [label = (\roman*)]
    \item \label{lem:flow:product:vw:bounded:item1} it holds that $\sup_{t \in [0, \infty)} \abs{\w{\Theta_t} \v{\Theta_t} } < \infty$,
    \item \label{lem:flow:product:vw:bounded:item2} it holds that 
    \begin{equation}
        \sup\nolimits_{t \in [0, \infty)} \abs{\w{\Theta_t} } 
        \leq \br[\big]{ \sup \nolimits_{t \in [0, \infty ) } \max \cu*{1 , \abs{\w{\Theta_0}}^2 + \abs{\b{\Theta_0 }} ^2 - \abs{\v{\Theta_0 }}^2 + \abs{\w{\Theta_t} \v{\Theta_t} }^2 } }^{1/2} < \infty,
    \end{equation}
    \item \label{lem:flow:product:vw:bounded:item3} it holds for all $t \in [0, \infty )$ that 
    \begin{equation}
       \sup\nolimits_{x \in [\scra , \scrb]} \abs{\realization{\Theta_t}(x)} \leq 2 \br[\big]{ \sup \nolimits_{x \in [\scra , \scrb ] } \abs{f ( x ) } } + ( \scrb - \scra ) \abs{\w{\Theta_t} \v{\Theta_t} } < \infty ,
    \end{equation}
    and
        \item \label{lem:flow:product:vw:bounded:item4} it holds for all $\alpha , \beta \in \R$ with $\forall \, x \in [\scra , \scrb] \colon f(x) = \alpha x + \beta$ that $\inf_{t \in [0, \infty)} \alpha \w{\Theta_t} \v{\Theta_t} > 0$.
\end{enumerate}
\end{lemma}

\begin{cproof}{lem:flow:product:vw:bounded}
Throughout this proof let $w = (w_t)_{t \in [0, \infty ) }$, $b = ( b_t ) _{t \in [0, \infty ) }$, $v = ( v_t )_{t \in [0, \infty ) }$, $c = ( c_t )_{t \in [0, \infty ) } \in C ( [ 0 , \infty ) , \R )$
satisfy for all $t \in [0, \infty)$ that 
\begin{equation}
    w_t = \w{\Theta_t}, \qquad b_t = \b{\Theta_t},
    \qquad v_t = \v {\Theta_t}, \qandq c_t = \c{\Theta_t} ,
\end{equation}
let $M \in \R$ satisfy $M = \sup_{x \in [\scra , \scrb ] } \abs{f ( x ) }$,
let $A \in \R$ satisfy $A = \abs{ w_0 }^2 +   \abs{ b_0 }^2 -  \abs{ v_0 }^2$,
and let $\fC \in [0, \infty ]$ satisfy $\fC = \sup_{t \in [0, \infty ) } \abs{w_t v_t }$.
\Nobs that \cref{lem:loss:integral} demonstrates for all $t \in [0, \infty)$ that $\cL ( \Theta_t ) \leq \cL ( \Theta_0 ) < m$.
\cref{cor:constant:approx:affine}, \cref{cor:prod:vw:positive},
and \cref{prop:product:vw:bounded} hence establish \cref{lem:flow:product:vw:bounded:item1,lem:flow:product:vw:bounded:item4}.

In the next step we prove \cref{lem:flow:product:vw:bounded:item2}. \Nobs that \cref{prop:flow:invariant} implies for all $t \in [0, \infty)$ that
\begin{equation}
     \abs{ w_t }^2 -   \abs{ v_t }^2
    \leq   \abs{ w_t }^2 +   \abs{ b_t }^2 -   \abs{ v_t }^2
    =   \abs{ w_0 }^2 +   \abs{ b_0 }^2 -  \abs{ v_0 }^2 = A.
\end{equation}
Combining this with the fact that $\sup_{t \in [0, \infty) } \abs{w_t v_t } = \fC < \infty$ ensures for all $t \in [0, \infty)$ with $\abs{w_t} \ge 1$ that
\begin{equation}
    \abs{w_t}^2 \leq A + \abs{v_t}^2 \leq A + \tfrac{ \fC^2}{ \abs{w_t}^2 } \leq A + \fC^2.
\end{equation}
Hence, we obtain for all $t \in [0, \infty)$ that
$\abs{w_t}^2 \leq \max \cu{A + \fC^2 , 1 } < \infty$.
 This establishes \cref{lem:flow:product:vw:bounded:item2}.

Finally, we prove \cref{lem:flow:product:vw:bounded:item3}.
\Nobs that \cref{lem:constant:approx} implies that $m \leq \rho \int_\scra^\scrb (f(y))^2 \, \d y \leq \rho ( \scrb - \scra ) M^2$.
Combining this with \cref{lem:loss:integral} assures that for all $t \in [0, \infty)$ we have that
\begin{equation}
   \rho \int_\scra^\scrb (\realization{\Theta_t} ( y ) - f ( y ) ) ^2 \, \d y 
   = \cL ( \Theta_t ) \leq \cL ( \Theta_0 ) \leq m \leq \rho ( \scrb - \scra ) M^2.
\end{equation}
This shows that there exists $x = (x_t)_{ t \in [0, \infty ) } \colon [0, \infty ) \to [\scra , \scrb]$ which satisfies for all $t \in [0, \infty )$ that
\begin{equation}
    \abs{ \realization{\Theta_t} ( x_t ) - f ( x_t ) } \leq M.
\end{equation}
In addition, \nobs that \cref{lem:realization:lipschitz:const} ensures that for all $t \in [0, \infty)$, $x,y \in [\scra , \scrb]$ it holds that $\abs{ \realization{\Theta_t} ( x ) - \realization{\Theta_t} ( y ) } \leq \abs{w_t v_t} \abs{x-y}$.
Hence, we obtain for all $t \in [0, \infty)$, $y \in [\scra , \scrb ]$ that
\begin{equation}
    \begin{split}
        \abs{\realization{\Theta_t} ( y ) } 
        &\leq \abs{\realization{\Theta_t } ( x_t ) } 
        + \abs{ \realization{\Theta_t} ( y ) - \realization{\Theta_t} ( x_t ) } \\
        & \leq \abs{f ( x_t ) } +  \abs{ \realization{\Theta_t} ( x_t ) - f ( x_t ) } + \abs{w_t v_t} \abs{y - x_t } \\
        & \leq M + M + \abs{w_t v_t} ( \scrb - \scra ) = 2 M + \abs{w_t v_t} ( \scrb - \scra ).
    \end{split}
\end{equation}
This establishes \cref{lem:flow:product:vw:bounded:item3}. 
\end{cproof}

\cfclear
\begin{cor} \label{cor:boundedness:partial}
Assume \cref{setting:snn:width1} and let $\Theta \in C([0, \infty) , \R^{4})$ satisfy for all $t \in [0, \infty)$ that $\Theta_t = \Theta_0 - \int_0^t \cG ( \Theta_s ) \,  \d s$ \cfadd{lem:gradient:measurable:width1}\cfload.
Then 
\begin{enumerate} [label = (\roman*)]
    \item \label{cor:boundedness:partial:item1} it holds that $\sup_{t \in [0, \infty)} \abs{\w{\Theta_t} \v{\Theta_t} } < \infty$ and
    \item \label{cor:boundedness:partial:item2} it holds that $\sup_{t \in [0, \infty)} \abs{\w{\Theta_t} } < \infty$.
    \end{enumerate}
\end{cor}
\begin{cproof} {cor:boundedness:partial}
Throughout this proof let $m \in \R$ satisfy 
\begin{equation}
    m = \rho \textstyle\int_\scra^\scrb \rbr[\big]{ f ( x ) - (\scrb - \scra)^{-1} \textstyle\int_\scra^\scrb f(y) \, \d y }^{ 2 } \, \d x.
\end{equation}
In the following we distinguish between the case $\inf_{t \in [0, \infty ) } \cL ( \Theta_t ) \ge m$ and the case $\inf_{t \in [0, \infty ) } \cL ( \Theta_t ) \allowbreak < m$. 
We first establish \cref{cor:boundedness:partial:item1,cor:boundedness:partial:item2} in the case
\begin{equation} \label{cor:boundedness:partial:case1}
    \inf\nolimits_{t \in [0, \infty ) } \cL ( \Theta_t ) \ge m .
\end{equation}
\Nobs that \cref{cor:boundedness:partial:case1} and \cref{cor:flow:bounded}
show that
\begin{equation}
    \sup\nolimits _{ t \in [0, \infty) } \norm{\Theta_t } \leq 3 \norm{\Theta_0 } ^2 + 8 \abs[\big]{ (\scrb - \scra)^{-1} \textstyle\int_\scra^\scrb f(y) \, \d y } ^2 < \infty.
\end{equation}
This establishes \cref{cor:boundedness:partial:item1,cor:boundedness:partial:item2} in the case $\inf_{t \in [0, \infty ) } \cL ( \Theta_t ) \ge m$.
In the next step we prove \cref{cor:boundedness:partial:item1,cor:boundedness:partial:item2} in the case 
\begin{equation}
    \label{cor:boundedness:partial:case2}
    \inf\nolimits_{t \in [0, \infty ) } \cL ( \Theta_t ) < m.
\end{equation}
\Nobs that \cref{cor:boundedness:partial:case2} assures that
there exists $T \in [0, \infty)$ which satisfies that $\cL ( \Theta_T ) < m$. \Nobs that the fact that $\Theta \colon [0, \infty ) \to \R^4$ is continuous implies that $\sup_{t \in [0, T ]} \abs{\w{\Theta_t} \v{\Theta_t} } < \infty$ and $\sup_{t \in [0, T ]} \abs{\w{\Theta_t} } < \infty$.
Next let $\varTheta \in C([0, \infty) , \R^{4})$ satisfy for all $t \in [0, \infty)$ that $\varTheta_t = \Theta_{T + t }$. \Nobs that the integral transformation theorem ensures for all $t \in [0, \infty)$ that $\cL ( \varTheta_0 ) = \cL ( \Theta_T ) < m$ and 
\begin{equation}
    \begin{split} 
    \varTheta_t 
    &= \Theta_{T + t } = \Theta_0 - \int_0^{T + t} \cG ( \Theta_s ) \, \d s 
    = \br*{\Theta_0 - \int_0^T \cG ( \Theta_s ) \, \d s } - \int_T^{T + t} \cG ( \Theta_s ) \, \d s \\
    &= \Theta_T - \int_0^t \cG ( \Theta_{T + s } ) \, \d s 
    = \varTheta _0 - \int_0^t \cG (\varTheta _s ) \,  \d s.
    \end{split}
\end{equation}
\cref{lem:flow:product:vw:bounded} hence proves that $\sup_{t \in [T , \infty ) } \abs{\w{\Theta_t} \v{\Theta_t} } = \sup_{t \in [0 , \infty ) } \abs{\w{\varTheta _t} \v{\varTheta_t} } < \infty$
and $\sup_{t \in [T , \infty ) } \abs{\w{\Theta_t} } \allowbreak = \sup_{t \in [0 , \infty ) } \abs{\w{\varTheta _t} } < \infty$. 
This establishes \cref{cor:boundedness:partial:item1,cor:boundedness:partial:item2} in the case $\inf_{t \in [0, \infty ) } \cL ( \Theta_t ) < m$.
\end{cproof}

\subsection{Properties of ANN parameters for convergent sequences of ANN realizations}
\label{subsection:convergent:realizations}

\begin{lemma} \label{lem:realizations:closed}
Assume \cref{setting:snn:width1}, let $(\theta_n)_{n \in \N} \subseteq \R^4$, $h \in C([\scra , \scrb] , \R)$ satisfy
\begin{equation} \label{lem:realizations:closed:eq:assumption}
    \limsup\nolimits_{n \to \infty} \sup\nolimits_{x \in [\scra , \scrb]} \abs{\realization{\theta_n} ( x ) - h ( x ) } = 0,
\end{equation}
and assume that $h$ is not constant. Then
\begin{enumerate} [label = (\roman*)]
    \item \label{lem:realizations:closed:item1} there exists $\vartheta \in \R^4$ which satisfies $ \realization{\vartheta} | _ { [ \scra , \scrb ] } = h $,
    \item \label{lem:realizations:closed:item2} it holds that $\limsup _{n \to \infty} \abs{ \w{\theta_n} \v{\theta_n} - \w{\vartheta}  \v{\vartheta} } = 0$, and
    \item \label{lem:realizations:closed:item3} it holds that $\limsup_{n \to \infty} \lambda ( I ^{\theta_n} \Delta I ^\vartheta ) = 0$.
\end{enumerate}
\end{lemma}

\begin{cproof} {lem:realizations:closed}
\Nobs that \cite[Theorem 3.8]{PetersenRaslanVoigtlaender2020} ensures that there exists $\vartheta \in \R^4$ which satisfies $ \realization{\vartheta} | _ { [ \scra , \scrb ] } = h $.
This establishes \cref{lem:realizations:closed:item1}.

In the next step we prove that $\limsup_{n \to \infty} \lambda ( I ^\vartheta \backslash I ^{\theta_n}  ) = 0$.
\Nobs that the assumption that $h$ is not constant implies that $ \lambda ( I ^\vartheta ) > 0 $ and $\w{\vartheta}  \v{\vartheta}  \not= 0$. 
Moreover,
\nobs that \cref{setting:snn:width1:eq:realization} ensures that
for all $n \in \N$, $x \in [\scra , \scrb] \backslash I ^{\theta_n}$ it holds that $\realization{\theta_n} ( x ) = \c{\theta_n}$. This, the fact that for all $x \in I ^\vartheta$ it holds that $ \realization{\vartheta} ( x ) = \w{\vartheta}  \v{\vartheta} x + \v{\vartheta} \b{\vartheta} + \c{\vartheta} $, and \cref{cor:constant:approx:affine} imply that for all $n \in \N$ we have that
\begin{equation} \label{lem:realizations:closed:eq1}
    \int_\scra^\scrb \abs{ \realization{\theta_n} ( x ) - \realization{\vartheta } ( x ) } ^2 \, \d x 
    \geq  \int_{I ^\vartheta \backslash I ^{\theta_n}} (h ( x ) - \c{\theta_n })^2 \, \d x
    \geq \frac{ \abs{ \w{\vartheta}  \v{\vartheta}  }^2 \rbr{ \lambda ( I ^\vartheta \backslash I ^{\theta_n}  ) } ^3 }{12 }.
\end{equation}
Furthermore, \nobs that \cref{lem:realizations:closed:eq:assumption} assures that
\begin{equation} \label{lem:realizations:closed:eq:l2conv}
    \limsup_{n \to \infty} \br*{ \int_\scra^\scrb \abs{ \realization{\theta_n} ( x ) - \realization{\vartheta} ( x ) } ^2 \, \d x }
    = 0.
\end{equation}
This, \cref{lem:realizations:closed:eq1}, and the fact that $\w{\vartheta}  \v{\vartheta}  \not= 0$ demonstrate that $\limsup_{n \to \infty} \lambda ( I ^\vartheta \backslash I ^{\theta_n}  ) = 0$. Hence, we have that $\limsup_{n \to \infty} \abs{ \lambda ( I ^\vartheta \cap I ^{\theta_n}  ) - \lambda(I ^\vartheta) } = 0$.
Next \nobs that \cref{setting:snn:width1:eq:realization} shows that for all $n \in \N$, $x \in I ^\vartheta \cap I ^{\theta_n} $ it holds that $\realization{\theta_n} ( x ) - \realization{\vartheta} ( x ) = (\w{\theta_n} \v{\theta_n} - \w{\vartheta}  \v{\vartheta}) x + (\v{\theta_n} \b{\theta_n} + \c{\theta_n} - \v{\vartheta} \b{\vartheta} - \c{\vartheta} )$.
Combining this and \cref{cor:constant:approx:affine} proves for all $n \in \N$ that
\begin{equation}
\begin{split}
        \int_\scra^\scrb \abs{ \realization{\theta_n} ( x ) - \realization{\vartheta} ( x ) } ^2 \, \d x
    &\geq \int_{I ^\vartheta \cap I ^{\theta_n}} \abs{\realization{\theta_n} ( x ) - \realization{\vartheta} ( x ) } ^2 \, \d x \\
    &\geq \frac{\abs{\w{\theta_n} \v{\theta_n} - \w{\vartheta}  \v{\vartheta}  }^2 \rbr{ \lambda ( I ^\vartheta \cap I ^{\theta_n}  ) } ^3}{12}.
\end{split}
\end{equation}
This, \cref{lem:realizations:closed:eq:l2conv},
and the fact that $\lim_{n \to \infty} \lambda ( I ^\vartheta \cap I ^{\theta_n}  ) = \lambda(I ^\vartheta) > 0$ ensure that $\limsup_{n \to \infty} \abs{\w{\theta_n} \v{\theta_n} - \w{\vartheta}  \v{\vartheta}  } = 0$, which establishes \cref{lem:realizations:closed:item2}.

It remains to prove that $\limsup_{n \to \infty} \lambda(  I ^{\theta_n} \backslash I ^\vartheta ) = 0$. \Nobs that \cref{setting:snn:width1:eq:realization} implies that
for all $x \in [\scra , \scrb] \backslash I ^\vartheta$ it holds that $\realization{\vartheta} ( x ) = \c{\vartheta}$. This, the fact that for all $n \in \N$, $x \in I^{\theta_n}$ we have that $\realization{\theta_n} ( x ) = \w{\theta_n} \v{\theta_n} x + \v{\theta_n} \b{\theta_n} + \c{\theta_n}$, and \cref{cor:constant:approx:affine} show that for all $n \in \N$ it holds that
\begin{equation}
     \int_\scra^\scrb \abs{ \realization{\theta_n} ( x ) - \realization{\vartheta } ( x ) } ^2 \, \d x \geq \int_{I ^{\theta_n} \backslash I ^\vartheta } \abs{ \realization{\theta_n} ( x ) - \c{\vartheta}}^2 \, \d x \geq \frac{\abs{\w{\theta_n} \v{\theta_n}}^2 \rbr{ \lambda(  I ^{\theta_n} \backslash I ^\vartheta ) }^3}{12}.
\end{equation}
Combining this and
\cref{lem:realizations:closed:eq:l2conv}
with the fact that $\lim_{n \to \infty} \w{\theta_n} \v{\theta_n} = \w{\vartheta}  \v{\vartheta}  \not= 0$ demonstrates that $\limsup_{n \to \infty} \lambda(  I ^{\theta_n} \backslash I ^\vartheta ) = 0$.
This proves \cref{lem:realizations:closed:item3}.
\end{cproof}

\subsection{Convergence of the risk of GFs to zero for affine linear target functions}
\label{subsection:gf:conv:affine}

\begin{prop} \label{prop:loss:gradient:width1}
Assume \cref{setting:snn:width1} and let $\theta \in \R^{\fd}$.
Then
\begin{equation} \label{eq:loss:gradient:width1}
\begin{split}
        \cG_{1} ( \theta) &= 2 \rho \v{\theta} \int_{I^\theta} x  ( \realization{\theta} (x) - f ( x ) ) \, \d x , \\
        \cG_{2} ( \theta) &=  2 \rho  \v{\theta} \int_{I^\theta} (\realization{\theta} (x) - f ( x ) ) \,  \d x , \\
        \cG_{3} ( \theta) &= 2 \rho \int_{\scra}^\scrb \br[\big]{\max \cu{ \w{\theta} x + \b{\theta} , 0 } } ( \realization{\theta}(x) - f ( x ) ) \, \d x , \\
       \text{and} \qquad  \cG_{4} ( \theta) &= 2 \rho \int_{\scra}^\scrb (\realization{\theta} (x) - f ( x ) ) \, \d x .
        \end{split}
\end{equation}
\end{prop}

\begin{cproof}{prop:loss:gradient:width1}
\Nobs that \cref{prop:loss:gradient} establishes \cref{eq:loss:gradient:width1}.
\end{cproof}

\begin{lemma} \label{lem:affine:integral:zero}
Let $\scra \in \R$, $\scrb \in (\scra , \infty)$, $\alpha_1, \alpha_2, \beta_1, \beta_2 \in \R$ satisfy
\begin{equation} \label{lem:affine:integral:zero:assumption}
    \int_\scra^\scrb x ( (\alpha_1 x + \beta_1) - (\alpha_2 x + \beta_2) ) \, \d x = \int_\scra^\scrb ((\alpha_1 x + \beta_1) - (\alpha_2 x + \beta_2) ) \, \d x = 0.
\end{equation}
Then $\alpha_1 = \alpha_2$ and $\beta_1 = \beta_2$.
\end{lemma}
\begin{cproof} {lem:affine:integral:zero}
\Nobs that \cref{lem:affine:integral:zero:assumption} assures that
\begin{equation}
    \begin{split}
        0 &= (\alpha_1 - \alpha_2 ) \br*{ \int_\scra^\scrb x ( (\alpha_1 x + \beta_1) - (\alpha_2 x + \beta_2) ) \, \d x } \\
        & \quad + (\beta_1 - \beta_2 ) \br*{ \int_\scra^\scrb ((\alpha_1 x + \beta_1) - (\alpha_2 x + \beta_2) ) \, \d x } \\
        &= \int_\scra^\scrb ( (\alpha_1 - \alpha_2) x + (\beta_1 - \beta_2) ) ( (\alpha_1 x + \beta_1) - (\alpha_2 x + \beta_2) ) \, \d x \\
        &= \int_\scra^\scrb ( (\alpha_1 - \alpha_2) x + (\beta_1 - \beta_2) ) ^2 \, \d x.
    \end{split}
\end{equation}
This and the fact that for all $x \in [\scra , \scrb]$ it holds that $( ( \alpha_1 - \alpha_2 ) x + (\beta_1 - \beta_2) ) ^2 \geq 0$ show that for all $x \in [\scra , \scrb]$ it holds that $( (\alpha_1 - \alpha_2) x + (\beta_1 - \beta_2) ) = 0$. Hence, we obtain that $\alpha _ 1 -\alpha_2 = \beta_1 - \beta_2 = 0$. 
\end{cproof}

\cfclear
\begin{theorem} \label{theo:flow:convergence:small:loss}
Assume \cref{setting:snn:width1}, let $\alpha, \beta \in \R$ satisfy for all $ x \in [\scra , \scrb]$ that $f(x)= \alpha x + \beta$,
and let $\Theta  \in C([0, \infty) , \R^{4})$ satisfy for all $t \in [0, \infty)$ that $\Theta_t = \Theta_0 - \int_0^t \cG ( \Theta_s ) \,  \d s$ and $\cL ( \Theta_0 ) < \frac{\rho \alpha^2 ( \scrb - \scra)^3 }{12}$ \cfadd{lem:gradient:measurable:width1}\cfload. Then $\limsup_{t \to \infty} \cL ( \Theta_t ) = 0$.
\end{theorem}

\begin{cproof}{theo:flow:convergence:small:loss}
Throughout this proof let $w = (w_t)_{t \in [0, \infty ) }$, $b = ( b_t ) _{t \in [0, \infty ) }$, $v = ( v_t )_{t \in [0, \infty ) }$, $c = ( c_t )_{t \in [0, \infty ) } \in C ( [ 0 , \infty ) , \R )$
satisfy for all $t \in [0, \infty)$ that 
\begin{equation}
    w_t = \w{\Theta_t}, \qquad b_t = \b{\Theta_t},
    \qquad v_t = \v {\Theta_t}, \qandq c_t = \c{\Theta_t} 
\end{equation}
and let $\cI_t \subseteq [\scra , \scrb]$, $t \in [0, \infty)$, satisfy for all $t \in [0, \infty)$ that $\cI_t = I ^{\Theta_t}$.
\Nobs that \cref{lem:loss:integral} implies that $[0, \infty) \ni t \mapsto \cL ( \Theta_t ) \in \R$ is non-increasing. Hence, we obtain that 
\begin{equation} \label{flow:theo:eq:convergence}
    \limsup\nolimits_{t \to \infty} \cL ( \Theta_t ) = \liminf\nolimits_{t \to \infty} \cL ( \Theta_t ) = \inf\nolimits_{t \in [0, \infty)} \cL ( \Theta_t ) .
\end{equation}
Next \nobs that \cref{lem:loss:integral} proves that $\int_0^\infty \norm{\cG ( \Theta_s ) } ^2 \, \d s < \infty$.
This demonstrates that $\liminf_{t \to \infty} \allowbreak
\norm{\cG ( \Theta_t ) } = 0$.
Therefore, we obtain that there exist $\tau_n \in [0, \infty )$, $n \in \N$,
which satisfy $\liminf_{n \to \infty} \tau_n = \infty$ and $\limsup_{n \to \infty} \norm{\cG ( \Theta_{\tau_n} ) } = 0$.
\Nobs that \cref{lem:flow:product:vw:bounded} implies that 
\begin{equation}
    \sup\nolimits_{n \in \N} \abs{w_{\tau_n} v_{ \tau_n} } < \infty \qandq
    \sup\nolimits_{n \in \N} \sup\nolimits_{x \in [\scra , \scrb]} \abs{\realization{\Theta_{\tau_n} }(x)} < \infty.
\end{equation} 
Combining this and \cref{lem:realization:lipschitz:const} proves that there exists $\fC \in \R$ such that for all $x,y \in [\scra , \scrb]$, $n \in \N$ it holds that  $\abs{\realization{\Theta_{\tau_n}} ( x ) - \realization{\Theta_{ \tau_n}} ( y ) } \leq \fC \abs{x - y}$ and $\abs{ \realization{\Theta_{\tau_n}} ( x ) } \leq \fC$. 
The Arzela-Ascoli theorem hence shows that there exist $h \in C([\scra , \scrb] , \R)$ and a strictly increasing $k \colon \N \to \N$ which satisfy
\begin{equation} \label{flow:theo:eq:defh}
    \limsup \nolimits_{n \to \infty} \sup\nolimits_{x \in [\scra , \scrb]} \abs{\realization{\Theta_{\tau_{k ( n ) }}} ( x ) - h ( x ) } = 0. 
\end{equation}
Combining this with \cref{flow:theo:eq:convergence} and the assumption that $\cL ( \Theta_0 ) < \frac{\rho \alpha^2 ( \scrb - \scra)^3}{12}$ implies that 
\begin{equation} \label{theo:flow:eq:loss:bound}
    \rho \int_\scra^\scrb ( f ( x ) - h ( x ) ) ^2 \, \d x = \limsup\nolimits_{n \to \infty} \cL ( \Theta_{\tau_{k ( n ) } } ) = \inf\nolimits_{t \in [0, \infty ) } \cL ( \Theta_t ) < \frac{\rho \alpha^2 ( \scrb - \scra)^3 }{12}. 
\end{equation}
This and \cref{cor:constant:approx:affine}
assure that $h$ is not constant.
\cref{lem:realizations:closed} hence ensures that there exists $\vartheta \in \R^4$ which satisfies $ \realization{\vartheta} | _ { [ \scra , \scrb ] } = h $.
Combining this and \cref{theo:flow:eq:loss:bound}
with \cref{cor:constant:approx:affine},
\cref{cor:prod:vw:positive},
and \cref{lem:neuron:active} demonstrates that $\alpha \w{ \vartheta } \v{ \vartheta } > 0 $ and $I^{\vartheta} \not= \emptyset$.
In addition,
\nobs that \cref{eq:loss:gradient:width1,flow:theo:eq:defh} show that
\begin{equation} \label{flow:theo:eq1}
\begin{split}
    0 &= \frac{1}{2 \rho } \br*{ \lim_{n \to \infty} \cG_4 ( \Theta_{\tau_{k ( n ) }} ) }
    = \lim_{n \to \infty} \br*{\int_\scra^\scrb ( \realization{\Theta_{\tau_{k ( n ) }}} ( x ) - (\alpha x + \beta ) ) \, \d x } \\
    &= \int_\scra^\scrb ( \realization{\vartheta} ( x ) - (\alpha x + \beta ) ) \, \d x.
\end{split}
\end{equation}
Furthermore, \nobs that \cref{flow:theo:eq:defh,lem:realizations:closed} prove that $\limsup_{n \to \infty} \lambda ( \cI_{\tau_{k ( n ) } } \Delta I ^\vartheta ) = 0$. Combining this and the fact that $\limsup_{n \to \infty} \sup_{x \in [\scra , \scrb]} \abs{\realization{\Theta_{\tau_{k ( n ) }}} ( x ) - \realization{\vartheta} ( x ) } = 0$ demonstrates that
\begin{equation} \label{flow:theo:eq2}
    \limsup_{n \to \infty} \abs*{ \int_{\cI_{\tau_{k ( n ) }}} x (\realization{\Theta_{\tau_{k ( n ) }}} ( x ) - (\alpha x + \beta ) ) \, \d x - \int_{I ^\vartheta} x ( \realization{\vartheta} ( x ) - (\alpha x + \beta ) ) \, \d x } = 0
\end{equation}
and
\begin{equation} \label{flow:theo:eq3}
        \limsup_{n \to \infty} \abs*{ \int_{\cI_{\tau_{k ( n ) }}}  (\realization{\Theta_{\tau_{k ( n ) }}} ( x ) - (\alpha x + \beta ) ) \, \d x - \int_{I ^\vartheta}  ( \realization{\vartheta} ( x ) - (\alpha x + \beta ) ) \, \d x } = 0.
\end{equation}
Moreover, \nobs that the fact that $\limsup_{n \to \infty} \norm{ \cG ( \Theta_{\tau_{k ( n ) }} ) } = 0$ and \cref{eq:loss:gradient:width1} imply that
\begin{equation} \label{flow:theo:eq4}
\begin{split}
     & \limsup_{n \to \infty} \abs*{ v_{\tau_{k ( n ) }} \int_{\cI_{\tau_{k ( n ) }}} x (\realization{\Theta_{\tau_{k ( n ) }}} ( x ) - (\alpha x + \beta ) ) \, \d x} \\
     &=   \limsup_{n \to \infty} \abs*{ v_{\tau_{k ( n ) }} \int_{ \cI_{\tau_{k ( n ) }}}  (\realization{\Theta_{\tau_{k ( n ) }}} ( x ) - (\alpha x + \beta ) ) \, \d x} = 0.
\end{split}
\end{equation}
In the next step we show that
\begin{equation} \label{flow:theo:eq:grad0}
    \abs*{ \int_{I ^\vartheta} x ( \realization{\vartheta} ( x ) - (\alpha x + \beta ) ) \, \d x} = \abs*{\int_{I ^\vartheta}  ( \realization{\vartheta} ( x ) - (\alpha x + \beta ) ) \, \d x} = 0.
\end{equation}
We prove \cref{flow:theo:eq:grad0} by contradiction.
We thus assume that
\begin{equation} \label{flow:theo:eq:contr}
    \abs*{ \int_{I ^\vartheta} x ( \realization{\vartheta} ( x ) - (\alpha x + \beta ) ) \, \d x} + \abs*{ \int_{I ^\vartheta}  ( \realization{\vartheta} ( x ) - (\alpha x + \beta ) ) \, \d x} > 0.
\end{equation}
\Nobs that \cref{flow:theo:eq2,flow:theo:eq3,flow:theo:eq4,flow:theo:eq:contr} prove that $\limsup_{n \to \infty} \abs{ v_{\tau_{k ( n ) }} } = 0$. 
In addition, \nobs that \cref{lem:realizations:closed} assures that $\lim_{n \to \infty} ( w_{\tau_{k ( n ) }} v_{\tau_{k ( n ) }} ) = \w{\vartheta}  \v{\vartheta} \not= 0$.
Combining this with \cref{lem:flow:product:vw:bounded:item2} in \cref{lem:flow:product:vw:bounded} demonstrates that
$\infty = \liminf_{n \to \infty} \abs{w_{\tau_{k ( n ) }}} < \infty$.
This contradiction establishes \cref{flow:theo:eq:grad0}.
Next \nobs that for all $x \in I^\vartheta$ it holds that $\realization{\vartheta} ( x ) = \w{\vartheta} \v{\vartheta} x + \v{\vartheta} \b{\vartheta} + \c{\vartheta}$. Combining this, \cref{flow:theo:eq:grad0},
and \cref{lem:affine:integral:zero} ensures that for all $x \in I^\vartheta$ it holds that 
\begin{equation} \label{theo:flow:eq:limit:realization}
	\realization{\vartheta} ( x ) = \alpha x + \beta. 
\end{equation}
\Nobs that for all $q \in (\scra , \scrb)$ with $I ^\vartheta = (q , \scrb ]$ it holds that $\forall \, x \in [\scra , q] \colon \realization{\vartheta} ( x ) = \realization{\vartheta} ( q ) = \alpha q + \beta$.
This,
\cref{flow:theo:eq1},
and \cref{flow:theo:eq:grad0} imply that for all
$q \in (\scra , \scrb)$ with $I ^\vartheta = (q , \scrb ]$
we have that
\begin{equation} \label{theo:flow:eq:interval:cont}
\begin{split}
    0 
    &= \int_\scra^\scrb (\realization{\vartheta} ( x ) - (\alpha x + \beta ) ) \, \d x 
    = \int_\scra^q (\realization{\vartheta} ( x ) - (\alpha x + \beta ) ) \, \d x \\
    &= \int_\scra^q (\alpha q - \alpha x ) \, \d x = \alpha \int _\scra^q ( q -  x ) \, \d x 
    = \frac{\alpha ( q - \scra ) ^2 }{2} \not= 0 .
\end{split}
\end{equation}
Furthermore, \nobs that
for all $q \in (\scra , \scrb)$ with $I ^\vartheta = [\scra , q)$ we have that $\forall \, x \in [q , \scrb ] \colon \realization{\vartheta} ( x ) = \realization{\vartheta} ( q ) = \alpha q + \beta$.
This, \cref{flow:theo:eq1},
and \cref{flow:theo:eq:grad0}
ensure that for all
$q \in (\scra , \scrb)$ with $I ^\vartheta = [\scra , q )$
it holds that
\begin{equation}
\begin{split}
    0 
    &= \int_\scra^\scrb (\realization{\vartheta} ( x ) - (\alpha x + \beta ) ) \, \d x 
    = \int_q ^\scrb (\realization{\vartheta} ( x ) - (\alpha x + \beta ) ) \, \d x \\
    &= \int_q^\scrb (\alpha q - \alpha x ) \, \d x 
    = \alpha \int_q^\scrb ( q - x ) \, \d x = - \frac{\alpha ( \scrb - q ) ^2 }{2} 
    \not= 0 .
\end{split}
\end{equation}
Combining this,
 \cref{theo:flow:eq:interval:cont}, and the fact that $\lambda ( I^\vartheta ) > 0$ shows that $I ^\vartheta \in \cu{ [\scra , \scrb ], (\scra , \scrb ] , [\scra , \scrb ) }$. This implies that $(\scra , \scrb ) \subseteq I^\vartheta$.
Combining this with \cref{theo:flow:eq:limit:realization} assures that for all $x \in (\scra , \scrb) $ we have that $\realization{\vartheta} ( x ) = \alpha x + \beta = f ( x )$. Hence, we obtain that
\begin{equation}
\int_{\scra}^\scrb ( f(x) - h(x) ) ^2 \, \d x = \int_{\scra}^\scrb ( f ( x ) - \realization{\vartheta} ( x ) ) ^2 \, \d x = 0.
\end{equation}
This, \cref{flow:theo:eq:convergence}, and \cref{theo:flow:eq:loss:bound} imply that $ \lim_{t \to \infty} \cL ( \Theta_t ) = \cL ( \vartheta ) = 0$.
\end{cproof}

\begin{cor} \label{theo:intro:1d:affine}
Let $\alpha, \beta, \scra \in \R$, $\scrb \in ( \scra, \infty)$,
let $\Rect_r \in C ( \R , \R )$, $r \in \N \cup \cu{ \infty } $, satisfy for all $x \in \R$ that $( \bigcup_{r \in \N} \cu{ \Rect_r } ) \subseteq C^1( \R , \R)$, $\Rect_\infty ( x ) = \max \cu{ x , 0 }$,
 $\sup_{r \in \N} \sup_{y \in [- \abs{x}, \abs{x} ] } \rbr*{ \abs{\Rect_r(y)} + \abs{ ( \Rect_r)'(y)}} \allowbreak
 < \infty$, and
\begin{equation}
    \limsup\nolimits_{r \to \infty}  \rbr*{ \abs { \Rect_r ( x ) - \Rect _\infty ( x ) } + \abs { (\Rect_r)' ( x ) - \indicator{(0, \infty)} ( x ) } } = 0,
\end{equation}
let $\cL_r \colon \R^4 \to \R$, $r \in \N \cup \cu{ \infty }$,
satisfy for all $r \in \N \cup \cu{ \infty }$, $\theta = (\theta_1, \ldots, \theta_4) \in \R^{4}$ that
\begin{equation}
    \cL_r ( \theta ) =  \int_{\scra}^\scrb \rbr[\big]{ \alpha x + \beta - \theta_{4} -  \theta_{3 }  \Rect_r ( \theta_{2}  +\theta_{1 } x ) }^2  \, \d x,
\end{equation}
let $\cG  \colon \R^4 \to \R^4$ satisfy for all
$\theta \in \cu{ \vartheta \in \R^4 \colon ( ( \nabla \cL_r ) ( \vartheta ) ) _{r \in \N} \text{ is convergent} }$
that $\cG ( \theta ) = \lim_{r \to \infty} (\nabla \cL_r) ( \theta )$,
let $\Theta \in C([0, \infty) , \R^4)$ satisfy for all $t \in [0, \infty)$ that $\Theta_t = \Theta_0 - \int_0^t \cG ( \Theta_s ) \, \d s$, and assume $\cL_\infty ( \Theta_0 ) < \frac{\alpha^2 ( \scrb - \scra ) ^3 }{12}$. 
Then $\limsup_{t \to \infty} \cL_\infty ( \Theta_t ) = 0$.

\end{cor}

\begin{cproof}{theo:intro:1d:affine}
\Nobs that \cref{theo:flow:convergence:small:loss} (applied with $\rho \with 1$ in the notation of \cref{theo:flow:convergence:small:loss}) shows that $\limsup_{t \to \infty} \cL_\infty ( \Theta_t ) = 0$.
\end{cproof}

\subsection{Uniform convergence of realizations of GFs for affine linear target functions}
\label{subsection:gf:conv:uniform}

\cfclear
\begin{cor} \label{cor:flow:convergence:uniform}
Assume \cref{setting:snn:width1}, 
 let $\alpha, \beta \in \R$ satisfy for all $ x \in [\scra , \scrb]$ that $f(x)= \alpha x + \beta$,
and let $\Theta  \in C([0, \infty) , \R^{4})$ satisfy for all $t \in [0, \infty)$ that $\Theta_t = \Theta_0 - \int_0^t \cG ( \Theta_s ) \,  \d s$ and $\cL ( \Theta_0 ) < \frac{\rho \alpha^2 ( \scrb - \scra ) ^3 }{12}$ \cfadd{lem:gradient:measurable:width1}\cfload. Then 
\begin{equation} \label{eq:flow:convergence:uniform}
\limsup\nolimits_{t \to \infty} \rbr[\big]{ \sup\nolimits_{x \in [\scra , \scrb ] } \abs{\realization{\Theta_t} ( x ) - (\alpha x + \beta ) } } = 0.
\end{equation}
\end{cor}
\begin{cproof}{cor:flow:convergence:uniform}
\Nobs that \cref{lem:flow:product:vw:bounded} assures that there exists $\fC \in (0, \infty)$ such that for all $t \in [0, \infty)$ it holds that $\abs{\w{\Theta_t} \v{\Theta_t}} \leq \fC$.
We now prove \cref{eq:flow:convergence:uniform} by contradiction. In the following we thus assume that
\begin{equation} \label{flow:convergence:uniform:eq:cont}
    \limsup\nolimits_{t \to \infty} \rbr[\big]{ \sup\nolimits_{x \in [\scra , \scrb ] } \abs{\realization{\Theta_t} ( x ) - (\alpha x + \beta ) } } > 0.
\end{equation}
\Nobs that \cref{flow:convergence:uniform:eq:cont} assures that there exist $\varepsilon \in (0, \infty)$ and $\tau_n \in [0, \infty)$, $n \in \N$, which satisfy $\liminf_{t \to \infty} \tau_n = \infty$ and 
\begin{equation} \label{flow:convergence:uniform:eq:tn}
   \inf\nolimits_{n \in \N} \rbr[\big]{ \sup\nolimits_{x \in [\scra , \scrb ]} \abs{\realization{\Theta_{\tau_n}} ( x ) - (\alpha x + \beta ) } } > \varepsilon.
\end{equation}
\Nobs that \cref{flow:convergence:uniform:eq:tn} shows that there exist $x_n \in [\scra , \scrb]$, $n \in \N$,
which satisfy for all $n \in \N$ that $\abs{\realization{\Theta_{\tau_n}} ( x_n ) - ( \alpha x_n + \beta ) } \geq \varepsilon$. 
Moreover, \nobs that \cref{lem:realization:lipschitz:const} proves that for all $n \in \N$, $y,z \in [\scra , \scrb]$ it holds that 
\begin{equation} \label{flow:convergence:uniform:eq:lipest}
\begin{split} 
    \abs{[\realization{\Theta_{\tau_n}} ( y ) - (\alpha y + \beta )] - [ \realization{\Theta_{\tau_n} } ( z ) - (\alpha z + \beta ) ] } 
    &\leq \abs{\realization{\Theta_{\tau_n}} ( y ) - \realization{\Theta_{ \tau_n}} ( z ) } + \abs{\alpha} \abs{y-z} \\
    &\leq ( \fC + \abs{\alpha} ) \abs{y-z}.
    \end{split}
\end{equation}
Next let $\delta \in (0, \infty)$ satisfy $\delta = \frac{\varepsilon}{2(\fC + \abs{\alpha} ) }$.
\Nobs that \cref{flow:convergence:uniform:eq:lipest} ensures that
for all $n \in \N$, $y \in [x_n-\delta, x_n + \delta] \cap [\scra , \scrb ]$ it holds that 
\begin{equation}
\begin{split}
&\abs{\realization{\Theta_{\tau_n}} ( y ) - (\alpha y + \beta ) } \\
&\geq \abs{\realization{\Theta_{\tau_n}} ( x_n ) - (\alpha x_n + \beta ) } -  \abs{[\realization{\Theta_{\tau_n}} ( x_n ) - (\alpha x_n + \beta )] - [ \realization{\Theta_{\tau_n} } ( y ) - (\alpha y + \beta ) ] }  \\
& \geq \varepsilon - ( \fC + \abs{\alpha} ) \abs{x_n - y } 
\geq \varepsilon - ( \fC + \abs{\alpha} ) \delta = \varepsilon - \tfrac{\varepsilon}{2} = \tfrac{\varepsilon}{2}.
\end{split}
\end{equation}
Furthermore, \nobs that for all $n \in \N$ we have that $\lambda ( [x_n-\delta, x_n + \delta] \cap [\scra , \scrb ] ) \geq \min \cu{ \delta , \scrb - \scra } $. 
This demonstrates that for all $n \in \N$ it holds that
\begin{equation}
    \cL ( \Theta_{\tau_n}) \geq \rho \int_{[x_n-\delta, x_n + \delta] \cap [\scra , \scrb ]} \abs{\realization{\Theta_{\tau_n}} ( y ) - (\alpha y + \beta ) }^2 \, \d y \geq \frac{\rho \varepsilon^2 \min\cu{ \delta , \scrb - \scra } }{4} .
\end{equation}
Combining this with \cref{theo:flow:convergence:small:loss} shows that
\begin{equation}
    0 = \limsup\nolimits_{t \to \infty} \cL ( \Theta_t ) \geq \limsup\nolimits_{n \to \infty} \cL ( \Theta_{\tau_n} ) \geq \tfrac{\rho \varepsilon^2 \min\cu{ \delta , \scrb - \scra } }{4} > 0.
\end{equation}
This is a contradiction.
\end{cproof}

\subsection*{Acknowledgements}
This work has been funded by the Deutsche Forschungsgemeinschaft
(DFG, German Research Foundation) under Germany’s Excellence Strategy EXC 2044-390685587, Mathematics Münster: Dynamics-Geometry-Structure.


\end{document}